\documentclass{article}
\usepackage{authblk}
\usepackage{times}
\usepackage{soul}
\usepackage{url}
\usepackage[hidelinks]{hyperref}
\usepackage[utf8]{inputenc}
\usepackage[small]{caption}
\usepackage{graphicx}
\usepackage{amsmath}
\usepackage{amsthm}
\usepackage{booktabs}
\usepackage{algorithm}
\usepackage{algorithmic}

\usepackage{helvet}
\usepackage{courier}
\usepackage{caption}

\usepackage[table,xcdraw]{xcolor}
\usepackage{makecell}
\usepackage{multirow}
\usepackage{algorithm}
\usepackage{algorithmic}
\usepackage{amsthm}
\usepackage{amssymb}
\usepackage[a4paper,left=2.5cm,right=2.5cm, top=3cm, bottom=3cm]{geometry}
\usepackage[
backend=biber,
sorting=ynt
]{biblatex}
\usepackage{amsmath}
\usepackage{todonotes}
\usepackage{booktabs}
\usepackage{subcaption}
\usepackage{newfloat}
\usepackage{listings}
\DeclareCaptionStyle{ruled}{labelfont=normalfont,labelsep=colon,strut=off} 
\floatstyle{ruled}
\newfloat{listing}{tb}{lst}{}
\floatname{listing}{Listing}

\newtheorem{definition}{Definition}
\newtheorem{property}{Property}
\newtheorem{corollary}{Corollary}
\newcommand{\TendsTo}[1]{\underset{#1}{\rightarrow}}
\DeclareRobustCommand{\bbone}{\text{\usefont{U}{bbold}{m}{n}1}}

\title{$k$-hop Fairness: Addressing Disparities in Graph Link Prediction Beyond First-Order Neighborhoods}
\date{}

\begin{document}

\author[1]{Lilian Marey}
\author[2]{Tiphaine Viard}
\author[3]{Charlotte Laclau}

\affil[1]{LTCI, Télécom Paris - Deezer Research\\
\texttt{lilian.marey@telecom-paris.fr}}
\affil[2]{i3, Télécom Paris\\
\texttt{tiphaine.viard@telecom-paris.fr}}
\affil[3]{LTCI, Télécom Paris\\
\texttt{charlotte.laclau@telecom-paris.fr}}

\maketitle

\begin{abstract}
Link prediction (LP) plays a central role in graph-based applications, particularly in social recommendation. However, real-world graphs often reflect structural biases, most notably homophily, the tendency of nodes with similar attributes to connect. While this property can improve predictive performance, it also risks reinforcing existing social disparities. In response, fairness-aware LP methods have emerged, often seeking to mitigate these effects by promoting inter-group connections, that is, links between nodes with differing sensitive attributes (e.g., gender), following the principle of \textit{dyadic fairness}. 
However, dyadic fairness overlooks potential disparities within the sensitive groups themselves.
To overcome this issue, we propose \emph{$k$-hop fairness}, a structural notion of fairness for LP, that assesses disparities conditioned on the distance between nodes in the graph. 
We formalize this notion through predictive fairness and structural bias metrics, and propose pre- and post-processing mitigation strategies.
Experiments across standard LP benchmarks reveal: (1) a strong tendency of models to reproduce structural biases at different $k$-hops; (2) interdependence between structural biases at different hops when rewiring graphs; and (3) that our post-processing method achieves favorable $k$-hop performance-fairness trade-offs compared to existing fair LP baselines.
\end{abstract}

\section{Introduction}\label{sec1}

Graphs are a natural way to represent interactions between individuals. In social networks, link prediction (LP) is used to recommend new connections, such as friendships or professional ties. These predictions can influence how people connect and interact, which has a direct impact on users' social or professional opportunities.
Recent work has shown that standard LP methods can amplify existing biases, raising concerns about their fairness and societal impact.
Notable examples include \textit{filter bubbles}~\cite{pariser2011filter}, which reinforce users’ existing preferences and limit exposure to diverse information, and \textit{glass ceiling effect}~\cite{avin2015homophily}, where certain groups face systemic barriers that hinder the formation of new links.
To address these harmful outcomes, several approaches have been proposed under the term \textit{fair LP}~\cite{laclau2024surveyfairnessmachinelearning,chen-survey}. 
These techniques rely on graphs having nodes annotated with sensitive attributes (e.g., race, gender, age), thereby highlighting communities that are potentially subject to such effects. 
While the formalization of fairness in this setting remains an open research question, the dyadic fairness paradigm~\cite{li2021dyadic} has emerged as the canonical framework. 
In this view, the sensitive attribute assigned to nodes is extended to edges, enabling the distinction between intra-group and inter-group links. 
Since real-world graphs exhibit strong homophily, fair predictors aim to counteract the tendency of algorithms to recommend intra-group links by acting on graph structures \cite{li2022fairlp,laclau2021all}, models' training \cite{khajehnejad2022crosswalk,buyl2020debayes}, or resulting predictions \cite{masrour2020bursting} to promote inter-group links.
However, this paradigm appears to address only part of the problem, which is the main focus of this paper. 
Specifically, considering predicted edges solely with respect to the sensitive attributes of the nodes they connect, without accounting for the existing graph structure, overlooks the fact that not all nodes within a community occupy equivalent positions, and therefore should not be treated identically. 
To account for these different node positions, we broaden the scope of dyadic fairness by centering our fairness formulation on the diversity of predictions within nodes \(k\)-hop neighborhoods.
Our approach is further motivated by industrial settings where edge recommendations are categorized by proximity, as $2$ or $3+$-hop suggestions in LinkedIn.
Our key contributions are as follows:
\begin{enumerate}
    \item We introduce $k$-hop fairness, enabling a more comprehensive evaluation of link predictors and graph structures by assessing disparities within nodes' $k$-hop neighborhoods, complemented with an analysis of the limitations of dyadic fairness.
    This framework applies to graphs with multi-valued sensitive attributes and naturally generalizes to weighted and directed networks.
    \item We propose algorithmic methods to quantify and mitigate $k$-hop inequities. These methods can be applied in pre- or post-processing, and are model-agnostic, broadening their applicability across diverse LP settings.
    \item Through extensive experiments on multiple real-world graphs, we investigate the relationship between $k$-hop structural bias and predictive fairness, analyze the interdependence of structural biases at different hops, and demonstrate the effectiveness of our post-processing mitigation method compared to fair LP baselines.
\end{enumerate}
In the following, Section~\ref{sec3} motivates and formalizes the concept of $k$-hop fairness, Section~\ref{sec4} details the methods to compute and mitigate the defined metrics, and Sections~\ref{sec5}--\ref{sec6} present experiments and results.

\section{Related Work}

Fairness is crucial in graph LP, where algorithms often reproduce biases present in graph data~\cite{avin2015homophily,pariser2011filter,baumann2020modeling}.

\subsection{Fair Link Predictors}
To tackle these biases, various fair LP methods have been proposed, which can be categorized based on the stage of the training pipeline they operate on.

\textit{Pre-processing} approaches focus on modifying the graph structure to reduce topological biases before model training. This can involve adjusting edge distributions to balance group representation~\cite{li2022fairlp}, or leveraging optimal transport to align edge weights~\cite{laclau2021all}.
\textit{In-processing} methods integrate fairness constraints directly into the representation learning process. 
Examples include FairWalk~\cite{rahman2019fairwalk}, which incorporates parity constraints into random walk sampling, CrossWalk~\cite{khajehnejad2022crosswalk}, which more finely adapts transition probabilities, and Debayes~\cite{buyl2020debayes}, which debiases embeddings using biased prior distributions.
\textit{Post-processing} techniques operate directly on model predictions, applying corrections to enforce fairness without modifying the underlying model. 
For example, fairness can be improved by adjusting prediction probabilities using adversarial training methods~\cite{masrour2020bursting}.

Several studies also investigate fair node representation learning~\cite{wang2022unbiased,dong2022edits,he2023fairmile,ling2023learning}, particularly in the context of Graph Neural Networks (GNNs). These methods can be adapted to link prediction by incorporating an appropriate prediction head and training with a suitable loss function.

In the aforementioned works, the fairness of designed predictors is evaluated within the framework of dyadic fairness~\cite{li2021dyadic}, following canonical metrics as \textit{demographic parity} (DP)~\cite{feldman2015certifying} and \textit{equal opportunity} (EO)~\cite{hardt2016equality}. Formally, the graph community has adopted the following metrics to assess fairness in link prediction.
\begin{definition}(Dyadic Fairness metrics)
\label{def_dyadic}
    $$\Delta DP = |\mathbb{E}(\hat{Y}|S=S') - \mathbb{E}(\hat{Y}|S\neq S')|,$$
    $$\Delta EO = |\mathbb{E}(\hat{Y}|S=S', Y = 1) - \mathbb{E}(\hat{Y}|S\neq S', Y = 1)|,$$
    where $\hat{Y}, Y$ are the predicted and true edge scores and $S, S'$ refer to nodes' sensitive attributes.
    Thus, dyadic fairness aims to balance inter- and intra-group predictions.
\end{definition}

This framework has only recently been questioned in~\cite{mattos2025breaking}, which highlights several undesirable effects in the context of ranking-based LP.
In this work, we continue along this line of research, arguing that this paradigm disregards the underlying graph structure, treating all inter-group edges as equivalent, and arguing that fairness assessment should rather incorporate the existing topology, particularly by considering the $k$-hop neighborhoods of the nodes involved.

\subsection{Relationship with Graph Structural Bias}

Several works have attempted to incorporate graph structural features into fairness frameworks, such as~\cite{marey2026topofair}.  
For instance, a number of studies use either explicitly or implicitly \cite{rahman2019fairwalk} the concept of \textit{neighborhood fairness}, which can be succinctly expressed as the requirement that neighborhoods should be diverse with respect to sensitive attributes.  
In~\cite{chen2022graph}, this concept is motivated by the observation that neighborhood aggregation is central in GNNs, \textit{“if a neighborhood is unfair, the resulting embeddings will also be unfair”}.  
However, these approaches still evaluate fairness using dyadic metrics, without fully grounding their work in a structural approach.

In parallel, some works have framed fairness on graphs as a problem of information access, thereby adopting a broader view of topology.  
For instance, in~\cite{arnaiz2023structural,jalali2020information}, the authors model structural bias via information flow in random walks, although these approaches are not directly tied to learning tasks.
To the best of our knowledge, \cite{dong2022edits} is the only work in which the fairness of predictors is considered across different neighborhood orders.
Nevertheless, this approach aggregates neighborhoods of several hops into a single measure, with a stronger weighting of close neighborhoods, which limits the granularity of the fairness assessment.  

In this work, we take a different perspective by evaluating fairness at a fixed distance $k>0$ and by explicitly grounding this notion of fairness in the structural bias observed at the same order.

\section{Assessing Fairness beyond Dyadic Framework}
\label{sec3}

\subsection*{Notations}
Let us consider a connected undirected graph $G = (\mathcal{V}, \mathcal{E}, \mathcal{S})$, where $\mathcal{V}$ is the set of nodes, $\mathcal{E} \subset \mathcal{V} \times \mathcal{V}$ is the set of edges, and $\mathcal{S} : \mathcal{V} \rightarrow \{0, 1\}$ is the sensitive attribute function. 
Also, we define the node distance $\sigma : \mathcal{V} \times \mathcal{V} \rightarrow \mathbb{N}$ as the length of shortest paths between two nodes. 
This allows us to define, for $v \in \mathcal{V}, s \in \{0, 1\}$: 
    the set of nodes having sensitive attribute value $s$, $\mathcal{V}_s = \{ v \in \mathcal{V} \mid \mathcal{S}(v) = s \}$;
    the set of $v$'s $k$-hop neighbors $N^{(k)}(v) = \{ v' \in \mathcal{V} | \sigma(v, v') = k \}$; 
    the set of nodes having a non-empty $k$-hop neighborhood $\mathcal{V}^{(k)} = \{ v \in \mathcal{V} | N^{(k)}(v) \neq \varnothing \}$; 
    and the set of nodes having sensitive value $s$ and a non-empty $k$-hop neighborhood $\mathcal{V}^{(k)}_s = \mathcal{V}_s \cap \mathcal{V}^{(k)}$. 
Finally, we denote $\mathcal{K} = \{ k \in \mathbb{N}^{*} | \mathcal{V}^{(k)} \neq \varnothing \}$ the set of valid neighborhood orders in $G$.

Denoting the true edge distribution $r : (v, v') \in \mathcal{V} \times \mathcal{V} \rightarrow \bbone_{\{(v, v') \in \mathcal{E} \}}$ (with $\bbone_X$ being the indicator of event $X$), LP can be formulated as finding $h : \mathcal{V} \times \mathcal{V} \rightarrow [0, 1]$ such that $h$ is close to $r$. 
In this work, we adopt a classification perspective (note that the LP task can also be framed as a ranking problem), thus for $v, v' \in \mathcal{V}$, $h(v, v')$ can be seen as an edge existence probability between the two nodes.

In the following, we work with random variables. Let $(V, V')$ be drawn uniformly from $\mathcal{V}\times\mathcal{V}$, and let $(S, S') = (\mathcal{S}(V), \mathcal{S}(V'))$  denote the sensitive attributes of sampled nodes.


\subsection{Motivations}

Dyadic fairness has gradually established itself as the standard approach for quantifying the fairness in LP. 
One reason is that, by propagating node-sensitive attributes to edges, it aligns naturally with the classical fairness framework. 
However, blindly promoting inter-group edges can be misleading, as it may obscure the underlying topology of the graph.
We detail below the undesirable effects that may arise from this perspective.

\begin{figure}[ht]
    \centering
    \includegraphics[width=.4\linewidth]{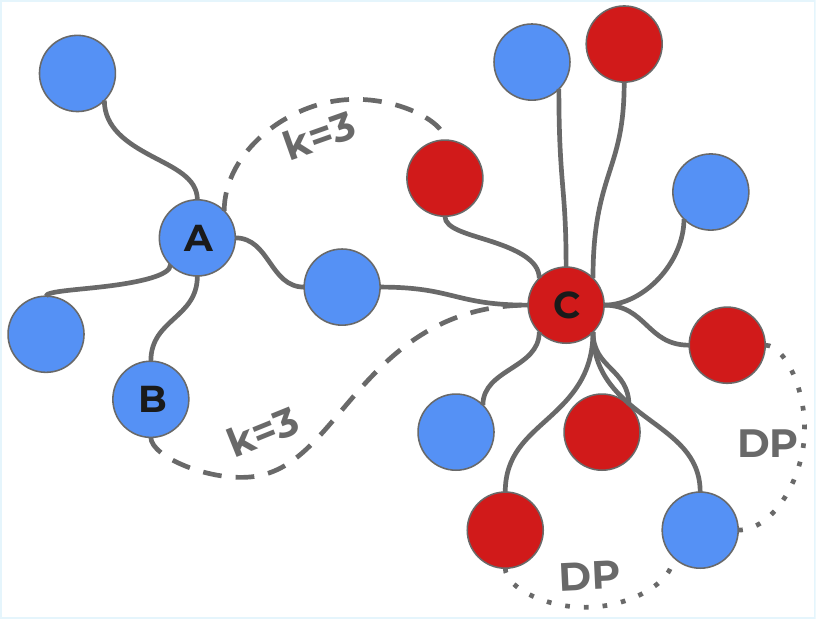}
    \caption{In this example, nodes A and B are segregated within a cluster of blue nodes. While a dyadic view would produce likely inter-edges (typically between nodes with shared neighbors around node C), $k$-hop fairness allows specifying the order at which fairness is desired (here, $k=3$), enabling connections to the segregated part.}
    \label{fig:fig1}
\end{figure}

\paragraph{Dyadic fairness obscures segregation phenomena.}

Consider a setting where a certain level of $\Delta DP$ is required for a link predictor in a homophilic setting. 
To counteract homophily amplification, the model will tend to promote inter-group edges while maximizing accuracy. 
As a result, the model primarily predicts the most likely inter-group edges, that is, in regions of the graph where inter-group edges are already present. 
This reinforces well-connected diverse structures at the expense of segregated parts of the graph. 

For instance, in Figure~\ref{fig:fig1}, where a diverse cluster appears on the right and a segregated cluster of blue nodes on the left, dyadic fairness is likely to drive the prediction of inter-group edges within the right-hand cluster. 
In terms of access to information, this results in only marginal changes. More importantly, it exacerbates intra-group disparities by improving the positions of blue nodes within the most diversified neighborhoods at the expense of the left cluster. 
This echoes well-known issues in which intra-group discrimination is amplified in favor of improved inter-group fairness~\cite{dwork2012fairness,castelnovo2022clarification,subramonian2023networked}.

To address this, we propose conditioning fairness on the distance between nodes in the original graph. 
In Figure~\ref{fig:fig1}, for instance, fairness at $k=3$-hops forces connecting the segregated part of the graph to the diverse cluster.

\paragraph{Dyadic fairness reinforces disparities within neighborhoods.} 

By mapping node sensitive attributes to edges, dyadic fairness may drift from the primary goal of algorithmic fairness, namely protecting sensitive nodes. 
It may be more relevant to instead ask: \textit{how does the local connectivity of individuals differ across sensitive groups?}, suggesting that fairness should be assessed in terms of node neighborhoods rather than just edge types. 
In particular, assuming that each new inter-group edge equally contributes to fairness, dyadic fairness makes no distinction between adding an inter-group edge between nodes whose neighborhoods are already diverse (node C in Figure~\ref{fig:fig1}) and between nodes where this is not the case (node B). 
This allows reinforcing the positions of nodes with already high access to information at the expense of isolated nodes. 
For this reason, we propose refocusing fairness in link prediction on node neighborhood rather than on the type of edge.

\subsection{Limits of Dyadic Fairness Metrics}

Building on the previous discussion, we formalize the notion of $k$-hop fairness by grounding it in the $k$-hop diversity of nodes.

\begin{definition}[$k$-hop node attribute exposure] 
For $s \in \{0, 1 \}, k\in \mathcal{K}, v \in \mathcal{V}^{(k)}$, and for a link predictor $h$, we define
$$f_s^{(k)}(v; h) = \mathbb{E}[h(V, V') \bbone_{ \{ S'=s \}} | \sigma(V, V') = k, V = v].$$     
\end{definition}

Intuitively, $f_s^{(k)}(v; h)$ quantifies how likely node $v$ is to form connections with nodes of sensitive attribute $s$ among those located exactly $k$ hops away. Building on a node-based measure ensures that each node contributes equally to the overall fairness assessment, unlike in the dyadic framework, as we will later show.

We now relate the canonical metric $\Delta DP$, defined in Definition \ref{def_dyadic}, to $k$-hop node attribute exposure $f$.
In the following, we treat the model $h$ as implicit and simply write $f_s^{(k)}(v)$ instead of $f_s^{(k)}(v; h)$.
While $\Delta DP$ is computed over pairs of sampled nodes, we now show that it can be explicitly decomposed in terms of the local terms $f^{(k)}_s$ by conditioning on the distance between the sampled nodes. 

\begin{property}[Decomposition of $\Delta DP$ in terms of $f^{(k)}_s$]\label{prop:deltadp}
\begin{equation*}
    \Delta DP  = \Big| \sum\limits_{k \in \mathcal{K}} \sum\limits_{v \in \mathcal{V}^{(k)}} \omega^{(k)}(v) \Big( \frac{f^{(k)}_{\mathrm{same}}(v)}{\mathbb{P}(S = S')}  - \frac{f^{(k)}_{\mathrm{diff}}(v)}{\mathbb{P}(S \neq S')} \Big) \Big|,
\end{equation*}
where $f^{(k)}_{\mathrm{same}}(v)\triangleq f^{(k)}_{\mathcal{S}(v)}(v)$ and $f^{(k)}_{\mathrm{diff}}(v)\triangleq f^{(k)}_{1 - \mathcal{S}(v)}(v)$ are the within- and cross-group exposure, respectively; and $\omega^{(k)}(v) = \mathbb{P}(V = v,  \sigma(V, V') = k)$.
\end{property}

The detailed derivation of this decomposition is provided in Supplementary B.
From this property, we note that although $\Delta DP$ is defined over node pairs, it aggregates contributions across all graph distances $k$ and all nodes, weighted by $\omega^{(k)}(v)$, which captures how frequently node $v$ appears in pairs separated by distance $k$.
As a result, structural effects arising at different hops, such as local homophily or long-range segregation, are mixed up into a single scalar, preventing the identification of the graph distances responsible for bias.
Moreover, the term $\omega^{(k)}(v)$ can overemphasize the influence of a small number of structurally central nodes, such as inter-group bridges, even when their connectivity patterns are not representative of their group as a whole.

In LP, where predicted interactions are directly constrained by graph structure, this lack of granularity limits both interpretability and targeted mitigation, as shown in Figure~\ref{fig:fig1}.

\subsection{Defining Fairness at Fixed Graph Distance}
To address this limitation, we adopt a hop-specific perspective and define fairness explicitly at a fixed graph distance $k$. 
In the following, we provide definitions that extend beyond the binary case and remain valid for sensitive attributes taking multiple values $s \in \mathbb{S}$.

\paragraph{From local to group level exposure}
The local quantity $f_s^{(k)}(v)$ captures how a node $v$ is predicted to connect, at distance $k$, to nodes of sensitive group $s$. To reason about fairness at the group level, we aggregate these local exposures across nodes sharing the same sensitive attribute.

\begin{definition}[$k$-hop group exposure]\label{def:khop-group}
For $ s, s' \in \mathbb{S},  k \in \mathcal{K}$ and an edge predictor $h$
    \begin{equation*}
        \phi_{s \rightarrow s'}^{(k)}(h) = \frac{1}{|\mathcal{V}^{(k)}_s|} \sum\limits_{v \in \mathcal{V}^{(k)}_s} f_{s'}^{(k)}(v; h)
    \end{equation*}
\end{definition}
The quantity $\phi_{s \rightarrow s'}^{(k)}$ can be interpreted as a group-to-group exposure measure,
capturing how nodes from sensitive group $s$ are collectively predicted to access nodes from group $s'$ at graph distance $k$. Importantly, this aggregation assigns equal weight to each node in $\mathcal{V}^{(k)}_s$, thereby characterizing a \emph{typical} exposure pattern rather than one dominated by structurally central nodes.

\paragraph{k-hop fairness}

At a fixed graph distance $k$, fairness can then be assessed by comparing how evenly nodes from the same source group are exposed to different sensitive groups.

\begin{definition}[$k$-hop neighborhood fairness]
\label{def:NF}
$\forall k \in \mathcal{K}$ and for an edge predictor $h$, we define the $k$-hop fairness gap as
    \begin{equation*}
        NF^{(k)}(h) = \max\limits_{s \in \mathbb{S}}
        \max\limits_{s_1 \neq s_2}|\phi_{s_1 \rightarrow s}^{(k)}(h) - \phi_{s_2 \rightarrow s}^{(k)}(h)|
    \end{equation*}
\end{definition}


Thus, a small $NF^{(k)}$ indicates balanced access across groups at hop $k$, while larger values reveal structural disparities in predicted connectivity.

Finally, this definition naturally induces a purely structural perspective on fairness. When the link predictor $h$ is replaced by the observed graph, that is, when the definition is evaluated on the $k$-hop indicator function, the resulting quantity reduces to a measure of diversity in the existing $k$-hop neighborhoods of the original graph. 

\begin{definition} 
For $k \in \mathcal{K}$, the $k$-hop structural bias of a graph $G$ is the fairness level of its $k$-hop indicator function:

$$NB^{(k)}(G) = \max\limits_{s \in \mathbb{S}} \max\limits_{s_1 \neq s_2} |\phi^{(k)}_{s_1 \rightarrow s}(r^{(k)}) - \phi^{(k)}_{s_2 \rightarrow s}(r^{(k)})|,$$
where $r^{(k)}(v, v') = \bbone_{ \{ \sigma(v, v') = k \}}$.
\end{definition}

In other words, instead of relying on the probability of connection, $NB^{(k)}$ focuses on sensitive attribute proportions in actual \( k \)-hop neighborhoods of the same group nodes.
In this context, an unbiased graph is one where each sensitive group has \( k \)-hop neighborhoods with similar proportions of sensitive attributes.
This raises the question of how structural bias at order $k$ aligns with the fairness of link predictors evaluated at the same order, which will later be explored.

In Supplementary A, we further characterize $NB^{(k)}$ through its relationship with assortativity, a standard measure of homophily in graphs, following the same analysis as for $\Delta DP$ and $NF^{(k)}$.
In addition, we show that all definitions introduced here naturally extend to weighted and directed graphs, and we detail these extensions in Supplementary C.

\section{Methodology: Computing and Controlling Bias and Fairness}
\label{sec4}

In this section, we outline the key challenges involved in measuring and controlling the disparity metrics defined above.
We introduce the main computational tools and corresponding algorithmic complexities, while leaving detailed methods and proofs to Supplementary C and D. 
First, we focus on assessing the $k$-hop structural bias of a graph $G$ and the $k$-hop fairness of a link predictor $h$ trained on this graph. 
Second, we present mitigation strategies based on the minimization of dedicated loss functions.

Here, we consider a fixed graph $G$ with $n$ nodes and $m$ edges, and a link predictor $h$ trained on $G$. 
In the following, $\|\cdot\|_F$ denotes the Frobenius norm.
All complexity results are derived under the standard sparsity assumption ($m \ll n^{2}$).

\subsection{Measuring k-hop Fairness and Structural Bias}
\label{sec:measure}
A primary use of the introduced metrics is to provide a diagnosis of a graph's structural bias or of a predictor's fairness. 
To this end, we present a general methodology for computing these quantities for a given graph and predictor.

As stated in Definitions~\ref{def:NF}, $NF^{(k)}(h)$ is computable from $\phi_{s \rightarrow s'}^{(k)}(h)$,  $\forall s, s' \in \mathbb{S}$.
By definition, and in its empirical version, this is given by the average over $v \in \mathcal{V}^{(k)}_s$ of $k$-hop group exposure $f^{(k)}_s(v; h)$ (Definition~\ref{def:khop-group}):
\[
\frac{1}{|N^{(k)}(v)|} \sum_{v' \in N^{(k)}(v)} h(v, v') \, \bbone_{\{S(v') = s\}}.
\]

We state that the quantities here involved are accessible in constant time, except for the $k$-hop neighborhood $N^{(k)}(v)$ of node $v$, whose computation over all nodes dominates the cost of evaluating $NF^{(k)}$. 
To encode this information for every node pair, we introduce the \textit{$h$-hop indicator matrix} specifying whether two nodes are at distance $k$,
\[
(\tilde{A}^{(k)})_{1 \leq i, j \leq n} \triangleq \bbone_{\{\sigma(v_i, v_j) = k\}},
\]
which will then allow us, through matrix operations, to recover the desired quantities.

To compute $\tilde{A}^{(k)}$, we perform a breadth-first search up to depth $k$ from every node in the graph, which yields the following property:

\begin{property}
For a graph $G$, for a link predictor $h$ and for $k \in \mathcal{K}$,
computing $NF^{(k)}(h)$ and $NB^{(k)}(G)$ is $\mathcal{O}(n^2)$.
\end{property}

\subsection{Controlling Fairness and Structural Bias}
\label{sec:mitigation}
A second use of the introduced metrics is to control bias levels, either through graph modifications or prediction adjustments, in order to mitigate undesired effects.

Thus, we propose mitigation methods for the biases defined above, relying on optimization procedures, using the Adam optimizer. 
In the losses we introduce, a $k$-hop bias term (either structural or predictive) is balanced with a fidelity term, allowing us to control the usual fairness–fidelity (or fairness–performance) trade-off.

\paragraph{Post-processing link predictions}

Let $P\in [0,1]^{n\times n}$ denote the edge probability matrix issued from predictor $h$, with $P_{ij}=h(v_i, v_j)$. 
In what follows, we will use $NF^{(k)}(h)$ and $NF^{(k)}(P)$ interchangeably.
In order to post-process predictions at $k$, we restrict modifications to node pairs that are exactly $k$-hops apart. 
We parametrize the post-processed $P'$ as 
$$P'=\Pi_{[0,1]}(P+U\odot \tilde{A}^{(k)}),$$
where $U\in \mathbb{R}^{n\times n}$ is the optimization variable and $\Pi_{[0,1]}$ denotes element-wise clipping to $[0,1]$. We then solve 
$$\min_{U \in \mathbb{R}^{n\times n}} \ell_{post}^{(k)}=NF^{(k)}(P')+\alpha||U\odot \tilde{A}^{(k)}||_F.$$

The regularization term $\|U \odot \tilde A^{(k)}\|_F$ controls the magnitude of the post-processing adjustments,
thereby preserving fidelity to the original predictor and preventing degenerate solutions.
It also contributes to the stability of the optimization, given the non-smooth nature of $NF^{(k)}$.

\textbf{Remark 2: } In this context, $\tilde{A}^{(k)}$ is independent of optimization variable $U$. In other words, the underlying graph structure supporting the computation of $NF^{(k)}$ remains fixed throughout the iterations.
This implies that once $\tilde{A}^{(k)}$ has been computed and stored, evaluating $\ell_{post}^{(k)}$ has negligible cost. Therefore, the overall complexity of the optimization is thus governed by that of the chosen optimizer.

\paragraph{Pre-processing graphs}
\label{rewiring}
To act on the structural bias of graph $G$, we operate directly on its adjacency matrix $A$. 
As before, the reduction of $k$-hop structural bias can be balanced with fidelity to the original graph through a coefficient $\alpha \in \mathbb{R}$. 
To this end, we define the following generic loss on a candidate adjacency matrix $A' \in \mathbb{R}^{n \times n}$:
\[
\ell_{\text{pre}}^{(k)}(A') = NB^{(k)}(A') + \alpha \|A' - A\|_F,
\]
where $NB^{(k)}(A')$ denotes the $k$-hop structural bias evaluated on the graph induced by the adjacency matrix $A'$.

This loss is versatile and can be adapted to various learning processes and constraint settings. 
For instance, by minimizing $\ell_{\text{pre}}^{(k)}$ while constraining  $A'$'s coefficients to $[0, 1]$, one obtains a weighted and directed graph that has lower bias. 
Additional constraints can be enforced depending on the target use case: for example, imposing symmetry on $A'$ at each iteration preserves the undirected nature of the graph, while restricting optimization to the non-zero entries of $A$ ensures that no new edges are introduced. 
Alternatively, one may iteratively identify and update the zero values of $A'$ having the lowest gradient, allowing for edge addition rewiring.

\textbf{Remark 3: } In this context, as we are rewiring the graph, the underlying $k$-hop structure is involved in the optimization process.
Consequently, it is necessary to provide a method to compute $\tilde{A}^{(k)}$ that is differentiable with respect to the optimization variable $A'$, which is not the case for the previously used breadth-first search method. 
To this end, we exploit the fact that successive powers of the adjacency matrix encode the number of walks of fixed length connecting each pair of nodes~\cite{biggs1993algebraic}.
This enables us to compute $\tilde{A}^{(k)}$ using gradient-preserving operations, and thus to establish the following property.

\begin{property}
For a graph $G$, $\alpha \in \mathbb{R}$, and $k \in \mathcal{K}$, the cost of one iteration in the minimization of $\ell_{\text{pre}}^{(k)}$ is $\mathcal{O}(k n^{2})$.
\end{property}

\section{Experiments}
\label{sec5}

In the previous section, we introduced computational tools to measure and control $k$-hop graph bias and predictor fairness, based on the definitions proposed in Section~\ref{sec3}. 
This framework enables us to investigate the core challenges and implications of $k$-hop fairness, which we explore through a series of experiments detailed below.

To structure our analysis, we organize our experiments around the following research questions:

\begin{itemize}
    \item \textbf{RQ1}: What are the relationships between $k$-hop structural bias and fairness at specific distances $k$ in standard LP models?
    \item \textbf{RQ2}: How does minimizing \(k\)-hop structural bias through a pre-processing strategy affect bias levels at other neighborhood orders?
    \item \textbf{RQ3}: To what extent does the post-processing strategy improve $k$-hop fairness of link predictors, and how does this affect predictive performance?
\end{itemize}

\paragraph{Standard models} 
To address these research questions, we employ classical link prediction algorithms: the random walk-based Node2Vec~\cite{grover2016node2vec} followed by logistic regression (\textbf{N2V}), a graph convolutional network designed for link prediction~\cite{kipf2016semi} (\textbf{GCN}), which has demonstrated strong performance across various real-world graphs~\cite{delarue2024link}, and GraphSAGE~\cite{hamilton2017inductive} (\textbf{SAGE}).

\paragraph{Fair baselines} To evaluate the effectiveness of our $k$-hop fairness mitigation method, we compare it against established fair LP baselines: \textbf{FairWalk}~\cite{rahman2019fairwalk}, \textbf{CrossWalk}~\cite{khajehnejad2022crosswalk}, \textbf{DeBayes}~\cite{buyl2020debayes}.
We also integrate fair representation learning baselines, namely \textbf{UGE}~\cite{wang2022unbiased}, \textbf{EDITS}~\cite{dong2022edits}, \textbf{FairMILE}~\cite{he2023fairmile}, \textbf{Graphair}~\cite{ling2023learning}.
For these baselines, we obtain fair node embeddings and train a logistic regression classifier on edge representations (computed as the Hadamard product of node vectors), ensuring identical training conditions.

\paragraph{Datasets} We evaluate our methods on diverse real-world graphs. \textbf{Polblogs}~\cite{adamic2005political} is a political blog network labeled by political affiliation. \textbf{Pokec} and \textbf{Facebook}~\cite{snapnets} are ego-networks from social platforms, with region and gender as sensitive attributes, respectively. \textbf{Citeseer}~\cite{caragea2014citeseer} is a co-authorship graph labeled by research field. We also include a \textbf{Synthetic} graph with three sensitive attributes, generated via a three-block stochastic block model~\cite{holland1983stochastic}.
For all these datasets, we reduce  $\mathcal{K}$ to the set of $k \in \mathbb{N}^*$ such that at least half of the nodes have a non-empty $k$-hop neighborhood, retaining meaningful neighborhood orders. 

\section{Results}
\label{sec6}

We now present the experimental results corresponding to the directions outlined in the previous section.
Models are trained on 80\% of the edges and an equal number of sampled negative edges, leaving the remaining 20\% for testing.  
$NB^{(k)}$ and $NF^{(k)}$ are computed on the full graph for evaluation, while post-processing mitigation relies on the training graph, avoiding leakage. To improve clarity and readability, we include only key results in the main paper. The full experimental results can be found in the Supplementary E.

\subsection{Relationship Between k-hop Structural Bias and Fairness (RQ1)}
\label{sec:rq1}
To characterize graph and predictor biases, (see Section~\ref{sec:measure}), we compute $NB^{(k)}$ on the observed graphs, as well as $NF^{(k)}$ for classical link predictors, across multiple $k$-hop neighborhoods and datasets.

\begin{figure}[t]
    \centering
    \begin{subfigure}{0.22\textwidth}
        \includegraphics[width=\linewidth]{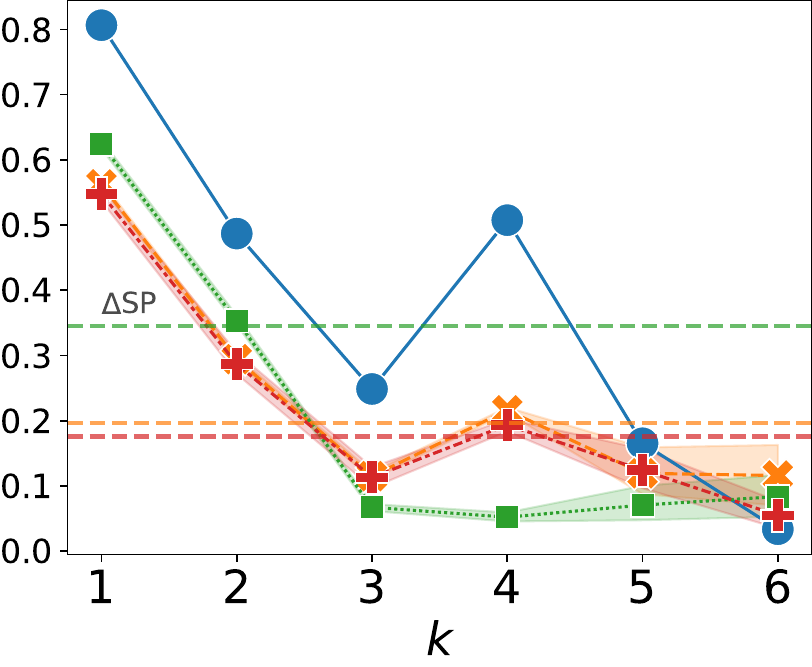}
        \caption{Polblogs}
    \end{subfigure}
    \hfill
    \begin{subfigure}{0.22\textwidth}
        \includegraphics[width=\linewidth]{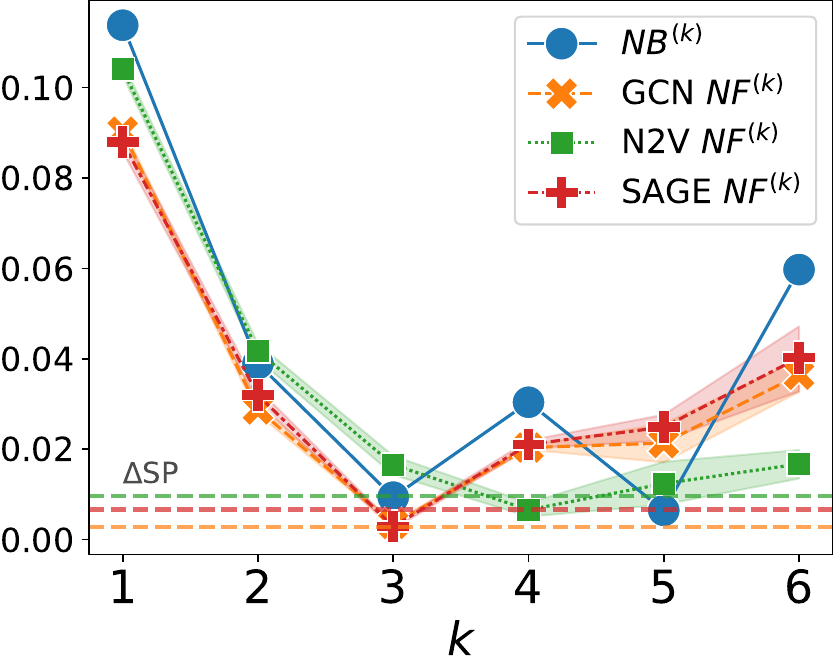}
        \caption{Facebook}

    \end{subfigure}
    \begin{subfigure}{0.22\textwidth}
        \includegraphics[width=\linewidth]{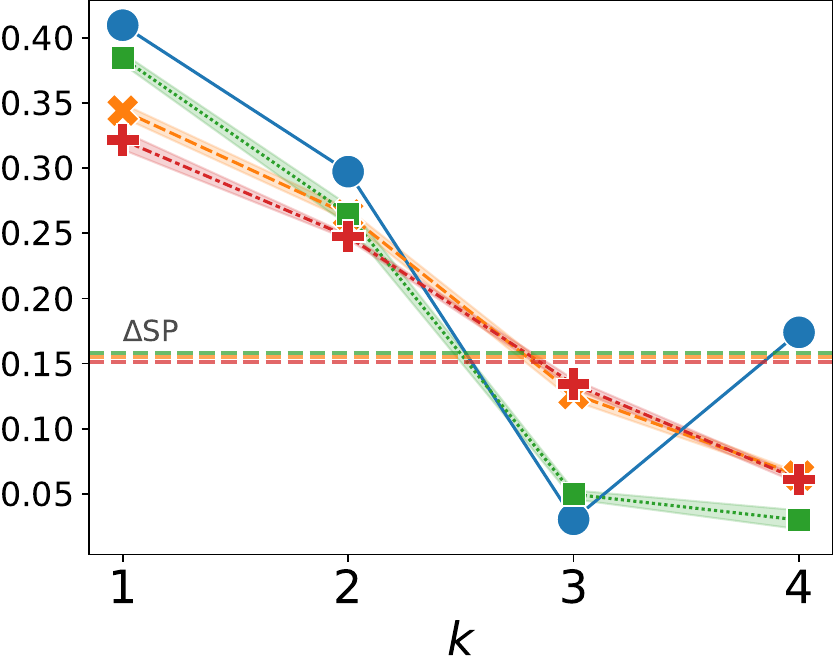}
        \caption{Pokec}
    \end{subfigure}
    \hfill
    \begin{subfigure}{0.22\textwidth}
        \includegraphics[width=\linewidth]{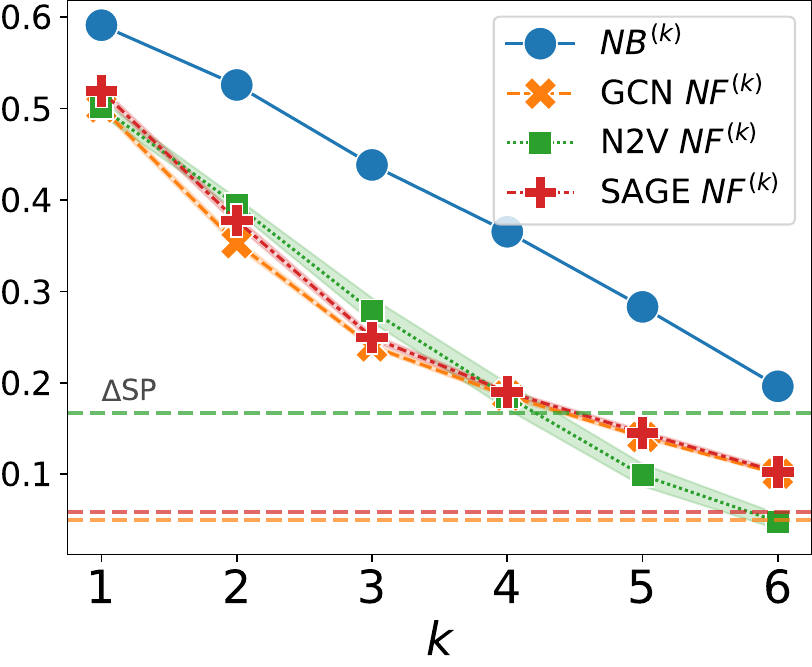}
        \caption{Citeseer}
    \end{subfigure}
    \caption{Comparison between graph structural bias and classical LP models $k$-hop fairness across meaningful hops.}
    \label{fig3}
\end{figure}

Figure~\ref{fig3} reveals that, in all presented graphs, the $1$-hop bias is consistently higher than the others, which supports the idea that local neighborhoods embed the primary source of bias.  
However, other significant biases occur at specific hops (such as at $4$-hop in Polblogs), highlighting the relevance of $k$-hop fairness for $k>1$.  
The variability in bias profiles across datasets reflects the diversity of topologies that can exist in real-world graphs, suggesting that uniformly aggregating disparities over hops as in~\cite{dong2022edits} is not a desirable strategy.  
Moreover, we find that $NB^{(k)}$ and $NF^{(k)}$ tend to match: in most cases, the hops with the highest structural bias correspond to the hops where fairness is lowest, and conversely.  
This both validates the relevance of $NB^{(k)}$ as a structural proxy for $NF^{(k)}$, and shows that unfair hops are traceable through structural analysis.  
Accordingly, for each dataset, we select as target hops the three most structurally biased hop distances in $\mathcal{K}$.


\subsection{Dependencies Across Different Hops (RQ2)}
\label{sec:rq2}

Here, we adopt a graph rewiring approach to mitigate structural biases across different hop distances following the edge-addition process described in Section~\ref{rewiring}. In particular, we examine how minimizing bias at a given hop affects the biases observed at other orders of the graph.
For a fixed $k \in \mathcal{K}$, we compute $\ell_{\text{pre}}^{(k)}$ on the adjacency matrix, with $\alpha = 0$. At each iteration, we identify the non-edge that most reduces $NB^{(k)}$ by selecting the coordinate with the lowest gradient value, and add the corresponding edge to the graph.
By repeating this procedure, we track the structural biases at other hops across iterations. 
The influence of mitigating $NB^{(k)}$ on other biases is then quantified using Pearson correlation coefficients between the corresponding bias trajectories.

\begin{figure}[t]
    \centering
    \begin{subfigure}{0.215\textwidth}
        \includegraphics[width=\linewidth]{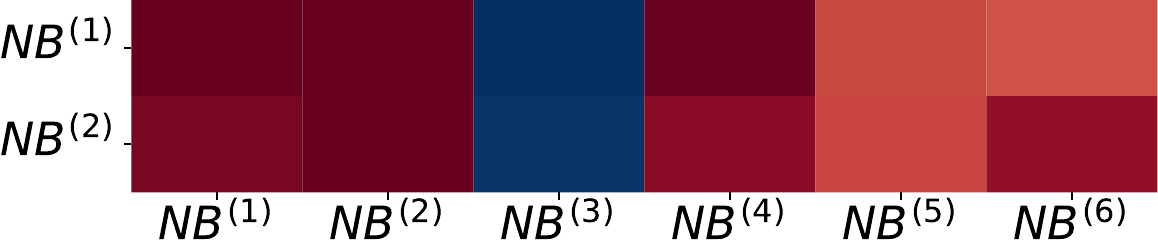}
        \caption{Polblogs}
    \end{subfigure}
    \hfill
    \begin{subfigure}{0.25\textwidth}
        \includegraphics[width=\linewidth]{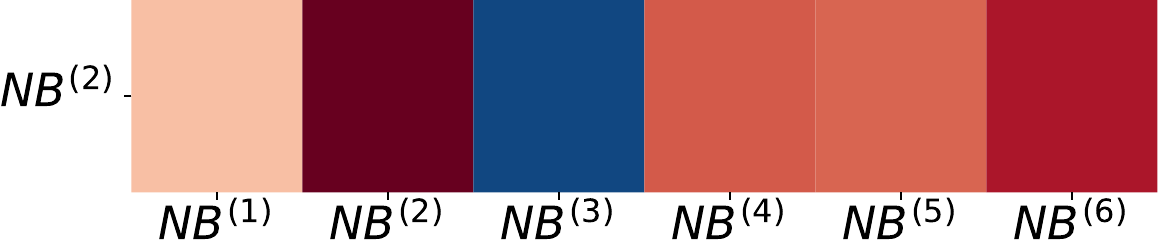}
        \caption{Facebook}
    \end{subfigure}
    \begin{subfigure}{0.22\textwidth}
        \includegraphics[width=\linewidth]{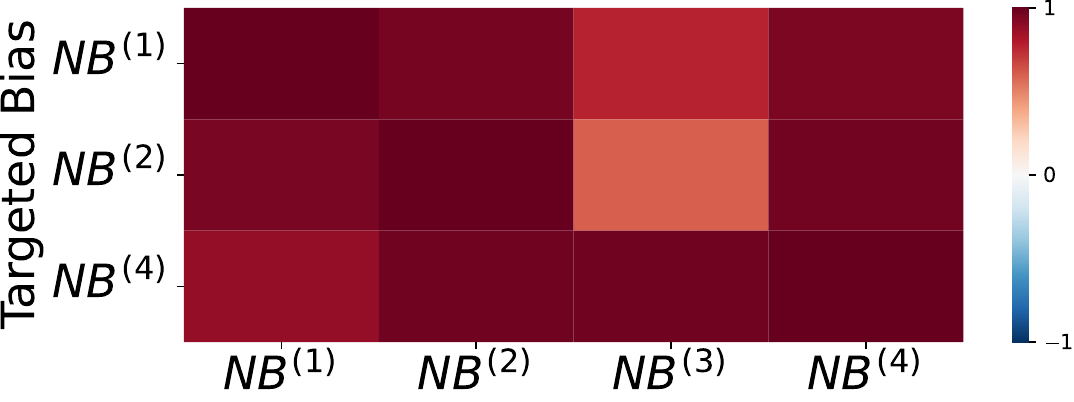}
        \caption{Pokec}
    \end{subfigure}
    \hfill
    \begin{subfigure}{0.24\textwidth}
        \includegraphics[width=\linewidth]{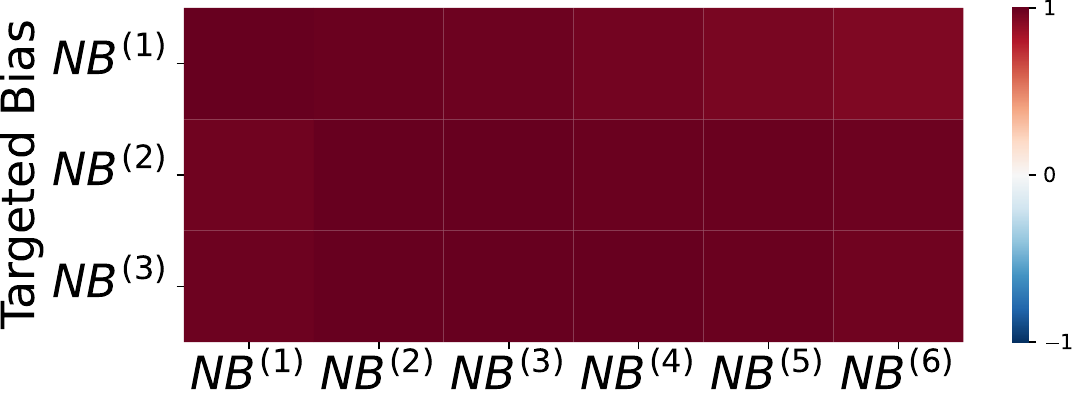}
        \caption{Citeseer}
    \end{subfigure}
    \caption{Influence of $NB^{(k)}$ minimization via edge addition on structural bias at other hops. All p-values reported were found to be below 0.01, indicating statistical significance. The evolution of biases over iterations is shown in Supplementary E.2. For some target $k$, the optimization process did not identify edges that reduce bias, these cases were excluded from the analysis.}
    \label{fig4}
\end{figure}

As shown in Figure~\ref{fig4}, the edge addition procedure reveals strong interdependencies between structural biases across hop distances. Consequently, bias mitigation at a single hop propagates to other structural scales of the graph. These interactions can be either reinforcing or compensatory. For instance, on Polblogs, reducing $NB^{(2)}$ also decreases $NB^{(1)}$ but increases $NB^{(3)}$, illustrating that effective multi-bias mitigation require data-specific strategies.


\begin{figure}[ht]
    \centering
    \begin{subfigure}{0.275\textwidth}
        \includegraphics[width=\linewidth]{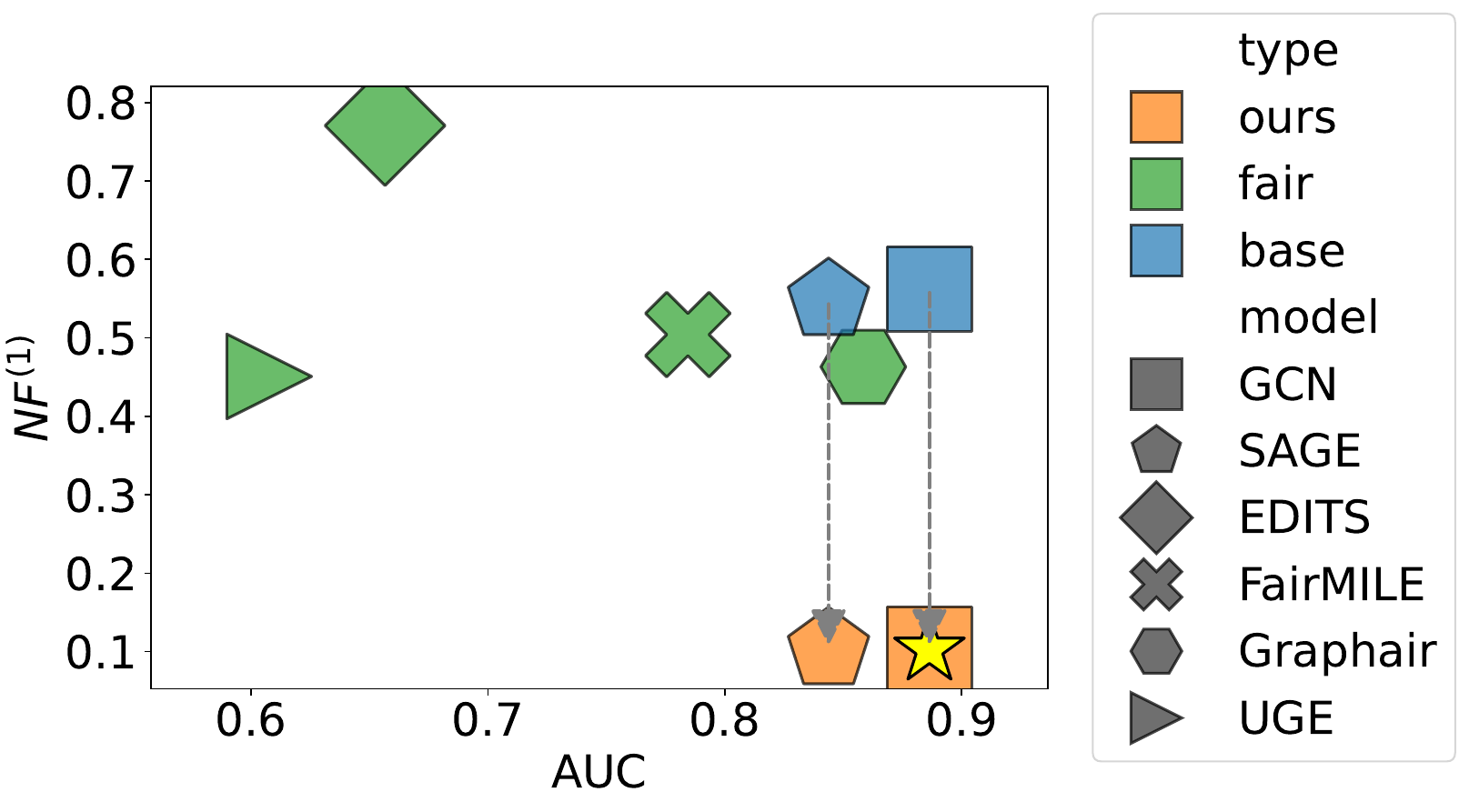}
    \end{subfigure}
    \hfill
    \begin{subfigure}{0.215\textwidth}
        \includegraphics[width=\linewidth]{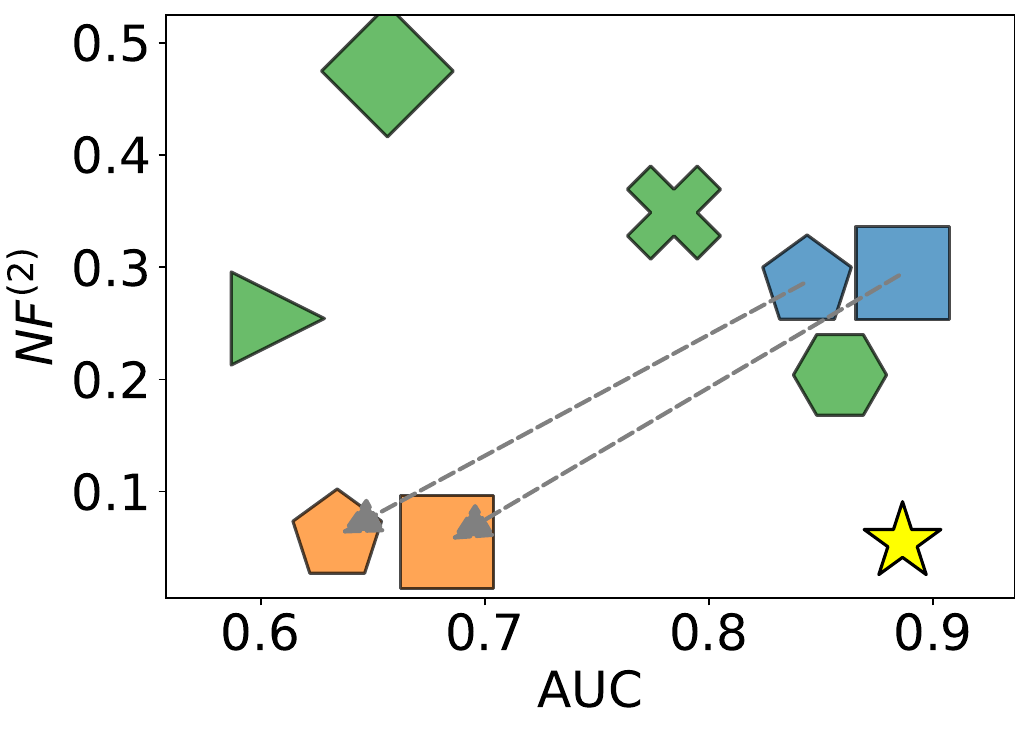}
    \end{subfigure}
    \hfill
    \begin{subfigure}{0.215\textwidth}
        \includegraphics[width=\linewidth]{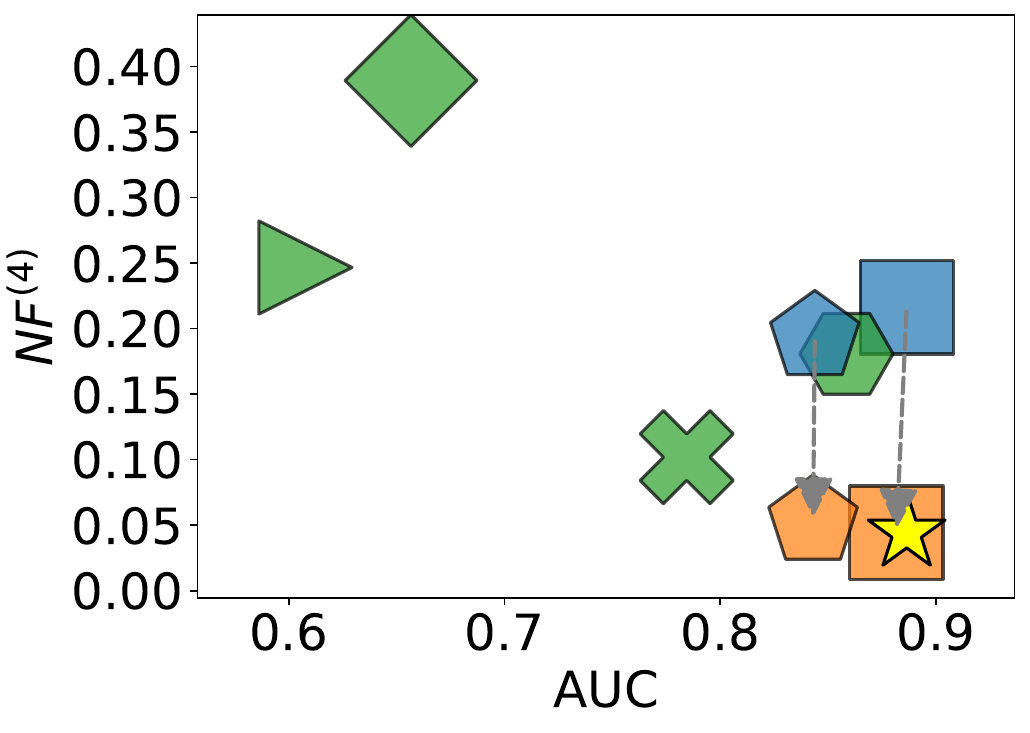}
    \end{subfigure}
    \hfill
    \begin{subfigure}{0.215\textwidth}
        \includegraphics[width=\linewidth]{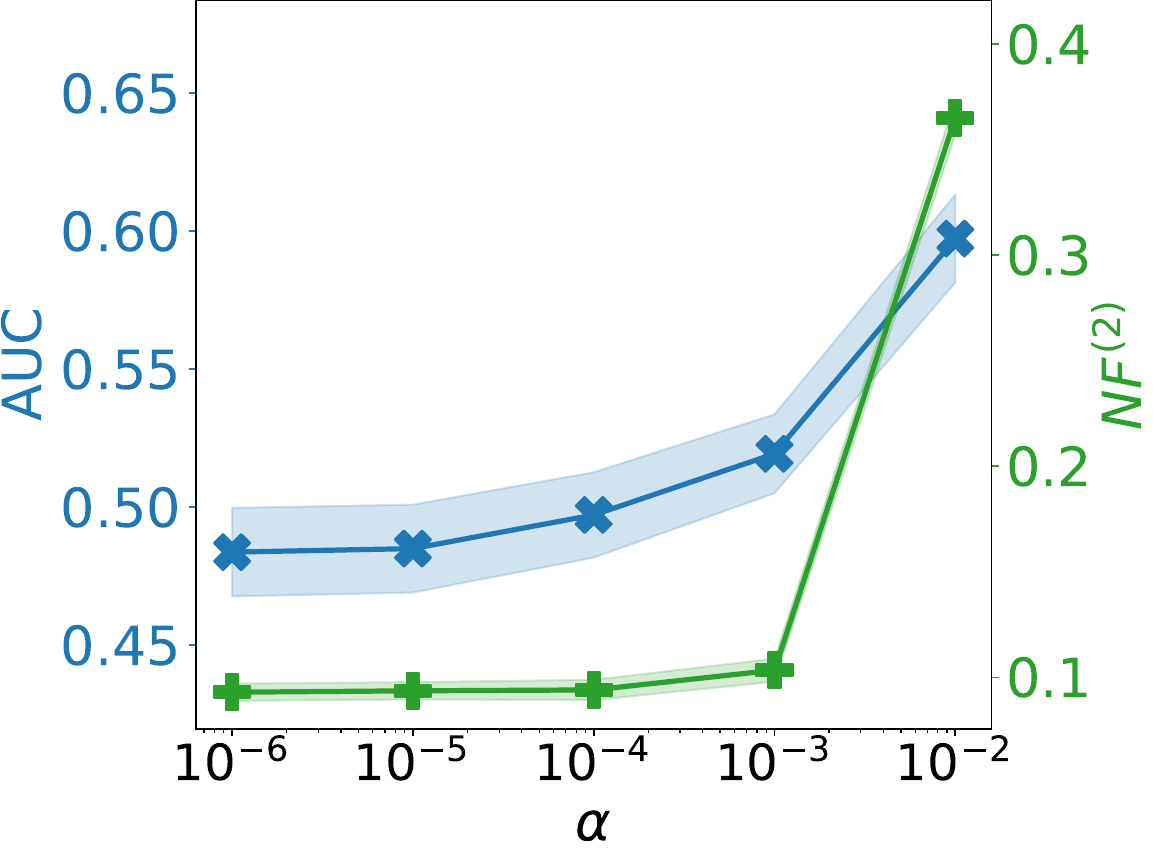}
    \end{subfigure}
    \hfill
    \begin{subfigure}{0.275\textwidth}
        \includegraphics[width=\linewidth]{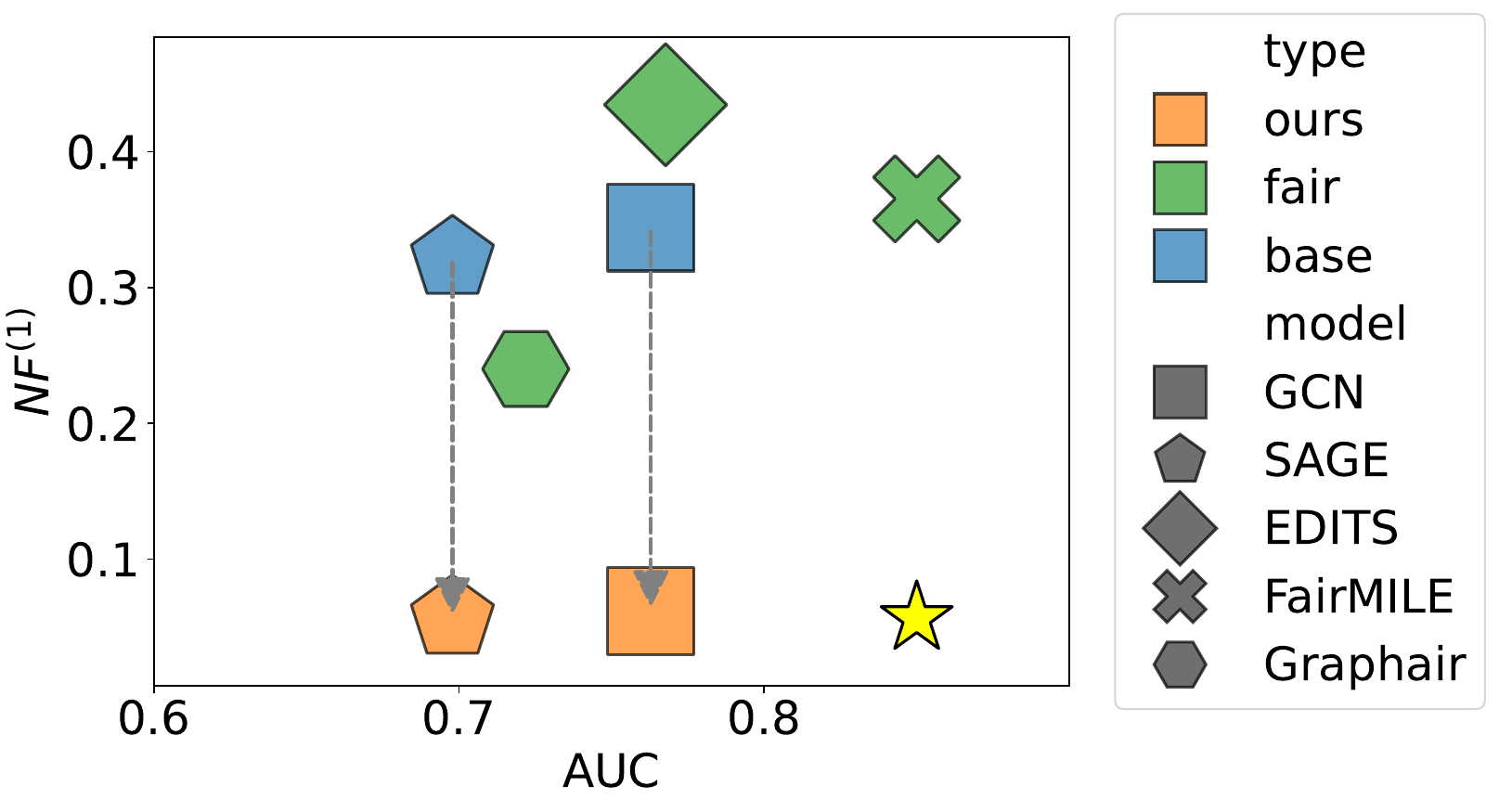}
        \caption{Polblogs $\uparrow$ Pokec $\downarrow$ - $k$ = 1}
    \end{subfigure}
    \hfill
    \begin{subfigure}{0.215\textwidth}
        \includegraphics[width=\linewidth]{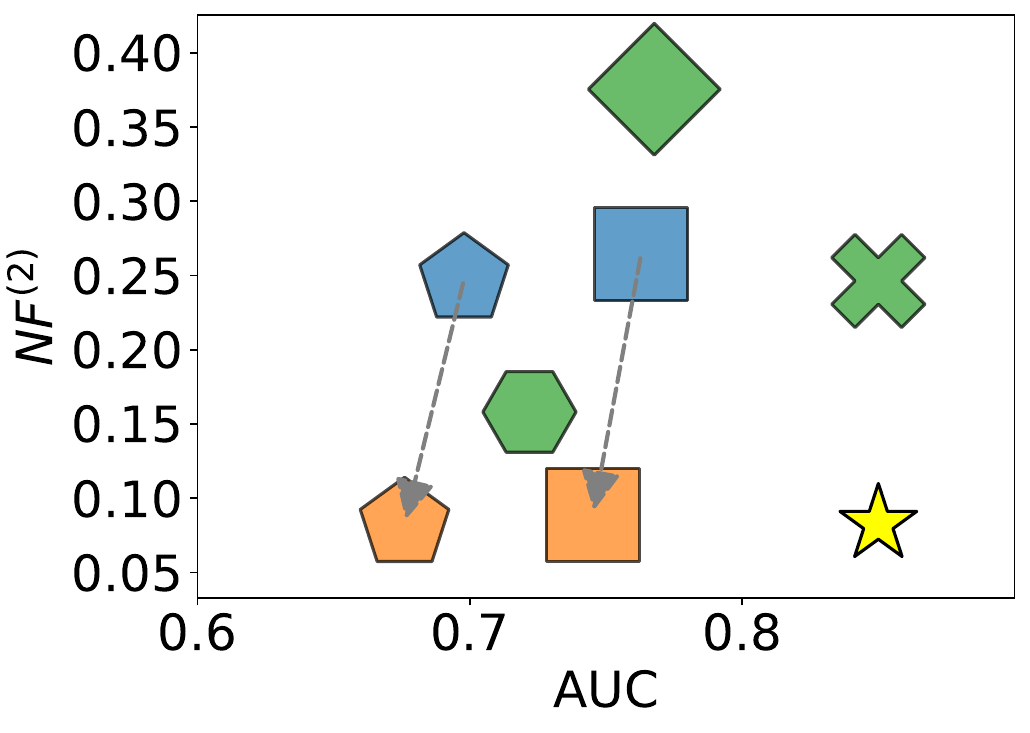}
        \caption{Polblogs $\uparrow$ Pokec $\downarrow$ - $k$ = 2}
    \end{subfigure}
    \hfill
    \begin{subfigure}{0.215\textwidth}
        \includegraphics[width=\linewidth]{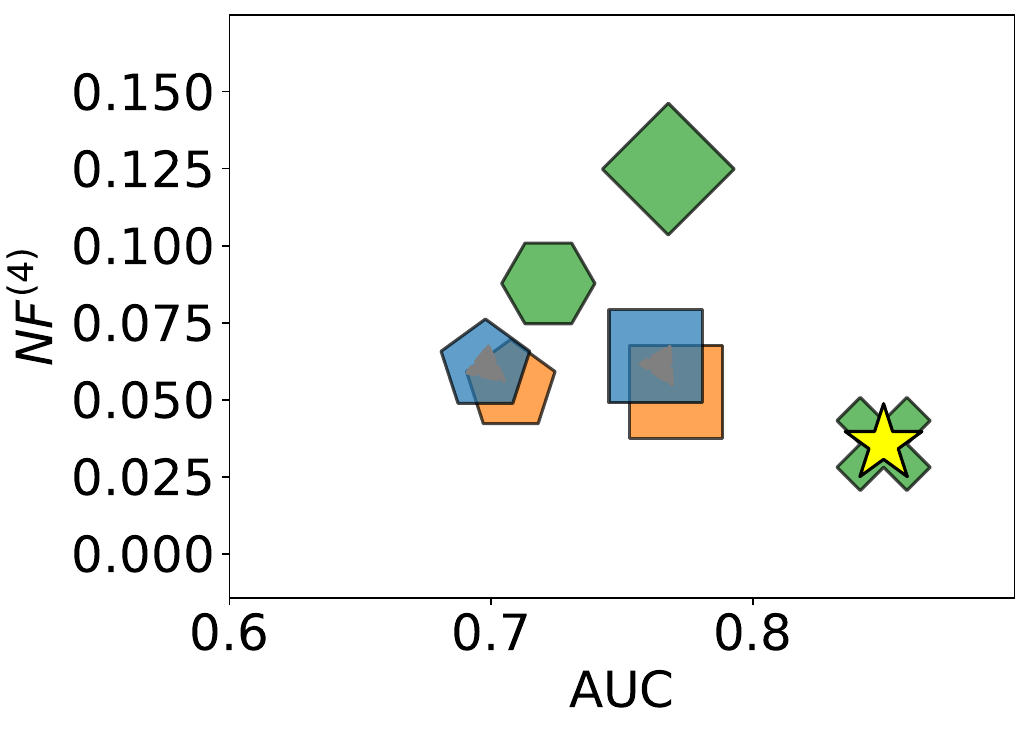}
        \caption{Polblogs $\uparrow$ Pokec $\downarrow$ - $k$ = 4}
    \end{subfigure}
    \hfill
    \begin{subfigure}{0.215\textwidth}
        \includegraphics[width=\linewidth]{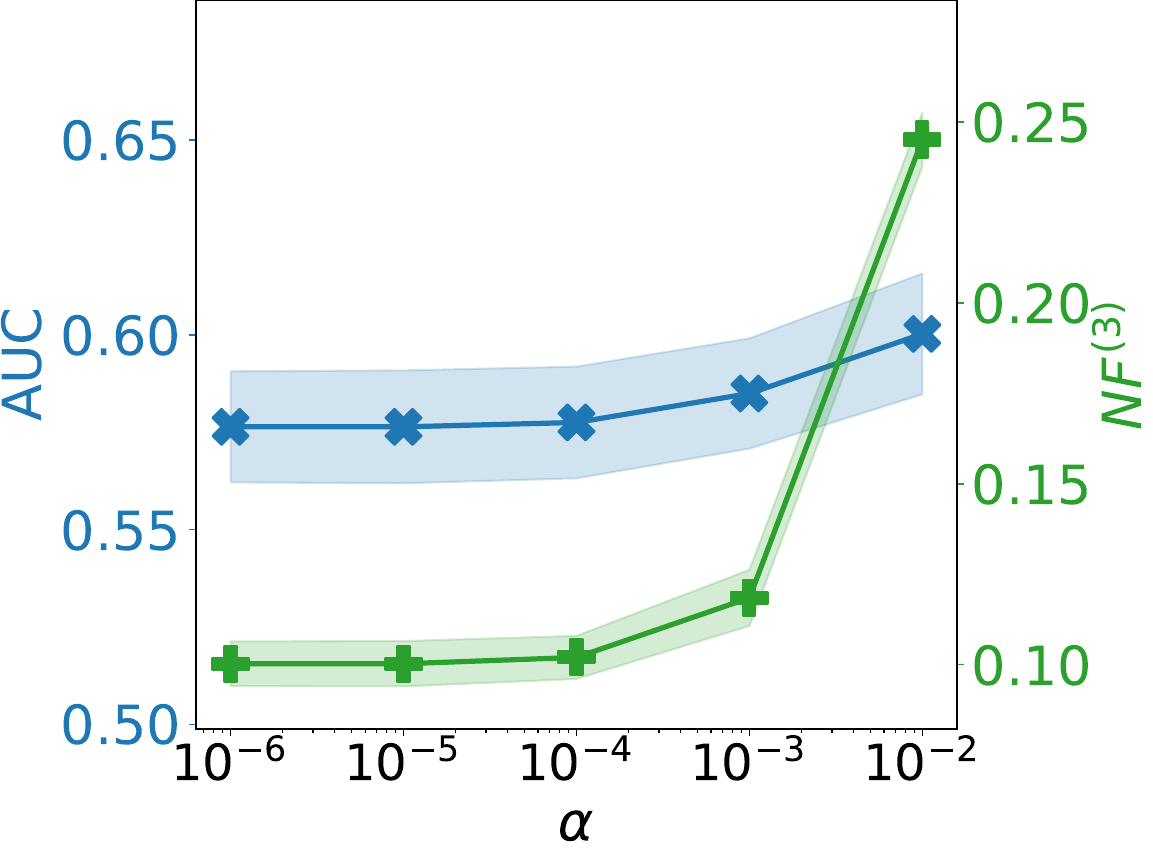}
        \caption{Citeseer - $k$ = 2 $\uparrow$, k = 3 $\downarrow$}
    \end{subfigure}
    \hfill
    \caption{AUC vs $NF^{(k)}$ plot for base and fairness-aware GNN-based baselines and for targeted $k$ values. Our post-processing results are shown through arrows from base models. The star marker denotes the hypothetical model achieving the best AUC and the best $NF^{(k)}$ across baselines. Only baselines with AUC $> 0.55$ are displayed. The outcome of our post-processing approach on the base models are represented with an arrow.  Figure (d) shows $\alpha$ effect in the post-processing method for \textbf{SAGE} predictions.}
    \label{fig:rq3}
\end{figure}

\subsection{Post-processing for k-hop Fair LP (RQ3)}
\label{sec:rq3}

We now target predictive bias $NF^{(k)}$ using the post-processing strategy introduced in Section~\ref{sec:mitigation}.
For each dataset, we train base link prediction (LP) models alongside fairness-aware baselines, and evaluate them using AUC, dyadic metrics  (DP, EO), and  $NF^{(k)}$ for $k \in \mathcal{K}$. 
For each standard model, we then apply our post-processing strategy on the selected target hops, using 500 training epochs and a learning rate of 0.01.
The parameter $\alpha$ is chosen via grid search to optimize a weighted trade-off between AUC and $NF^{(k)}$.
Results are averages over 10 random train/test splits. Figures~\ref{fig:rq3}(a--c) report the $NF^{(k)}$ versus AUC for all baselines. 

First, regardless of whether models are standard or fairness aware, none explicitly target $NF^{(k)}$ at any order, confirming that $k$-hop predictive biases are not addressed by existing dyadic approaches.
In contrast, our post-processing method consistently reduces the targeted $NF^{(k)}$, with trade-offs on AUC that depend on the hop order. At $k=1$, our method incurs no loss in predictive performance (AUC), while achieving higher-hop fairness becomes more challenging. This is explained by the post-processing mechanism: since $NF^{(1)}$ only involves training edges, modifying their scores does not affect the test prediction used to compute AUC. At intermediate orders, especially $k=2$, $NF^{(k)}$ involves many node pairs that overlap with test edges due to the high local clustering of real-world graphs~\cite{watts1998collective}. As a result, mitigation directly alters test scores and can substantially degrade performance. For larger $k$, test edges are less likely to be involved, leading to limited performance degradation and, in some cases, slight improvements in AUC. This observation is consistent with prior work in fairness-aware representation learning, where removing biases component can reveal useful predictive signals \cite{zemel13}.

Figure~\ref{fig:rq3}(d) presents a sensitivity analysis of $\alpha$.
Increasing $\alpha$ favors fidelity to the original predictions, thereby preserving AUC while limiting bias reduction. 
As expected, the trade-off is more pronounced at $k=2$ than at $k=3$, which is consistent with the previous discussion.

Finally, while structural biases were shown to be strongly correlated across hop orders under pre-processing interventions, our post-processing approach guarantees independence across predictive biases $NF^{(k)}$. By construction, node pairs at exact distance $k$ in the training graph do not contribute to $NF^{(k')}$ for $k' \neq k$. As a result, optimizing $NF^{(k)}$ cannot worsen predictive bias at other orders, enabling controlled and modular trade-offs across multiple hops. 

Additional results, including full numerical metrics with standard deviations, empirical comparisons with dyadic fairness metrics (DP and EO), results on random-walk–based approaches, and the complete sensitivity analysis of $\alpha$ across all datasets, are reported in Supplementary~E.3.

\section{Conclusions and Limitations}
This paper expands the conventional view of fairness in graph link prediction by accounting for structures underlying the prediction process, focusing on sensitive attribute diversity within nodes' $k$-hop neighborhoods.  
This direction was motivated by highlighting limitations of standard fairness and structural bias metrics, which we redefine within the $k$-hop fairness framework. Finally, we proposed mitigation strategies targeting these disparities, with experimental results demonstrating their effectiveness across several datasets and link prediction algorithms.
\paragraph{Limitations} 
Our framework relies on exact shortest-path distances, which limits scalability on large graphs despite the use of sparse representations. Approximate shortest-path estimators could mitigate this limitation.  
Moreover, our approach is intentionally purely structural, whereas real-world graphs often include node features, suggesting that feature-based fairness constitutes a complementary direction.

\printbibliography

@inproceedings{li2021dyadic,
  title={On dyadic fairness: Exploring and mitigating bias in graph connections},
  author={Li, Peizhao and Wang, Yifei and Zhao, Han and Hong, Pengyu and Liu, Hongfu},
  booktitle={International conference on learning representations},
  year={2021}
}

@article{rahman2019fairwalk,
  title={Fairwalk: Towards fair graph embedding},
  author={Rahman, Tahleen and Surma, Bartlomiej and Backes, Michael and Zhang, Yang},
  year={2019},
  journal={International Joint Conference on Artificial Intelligence (IJCAI)},
  publisher={CISPA}
}

@inproceedings{li2022fairlp,
  title={Fairlp: Towards fair link prediction on social network graphs},
  author={Li, Yanying and Wang, Xiuling and Ning, Yue and Wang, Hui},
  booktitle={Proceedings of the international AAAI conference on web and social media},
  volume={16},
  pages={628--639},
  year={2022}
}

@inproceedings{feldman2015certifying,
  title={Certifying and removing disparate impact},
  author={Feldman, Michael and Friedler, Sorelle A and Moeller, John and Scheidegger, Carlos and Venkatasubramanian, Suresh},
  booktitle={proceedings of the 21th ACM SIGKDD international conference on knowledge discovery and data mining},
  pages={259--268},
  year={2015}
}

@article{hardt2016equality,
  title={Equality of opportunity in supervised learning},
  author={Hardt, Moritz and Price, Eric and Srebro, Nati and al},
  journal={Advances in neural information processing systems},
  volume={29},
  year={2016}
}

@article{subramonian2023networked,
  title={Networked inequality: preferential attachment bias in graph neural network link prediction},
  author={Subramonian, Arjun and Sagun, Levent and Sun, Yizhou and al},
  journal={arXiv preprint arXiv:2309.17417},
  year={2023}
}

@article{castelnovo2022clarification,
  title={A clarification of the nuances in the fairness metrics landscape},
  author={Castelnovo, Alessandro and Crupi, Riccardo and Greco, Greta and Regoli, Daniele and Penco, Ilaria Giuseppina and Cosentini, Andrea Claudio},
  journal={Scientific Reports},
  volume={12},
  number={1},
  pages={4209},
  year={2022},
  publisher={Nature Publishing Group UK London}
}

@inproceedings{dwork2012fairness,
  title={Fairness through awareness},
  author={Dwork, Cynthia and Hardt, Moritz and Pitassi, Toniann and Reingold, Omer and Zemel, Richard},
  booktitle={Proceedings of the 3rd innovations in theoretical computer science conference},
  pages={214--226},
  year={2012}
}

@InProceedings{zemel13,
  title = 	 {Learning Fair Representations},
  author = 	 {Zemel, Rich and Wu, Yu and Swersky, Kevin and Pitassi, Toni and Dwork, Cynthia},
  booktitle = 	 {Proceedings of the 30th International Conference on Machine Learning},
  pages = 	 {325--333},
  year = 	 {2013},
  volume = 	 {28},
  number =       {3},
  series = 	 {Proceedings of Machine Learning Research},
  month = 	 Jun,
  publisher =    {PMLR}
}

@article{chen2022graph,
  title={Graph learning with localized neighborhood fairness},
  author={Chen, April and Rossi, Ryan and Lipka, Nedim and Hoffswell, Jane and Chan, Gromit and Guo, Shunan and Koh, Eunyee and Kim, Sungchul and Ahmed, Nesreen K},
  journal={arXiv preprint arXiv:2212.12040},
  year={2022}
}

@inproceedings{jalali2020information,
  title={On the information unfairness of social networks},
  author={Jalali, Zeinab S and Wang, Weixiang and Kim, Myunghwan and Raghavan, Hema and Soundarajan, Sucheta},
  booktitle={Proceedings of the 2020 SIAM International Conference on Data Mining},
  pages={613--521},
  year={2020},
  organization={SIAM}
}

@article{arnaiz2023structural,
  title={Structural Group Unfairness: Measurement and Mitigation by means of the Effective Resistance},
  author={Arnaiz, Rodriguez and Adrian and Curto, Georgina and Oliver, Nuria},
  journal={arXiv preprint arXiv:2305.03223},
  year={2023}
}

@inproceedings{laclau2021all,
  title={All of the fairness for edge prediction with optimal transport},
  author={Laclau, Charlotte and Redko, Ievgen and Choudhary, Manvi and Largeron, Christine},
  booktitle={International Conference on Artificial Intelligence and Statistics},
  pages={1774--1782},
  year={2021},
  organization={PMLR},
}

@inproceedings{grover2016node2vec,   
    title={node2vec: Scalable feature learning for networks},   author={Grover, Aditya and Leskovec, Jure},   
    booktitle={Proceedings of the 22nd ACM SIGKDD international conference on Knowledge discovery and data mining},   
    pages={855--864},   
    year={2016} 
    }

@book{pariser2011filter,
  title={The filter bubble: What the Internet is hiding from you},
  author={Pariser, Eli},
  year={2011},
  publisher={penguin UK}
}

@inproceedings{avin2015homophily,
  title={Homophily and the glass ceiling effect in social networks},
  author={Avin, Chen and Keller, Barbara and Lotker, Zvi and Mathieu, Claire and Peleg, David and Pignolet, Yvonne-Anne},
  booktitle={Proceedings of the 2015 conference on innovations in theoretical computer science},
  pages={41--50},
  year={2015}
}

@article{baumann2020modeling,
  title={Modeling echo chambers and polarization dynamics in social networks},
  author={Baumann, Fabian and Lorenz-Spreen, Philipp and Sokolov, Igor M and Starnini, Michele},
  journal={Physical review letters},
  volume={124},
  number={4},
  pages={048301},
  year={2020},
  publisher={APS}
}

@inproceedings{buyl2020debayes,
  title={Debayes: a bayesian method for debiasing network embeddings},
  author={Buyl, Maarten and De Bie, Tijl},
  booktitle={International Conference on Machine Learning},
  pages={1220--1229},
  year={2020},
  organization={PMLR}
}

@inproceedings{masrour2020bursting,
  title={Bursting the filter bubble: Fairness-aware network link prediction},
  author={Masrour, Farzan and Wilson, Tyler and Yan, Heng and Tan, Pang-Ning and Esfahanian, Abdol},
  booktitle={Proceedings of the AAAI conference on artificial intelligence},
  volume={34},
  pages={841--848},
  year={2020}
}

@inproceedings{khajehnejad2022crosswalk,
  title={Crosswalk: Fairness-enhanced node representation learning},
  author={Khajehnejad, Ahmad and Khajehnejad, Moein and Babaei, Mahmoudreza and Gummadi, Krishna P and Weller, Adrian and Mirzasoleiman, Baharan},
  booktitle={Proceedings of the AAAI Conference on Artificial Intelligence},
  volume={36},
  pages={11963--11970},
  year={2022}
}

@article{newman2002assortative,
  title={Assortative mixing in networks},
  author={Newman, Mark EJ},
  journal={Physical review letters},
  volume={89},
  number={20},
  pages={208701},
  year={2002},
  publisher={APS}
}

@book{biggs1993algebraic,
  title={Algebraic graph theory},
  author={Biggs, Norman},
  number={67},
  year={1993},
  publisher={Cambridge university press}
}

@inproceedings{adamic2005political,
  title={The political blogosphere and the 2004 US election: divided they blog},
  author={Adamic, Lada A and Glance, Natalie},
  booktitle={Proceedings of the 3rd international workshop on Link discovery},
  pages={36--43},
  year={2005}
}

@misc{snapnets,
                    author       = {Jure Leskovec and Andrej Krevl},
                    title        = {{SNAP Datasets}: {Stanford} Large Network Dataset Collection},
                    howpublished = {\url{http://snap.stanford.edu/data}},
                    month        = jun,
                    year         = 2014
                  }

@inproceedings{caragea2014citeseer,
  title={Citeseer x: A scholarly big dataset},
  author={Caragea, Cornelia and Wu, Jian and Ciobanu, Alina and Williams, Kyle and Fern{\'a}ndez-Ram{\'\i}rez, Juan and Chen, Hung-Hsuan and Wu, Zhaohui and Giles, Lee},
  booktitle={European Conference on Information Retrieval},
  pages={311--322},
  year={2014},
  organization={Springer}
}

@article{holland1983stochastic,
  title={Stochastic blockmodels: First steps},
  author={Holland, Paul W and Laskey, Kathryn Blackmond and Leinhardt, Samuel and al},
  journal={Social networks},
  volume={5},
  number={2},
  pages={109--137},
  year={1983},
  publisher={Elsevier}
}

@inproceedings{delarue2024link,
  title={Link Prediction Without Learning},
  author={Delarue, Simon and Bonald, Thomas and Viard, Tiphaine and al},
  booktitle={European Conference on Artificial Intelligence},
  year={2024}
}

@article{kipf2016semi,
  title={Semi-Supervised Classification with Graph Convolutional Networks},
  author={Kipf, TN},
  journal={arXiv preprint arXiv:1609.02907},
  year={2016}
}

@article{watts1998collective,
  title={Collective dynamics of ‘small-world’networks},
  author={Watts, Duncan J and Strogatz, Steven H},
  journal={nature},
  volume={393},
  number={6684},
  pages={440--442},
  year={1998},
  publisher={Nature Publishing Group}
}

@article{hamilton2017inductive,
  title={Inductive representation learning on large graphs},
  author={Hamilton, Will and Ying, Zhitao and Leskovec, Jure},
  journal={Advances in neural information processing systems},
  volume={30},
  year={2017}
}

@inproceedings{wang2022unbiased,
  title={Unbiased graph embedding with biased graph observations},
  author={Wang, Nan and Lin, Lu and Li, Jundong and Wang, Hongning},
  booktitle={Proceedings of the ACM web conference 2022},
  pages={1423--1433},
  year={2022}
}

@inproceedings{dong2022edits,
  title={Edits: Modeling and mitigating data bias for graph neural networks},
  author={Dong, Yushun and Liu, Ninghao and Jalaian, Brian and Li, Jundong},
  booktitle={Proceedings of the ACM web conference 2022},
  pages={1259--1269},
  year={2022}
}

@inproceedings{he2023fairmile,
  title={FairMILE: Towards an efficient framework for fair graph representation learning},
  author={He, Yuntian and Gurukar, Saket and Parthasarathy, Srinivasan},
  booktitle={Proceedings of the 3rd ACM Conference on Equity and Access in Algorithms, Mechanisms, and Optimization},
  pages={1--10},
  year={2023}
}

@inproceedings{ling2023learning,
  title={Learning fair graph representations via automated data augmentations},
  author={Ling, Hongyi and Jiang, Zhimeng and Luo, Youzhi and Ji, Shuiwang and Zou, Na},
  booktitle={International Conference on Learning Representations (ICLR)},
  year={2023}
}

@article{mattos2025breaking,
  title={Breaking the Dyadic Barrier: Rethinking Fairness in Link Prediction Beyond Demographic Parity},
  author={Mattos, Jo{\~a}o and Lina, Debolina Halder and Silva, Arlei},
  journal={arXiv preprint arXiv:2511.06568},
  year={2025}
}

@misc{laclau2024surveyfairnessmachinelearning,
      title={A Survey on Fairness for Machine Learning on Graphs}, 
      author={Charlotte Laclau and Christine Largeron and Manvi Choudhary},
      year={2024},
      eprint={2205.05396},
      archivePrefix={arXiv},
      primaryClass={cs.LG},
      url={https://arxiv.org/abs/2205.05396}, 
}

@article{chen-survey,
author = {Chen, April and Rossi, Ryan A. and Park, Namyong and Trivedi, Puja and Wang, Yu and Yu, Tong and Kim, Sungchul and Dernoncourt, Franck and Ahmed, Nesreen K.},
title = {Fairness-Aware Graph Neural Networks: A Survey},
year = {2024},
issue_date = {July 2024},
publisher = {Association for Computing Machinery},
address = {New York, NY, USA},
volume = {18},
number = {6},
issn = {1556-4681},
url = {https://doi.org/10.1145/3649142},
doi = {10.1145/3649142},
journal = {ACM Trans. Knowl. Discov. Data},
month = apr,
articleno = {138},
numpages = {23}
}

@article{marey2026topofair,
  title={TopoFair: Linking Topological Bias to Fairness in Link Prediction Benchmarks},
  author={Marey, Lilian and Perez, Mathilde and Viard, Tiphaine and Laclau, Charlotte},
  journal={arXiv preprint arXiv:2602.11802},
  year={2026}
}

\newpage
\appendix

\section{Case Study: Comparing k-hop Structural Bias with Assortativity}

To support the previously introduced definitions, we present a series of toy graphs in which we compute $NB^{(k)}$ for several values of $k$, and compare it to the canonical homophily metric, namely assortativity with respect to sensitive attribute~\cite{newman2002assortative}, denoted $a$, which is, as in dyadic fairness, computed based on edge types. 
In the following, node colors refer to sensitive attribute values. 

\subsection{Results}
\subsubsection{Sensitive Star Graph}
We begin with a simple example: consider the star graph composed of one blue central node and 
\( n \) peripheral nodes, among which a proportion \( p \) are blue and the remaining \( 1-p \) are red, 
as illustrated in Figure~\ref{fig:star_graph}.

\begin{figure}[htbp]
    \centering
    \includegraphics[width=.4\linewidth]{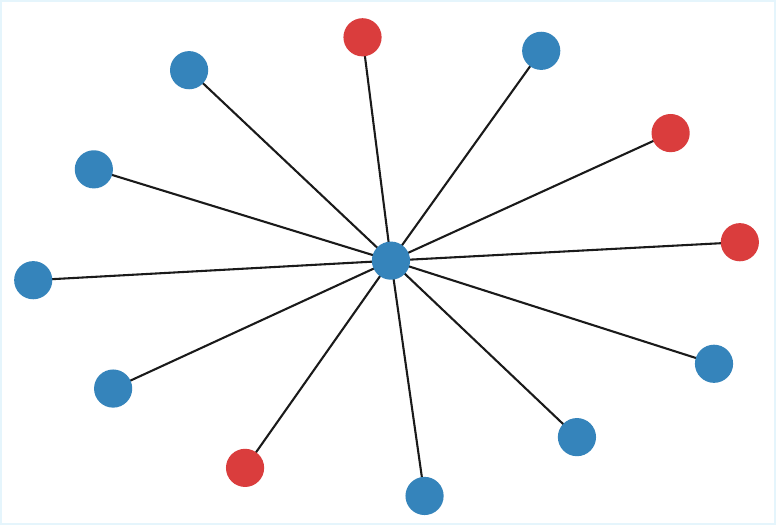}
    \caption{Sensitive star graph with $n=12$ and $p=\frac{2}{3}$}
    \label{fig:star_graph}
\end{figure}

Here, assortativity is $a = -\frac{1-p}{p+1}$ and \( NB^{(1)} = \frac{1-p}{np+1} \).
It is noteworthy that assortativity is independent of \( n \), whereas \( NB^{(1)} \) depends on \( n \) through its denominator, causing it to converge to zero as \( n \) grows. 
In fact, the central node occupies a favorable position because it can access all nodes with $1$ hop, while it is not the case for peripheral nodes. 
As \( n \) increases, the number of peripheral nodes grows, reducing the central node's share in the overall blue group. 
Asymptotically, both groups tend to occupy equivalent structural positions, making it coherent for the bias to vanish.
This result highlights a fundamental feature of our bias definition: it is node-centric rather than edge-based.

\begin{figure*}[t]
    \centering
    \includegraphics[width=.9\linewidth]{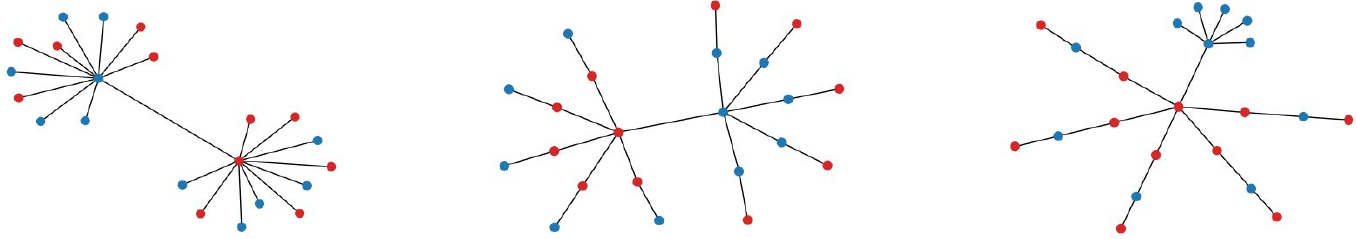}
    \caption{Toy graphs (a), (b), (c) (left to right) with constant assortativity. Node colors represent sensitive attributes and the dotted lines represent different edge prediction sets according to color.}
    \label{fig:toy_graphs}
\end{figure*}

\subsubsection{Toy graphs}

We now focus on the three toy graphs in Figure~\ref{fig:toy_graphs}. 
Each graph starts with two connected bridge nodes (one blue, one red). 
In \textbf{(a)}, \(n\) blue and \(n\) red nodes are attached to each bridge node, in \textbf{(b)}, the same nodes are hierarchically attached according to bridge-node color, and in \textbf{(c)}, \(n\) blue nodes attach to the blue bridge node, while the red bridge node receives \(n\) red, \(n\) blue, then \(n\) red nodes hierarchically. 
Graphs in Figure~\ref{fig:toy_graphs} are shown for \(n=5\).

In Toy Graph (a), all peripheral nodes experience identical information flow across successive orders. In Toy Graph (b), some peripheral nodes are relegated to the fringe, requiring an additional hop to access the full network information, while in Toy Graph (c), the two parts of the network are unbalanced, resulting in a different form of disparity in information access.
This observation is confirmed by Table~\ref{toygraph_bias}, where the three graph structures display significant variability: asymptotically, graph (a) is unbiased, graph (b) is biased at hops $1$ and $3$, while graph (c) is biased at all three hops. 

\begin{table}[ht]
\renewcommand{\arraystretch}{1.5}
\centering
\begin{tabular}{c|cccc}
Toy Graph & $a$ & $NB^{(1)}$   & $NB^{(2)}$        & $NB^{(3)}$        \\ \hline
\textbf{(a)}   & $\sim\frac{1}{n}$    & $\sim\frac{1}{n^2}$ & $\sim\frac{1}{n}$ & $0$               \\
\textbf{(b)}   & $\sim\frac{1}{n}$    & $\sim\frac{1}{2}$   & $\sim\frac{1}{n}$ & $\sim1$           \\
\textbf{(c)}   & $\sim\frac{1}{n}$   & $\sim\frac{1}{4}$   & $\sim\frac{1}{2}$ & $\sim\frac{1}{2}$
\end{tabular}
\caption{Asymptotic assortativity and $NB^{(k)}$ for $k \in \{ 1, 2, 3\}$ when $n \rightarrow \infty$ for the three presented toy graphs.}
\label{toygraph_bias}
\end{table}

However, we note that the three proposed networks are invariant in terms of number and type of edges, which results in a fixed assortativity value of \( a = \frac{1}{4n + 1} \TendsTo{n\rightarrow \infty} 0 \). 
This constancy appears counter-intuitive given the clearly different disparities described across the three cases. 
By contrast, $NB^{(1)}$ (its closest conceptual counterpart) varies across the three networks and remains asymptotically non-zero for toy graphs (b) and (c), which is more consistent with the topologies at hand. 
By taking a node-level view and accounting for structural bias at a fixed hop $k$, our approach captures disparities that remain hidden under the traditional homophily and dyadic fairness paradigm.

\subsection{Proofs}

In this section, we detail the calculation of $k$-hop structural bias $NB^{(k)}$ and assortativity $a$ for the presented toy graphs. 
Since these graphs involve only two sensitive attribute values, the general expressions of $NB^{(k)}$ and $a$ simplify in this setting. The next two sections focus on these simplifications.
We then provide explicit computations for the \textit{Sensitive Star Graph} and for the three illustrative Toy Graphs (a), (b), and (c).

In the following, as $|\mathbb{S}| = 2$, we assume for convenience $\mathbb{S}  = \{ 0, 1 \}$, and consider $S(v) = 0$ and $S(v) = 1$ for blue and red nodes, respectively.

\subsubsection{Structural k-hop bias with binary sensitive attribute}

Here, we show that the expression of \( NB^{(k)} \) simplifies in the case of binary sensitive attributes, allowing us to avoid the use of the \(\max\) operator.

In the following, for $s \in \mathbb{S}$ and $k \in \mathbb{N}^*$, we will denote $\mathcal{V}^{(k)}_s$ the set of nodes with sensitive attribute $s$ that have at least one $k$-hop neighbor. Formally,  $\mathcal{V}^{(k)}_s = \{ v \in \mathcal{V}_s, N^{(k)}(v) \neq \varnothing \} $.

\begin{property}
\label{propNB}
    If sensitive attributes are binary, then, for $k \in \mathbb{N}^*, NB^{(k)} = |\phi^{(k)}_{0 \rightarrow 0}(r^{(k)}) + \phi^{(k)}_{1 \rightarrow 1}(r^{(k)}) -1|$
\end{property}

\begin{proof}
By definition,
\begin{equation*}
    \begin{aligned}
        NB^{(k)} = \max (&|\phi^{(k)}_{0 \rightarrow 0}(r^{(k)}) - \phi^{(k)}_{1 \rightarrow 0}(r^{(k)})|, \\ 
        &|\phi^{(k)}_{0 \rightarrow 1}(r^{(k)}) - \phi^{(k)}_{1 \rightarrow 1}(r^{(k)})|)
    \end{aligned}
\end{equation*}
Yet, 
\begin{equation*}
    \begin{aligned}
        \phi^{(k)}_{0 \rightarrow 1}(r^{(k)}) 
        &= \frac{1}{|\mathcal{V}^{(k)}_0|} \sum\limits_{v \in \mathcal{V}^{(k)}_0} \frac{|N^{(k)}_1(v)|}{|N^{(k)}(v)|} \\
        &= \frac{1}{|\mathcal{V}^{(k)}_0|} \sum\limits_{v \in \mathcal{V}^{(k)}_0} (1 - \frac{|N^{(k)}_0(v)|}{|N^{(k)}(v)|}) \\
        &= 1 - \phi^{(k)}_{0 \rightarrow 0}(r^{(k)})
    \end{aligned}
\end{equation*}

Similarly, $\phi^{(k)}_{1 \rightarrow 0}(r^{(k)}) = 1 - \phi^{(k)}_{1 \rightarrow 1}(r^{(k)})$.

Thus, 
\[
\begin{aligned}
NB^{(k)} = \max ( 
    &|\phi^{(k)}_{0 \rightarrow 0}(r^{(k)}) + \phi^{(k)}_{1 \rightarrow 1}(r^{(k)}) - 1|, \\
    &|1 - (\phi^{(k)}_{0 \rightarrow 0}(r^{(k)}) + \phi^{(k)}_{1 \rightarrow 1}(r^{(k)})|)
\end{aligned}
\]
\[
NB^{(k)} = |\phi^{(k)}_{0 \rightarrow 0}(r^{(k)}) + \phi^{(k)}_{1 \rightarrow 1}(r^{(k)}) - 1|
\]
\end{proof}

\subsubsection{Assortativity with respect to binary sensitive attribute}
Here, we recall the definition of assortativity with respect to a sensitive attribute~\cite{newman2002assortative}, then simplify the expression for the binary attribute case.

\begin{definition}
\label{def:assortativity}
    Assortativity with respect to sensitive attribute is defined as:
\begin{equation*}
    a = \frac{Tr(e) - \sum e^2}{1 - \sum e^2}
\end{equation*}

with $e$ being the mixing matrix:
\[
e = 
\frac{1}{|\mathcal{E}|}
\begin{bmatrix} 
|\mathcal{E}_{ 0, 0}| & 
\frac{1}{2}|\mathcal{E}_{ 0, 1}| \\ 
\frac{1}{2}|\mathcal{E}_{ 0, 1}| &
|\mathcal{E}_{ 1, 1}| \end{bmatrix}
\]
with $\mathcal{E}_{s, s'} = \{ \{v, v'\} \in \mathcal{E} | S(v) = s, S(v') = s' \}$ and $\sum e^2$ being the sum of all the coefficient of $e^2$ matrix.
\end{definition}

We now expand the expression of \(\sum e^2\).

\begin{property}
\label{prop:mixingmatrix}
     $\sum e^2 = (e_{0, 0} + e_{0, 1})^2 + (e_{1, 1} + e_{0, 1})^2$.
\end{property}

\begin{proof}
\[
e^2 = \begin{bmatrix} 
e_{0, 0}^2 + e_{0, 1}^2 & 
e_{0, 0} e_{0, 1} + e_{0, 1} e_{1, 1} \\ 
e_{0, 0} e_{0, 1} + e_{0, 1} e_{1, 1} &
e_{1, 1}^2 + e_{0, 1}^2 \end{bmatrix}
\]
Hence, 
\begin{equation*}
    \begin{aligned}
\sum e^2 &= 
e_{0, 0}^2 + e_{0, 1}^2 + 
2 e_{0, 0} e_{0, 1} + 2 e_{0, 1} e_{1, 1} +
e_{1, 1}^2 + e_{0, 1}^2 \\
\sum e^2 &= (e_{0, 0}^2 + 2 e_{0, 0} e_{0, 1} + e_{0, 1}^2)
+ (e_{1, 1}^2 + 2 e_{1, 1} e_{0, 1} + e_{0, 1}^2) \\
\sum e^2 &= (e_{0, 0} + e_{0, 1})^2 + (e_{1, 1} + e_{0, 1})^2
    \end{aligned}
\end{equation*}
    
\end{proof}

This simplification leads us to a more direct closed-form expression for calculating assortativity.

\begin{corollary}
\label{cor:assortativity}
   \begin{equation*}
    a = \frac{(e_{0, 0} - (e_{0, 0} + e_{0, 1})^2) + (e_{1, 1} - (e_{1, 1} + e_{0, 1})^2)}{(\frac{1}{2} - (e_{0, 0} + e_{0, 1})^2) + (\frac{1}{2} - (e_{1, 1} + e_{0, 1})^2)}
\end{equation*} 
\end{corollary}

\begin{proof}
Injecting the result of Property~\ref{prop:mixingmatrix} in Definition~\ref{def:assortativity} gives the result.
    
\end{proof}

\subsubsection{Sensitive Star Graph}

\paragraph{Assortativity}
The edges proportions lead to the following values in mixing matrix $e$:
\[
e_{0, 0} = p, \quad e_{1, 1} = 0, \quad e_{0, 1} = \frac{1}{2}(1-p)
\]

Hence, using Corollary~\ref{cor:assortativity},
\begin{equation*}
\begin{aligned}
    a &=  \frac{(e_{0, 0} - (e_{0, 0} + e_{0, 1})^2) + (e_{1, 1} - (e_{1, 1} + e_{0, 1})^2)}{(\frac{1}{2} - (e_{0, 0} + e_{0, 1})^2) + (\frac{1}{2} - (e_{1, 1} + e_{0, 1})^2)} \\
    a &= \frac{(p - (p + \frac{1}{2}(1-p))^2) + (0 - (0 + \frac{1}{2}(1-p))^2)}{(\frac{1}{2} - (p + \frac{1}{2}(1-p))^2) + (\frac{1}{2} - (0 + \frac{1}{2}(1-p))^2)} \\
    a & = \frac{(\frac{-1}{4} + \frac{p}{2} - \frac{p^2}{4}) + (\frac{-1}{4} + \frac{p}{2} - \frac{p^2}{4})}{(\frac{1}{4} - \frac{p}{2} - \frac{p^2}{4}) + (\frac{1}{4} + \frac{p}{2} - \frac{p^2}{4})} \\
    a & = \frac{\frac{-1}{2} + p - \frac{p^2}{2}}{\frac{1}{2} - \frac{p^2}{2}} \\
    a & = \frac{(1-p)^2}{1-p^2} \\
    & \boxed{a = -\frac{1-p}{p+1}}
\end{aligned}
\end{equation*}

\paragraph{$1$-hop Structural Bias}
Let us compute $\phi^{(1)}_{0 \rightarrow 0}$.
\begin{itemize}
    \item The central blue node is surrounded by $np$ blue nodes and $n(1-p)$ red nodes $\rightarrow p$
    \item The $np$ peripheral blue nodes only access the blue central node $\rightarrow 1$
\end{itemize}
Thus, $\phi^{(1)}_{0 \rightarrow 0} = \frac{1}{np+1}(p+np) = p\frac{n+1}{np+1}$

Let us compute $\phi^{(1)}_{1 \rightarrow 1}$.
\begin{itemize}
    \item The n peripheral red nodes access 1 blue node $\rightarrow 0$
\end{itemize}
Thus, $\phi^{(1)}_{1 \rightarrow 1} = 0$

Hence, using Property~\ref{propNB}, $NB^{(1)} = \Big|p\frac{n+1}{np+1} + 0 - 1   \Big|$
$$\boxed{NB^{(1)} = \frac{1-p}{np+1}}$$

\subsubsection{Toy Graph (a)}

\paragraph{Assortativity}

The edges proportions lead to the following values in mixing matrix $e$:
\[
e_{0, 0} = \frac{n}{4n+1} \quad e_{1, 1} = \frac{n}{4n+1} \quad e_{0, 1} = \frac{1}{2}\frac{2n+1}{4n+1}
\]

Hence,
\begin{equation*}
\begin{aligned}
    &(e_{0, 0} + e_{0, 1}) = (e_{1, 1} + e_{0, 1}) = \frac{n}{4n+1} + \frac{1}{2}\frac{2n+1}{4n+1} = \frac{1}{2}\\
    &(e_{0, 0} + e_{0, 1})^2 = (e_{1, 1} + e_{0, 1})^2 = \frac{1}{4}\\
    &e_{0, 0} - (e_{0, 0} + e_{0, 1})^2 = e_{1, 1} - (e_{1, 1} + e_{0, 1})^2 \\
    &= \frac{n}{4n+1} - \frac{1}{4} = \frac{-1}{4(4n+1)}\\
    &\frac{1}{2} - (e_{0, 0} + e_{0, 1})^2 = \frac{1}{2} - (e_{1, 1} + e_{0, 1})^2 = \frac{1}{2} - \frac{1}{4} = \frac{1}{4}\\
\end{aligned}
\end{equation*}

Thus,  using Corollary~\ref{cor:assortativity},
\begin{equation*}
\begin{aligned}
    a &=  \frac{(e_{0, 0} - (e_{0, 0} + e_{0, 1})^2) + (e_{1, 1} - (e_{1, 1} + e_{0, 1})^2)}{(\frac{1}{2} - (e_{0, 0} + e_{0, 1})^2) + (\frac{1}{2} - (e_{1, 1} + e_{0, 1})^2)} \\
    a &=  \frac{\frac{-1}{4(4n+1)} + \frac{-1}{4(4n+1)}}{\frac{1}{4} + \frac{1}{4}}  \\
    & \boxed{a =  \frac{-1}{4n+1}}
\end{aligned}
\end{equation*}

\paragraph{$1$-hop Structural Bias}
Let us compute $\phi^{(1)}_{0 \rightarrow 0}$.
\begin{itemize}
    \item The blue bridge node accesses $n$ blue nodes and $n+1$ red nodes $\rightarrow \frac{n}{2n+1}$
    \item The $n$ peripheral blue nodes situated around the blue bridge node only access the blue bridge node $\rightarrow 1$
    \item The $n$ peripheral blue nodes situated around the red bridge node only access the red bridge node $\rightarrow 0$
\end{itemize}
Thus, $\phi^{(1)}_{0 \rightarrow 0} = \frac{1}{2n+1}\Big( \frac{n}{2n+1} + n \Big) = \frac{2n(n+1)}{(2n+1)^2}$.

Symmetrically, $\phi^{(1)}_{1 \rightarrow 1} = \frac{2n(n+1)}{(2n+1)^2}$.

Hence, using Property~\ref{propNB}, $NB^{(1)} = \Big|2\frac{2n(n+1)}{(2n+1)^2} - 1   \Big|$
$$\boxed{NB^{(1)} = \frac{1}{(2n+1)^2}}$$

\paragraph{$2$-hop Structural Bias}
Let us compute $\phi^{(2)}_{0 \rightarrow 0}$.

At $2$-hop:
\begin{itemize}
    \item The blue bridge node accesses $n$ blue nodes and $n$ red nodes $\rightarrow \frac{1}{2}$
    \item The $n$ peripheral blue nodes situated around the blue bridge node access $n-1$ blue nodes and $n+1$ red nodes $\rightarrow \frac{n-1}{2n}$
    \item The $n$ peripheral blue nodes situated around the red bridge node access $n$ blue nodes and $n$ red nodes $\rightarrow \frac{1}{2}$
\end{itemize}
Thus, $\phi^{(2)}_{0 \rightarrow 0} = \frac{1}{2n+1}\Big( \frac{1}{2} + \frac{n(n-1)}{2n} + \frac{n}{2}\Big) = \frac{n}{2n+1}$.

Symmetrically, $\phi^{(2)}_{1 \rightarrow 1} = \frac{n}{2n+1}$.

Hence, using Property~\ref{propNB}, $NB^{(2)} = \Big|2\frac{n}{2n+1} - 1   \Big|$
$$\boxed{NB^{(2)} = \frac{1}{2n+1}}$$

\paragraph{$3$-hop Structural Bias}
Let us compute $\phi^{(3)}_{0 \rightarrow 0}$.

At $3$-hop:
\begin{itemize}
    \item The blue bridge node does not access any nodes.
    \item The $n$ peripheral blue nodes situated around the blue bridge node access $n$ blue nodes and $n$ red nodes $\rightarrow \frac{1}{2}$
    \item The $n$ peripheral blue nodes situated around the red bridge node access $n$ blue nodes and $n$ red nodes $\rightarrow \frac{1}{2}$
\end{itemize}
Thus, $\phi^{(3)}_{0 \rightarrow 0} = \frac{1}{2n}\Big( \frac{n}{2} + \frac{n}{2}\Big) = \frac{1}{2}$.

Symmetrically, $\phi^{(2)}_{1 \rightarrow 1} = \frac{1}{2}$.

Hence, using Property~\ref{propNB}, $NB^{(3)} = \Big|2\frac{1}{2} - 1   \Big|$
$$\boxed{NB^{(3)} = 0}$$

\subsubsection{Toy Graph (b)}

\paragraph{Assortativity}

The edges proportions are the same as in Toy Graph (a), thus
$\boxed{a = \frac{-1}{4n+1}}$.

\paragraph{$1$-hop Structural Bias}
Let us compute $\phi^{(1)}_{0 \rightarrow 0}$.
\begin{itemize}
    \item The blue bridge node accesses $n$ blue nodes and $1$ red nodes $\rightarrow \frac{n}{n+1}$
    \item The $n$ peripheral blue nodes situated around the blue bridge node access $1$ blue node and $1$ red node $\rightarrow \frac{1}{2}$
    \item The $n$ peripheral blue nodes situated around the red bridge node access $1$ blue node and $1$ red node $\rightarrow \frac{1}{2}$
\end{itemize}
Thus, $\phi^{(1)}_{0 \rightarrow 0} = \frac{1}{2n+1}\Big( \frac{n}{n+1} + \frac{1}{2} \Big) = \frac{n^2 + 3n}{2(n+1)(2n+1)}$.

Symmetrically, $\phi^{(1)}_{1 \rightarrow 1} = \frac{n^2 + 3n}{2(n+1)(2n+1)}$.

Hence, using Property~\ref{propNB}, $NB^{(1)} = \Big|2\frac{n^2 + 3n}{2(n+1)(2n+1)} - 1 \Big|$
$$\boxed{NB^{(1)} = \frac{n^2 + 1}{(n+1)(2n+1)}}$$

\paragraph{$2$-hop Structural Bias}
Let us compute $\phi^{(2)}_{0 \rightarrow 0}$.

At $2$-hop:
\begin{itemize}
    \item The blue bridge node only accesses red nodes $\rightarrow 0$
    \item The $n$ peripheral blue nodes situated around the blue bridge node access $n-1$ blue node and $1$ red node $\rightarrow \frac{n-1}{n}$
    \item The $n$ peripheral blue nodes situated around the red bridge node only access $1$ red node $\rightarrow 0$
\end{itemize}
Thus, $\phi^{(2)}_{0 \rightarrow 0} = \frac{1}{2n+1}\Big( n-1 \Big) = \frac{n-1}{2n+1}$.

Symmetrically, $\phi^{(2)}_{1 \rightarrow 1} = \frac{n-1}{2n+1}$.

Hence, using Property~\ref{propNB}, $NB^{(2)} = \Big|2\frac{n-1}{2n+1} - 1 \Big|$
$$\boxed{NB^{(2)} = \frac{3}{2n+1}}$$

\paragraph{$3$-hop Structural Bias}
Let us compute $\phi^{(3)}_{0 \rightarrow 0}$.

At $3$-hop:
\begin{itemize}
    \item The blue bridge node accesses $n$ blue nodes $\rightarrow 1$
    \item The $n$ peripheral blue nodes situated around the blue bridge node only access red nodes $\rightarrow 0$
    \item The $n$ peripheral blue nodes situated around the red bridge node access $1$ blue node and $n-1$ red nodes $\rightarrow \frac{1}{n}$
\end{itemize}
Thus, $\phi^{(3)}_{0 \rightarrow 0} = \frac{1}{2n+1}\Big( 1 + 1 \Big) = \frac{2}{2n+1}$.

Symmetrically, $\phi^{(3)}_{1 \rightarrow 1} = \frac{2}{2n+1}$.

Hence, using Property~\ref{propNB}, $NB^{(3)} = \Big|2\frac{2}{2n+1} - 1 \Big|$
$$\boxed{NB^{(3)} = \frac{2n-3}{2n+1}}$$

\subsubsection{Toy Graph (c)}

\paragraph{Assortativity}

The edges proportions are the same as in Toy Graph (a), thus
$\boxed{a = \frac{-1}{4n+1}}$.

\paragraph{$1$-hop Structural Bias}
Let us compute $\phi^{(1)}_{0 \rightarrow 0}$.
\begin{itemize}
    \item The blue bridge node accesses $n$ blue nodes and $1$ red node $\rightarrow \frac{n}{n+1}$
    \item The $n$ peripheral blue nodes situated around the blue bridge node only access $1$ blue node $\rightarrow 1$
    \item The $n$ peripheral blue nodes situated around the red bridge node only access red nodes $\rightarrow 0$
\end{itemize}
Thus, $\phi^{(1)}_{0 \rightarrow 0} = \frac{1}{2n+1}\Big( \frac{n}{n+1} + n \Big) = \frac{n^2 + 2n}{(n+1)(2n+1)}$.

Let us compute $\phi^{(1)}_{1 \rightarrow 1}$.
\begin{itemize}
    \item The red bridge node accesses $n$ red nodes and $1$ blue node $\rightarrow \frac{n}{n+1}$
    \item The $n$ peripheral red nodes situated close to the red bridge node access $1$ red node and $1$ blue node $\rightarrow \frac{1}{2}$
    \item The $n$ peripheral red nodes situated far from the red bridge node only access blue nodes $\rightarrow 0$\end{itemize}
Thus, $\phi^{(1)}_{1 \rightarrow 1} = \frac{1}{2n+1}\Big( \frac{n}{n+1} + \frac{n}{2} \Big) = \frac{n^2 + 3n}{2(n+1)(2n+1)}$.

Hence, using Property~\ref{propNB}, $$NB^{(1)} = \Big|\frac{n^2 + 2n}{(n+1)(2n+1)} + \frac{n^2 + 3n}{2(n+1)(2n+1)} - 1 \Big|$$
$$\boxed{NB^{(1)} = \frac{n^2 - n + 2}{2(n+1)(2n+1)}}$$

\paragraph{$2$-hop Structural Bias}
Let us compute $\phi^{(2)}_{0 \rightarrow 0}$.

At $2$-hop:
\begin{itemize}
    \item The blue bridge node only accesses red nodes $\rightarrow 0$
    \item The $n$ peripheral blue nodes situated around the blue bridge node access $n-1$ blue nodes and $1$ red node $\rightarrow \frac{n-1}{n}$
    \item The $n$ peripheral blue nodes situated around the red bridge node only access red nodes $\rightarrow 0$
\end{itemize}
Thus, $\phi^{(2)}_{0 \rightarrow 0} = \frac{1}{2n+1}\Big( n-1 \Big) = \frac{n-1}{2n+1}$.

Let us compute $\phi^{(2)}_{1 \rightarrow 1}$.

At $2$-hop:
\begin{itemize}
    \item The red bridge node only accesses blue nodes $\rightarrow 0$
    \item The $n$ peripheral red nodes situated close to the red bridge node access $n$ red node and $1$ blue node $\rightarrow \frac{n}{n+1}$
    \item The $n$ peripheral red nodes situated far from the red bridge node only access red nodes $\rightarrow 1$\end{itemize}
Thus, $\phi^{(2)}_{1 \rightarrow 1} = \frac{1}{2n+1}\Big( \frac{n^2}{n+1} + n \Big) = \frac{2n^2 + n}{(n+1)(2n+1)}$.

Hence, using Property~\ref{propNB}, $$NB^{(2)} = \Big|\frac{n-1}{2n+1} + \frac{2n^2 + n}{(n+1)(2n+1)} - 1 \Big|$$
$$\boxed{NB^{(2)} = \frac{n^2 - 2n - 2}{(n+1)(2n+1)}}$$

\paragraph{$3$-hop Structural Bias}
Let us compute $\phi^{(3)}_{0 \rightarrow 0}$.

At $3$-hop:
\begin{itemize}
    \item The blue bridge node only accesses blue nodes $\rightarrow 1$
    \item The $n$ peripheral blue nodes situated around the blue bridge node only access red nodes $\rightarrow 0$
    \item The $n$ peripheral blue nodes situated around the red bridge node access $1$ blue node and $n-1$ red nodes $\rightarrow \frac{1}{n}$
\end{itemize}
Thus, $\phi^{(3)}_{0 \rightarrow 0} = \frac{1}{2n+1}\Big( 1 + 1 \Big) = \frac{2}{2n+1}$.

Let us compute $\phi^{(3)}_{1 \rightarrow 1}$.

At $3$-hop:
\begin{itemize}
    \item The red bridge node only accesses red nodes $\rightarrow 1$
    \item The $n$ peripheral red nodes situated close to the red bridge node only access blue nodes $\rightarrow 0$
    \item The $n$ peripheral red nodes situated far from the red bridge node only access red nodes $\rightarrow 1$\end{itemize}
Thus, $\phi^{(3)}_{1 \rightarrow 1} = \frac{1}{2n+1}\Big( 1 + n \Big) = \frac{n+1}{2n+1}$.

Hence, using Property~\ref{propNB}, $$NB^{(3)} = \Big|\frac{2}{2n+1} + \frac{n+1}{2n+1} - 1 \Big|$$
$$\boxed{NB^{(3)} = \frac{n-2}{2n+1}}$$

\section{Decomposition of DP in terms of node k-hop exposure}

In the paper we state the following property:
\begin{property}
\begin{equation*}
    \Delta DP  = \Big| \sum\limits_{k \in \mathcal{K}} \sum\limits_{v \in \mathcal{V}^{(k)}} \omega^{(k)}(v) \Big( \frac{f^{(k)}_{\mathrm{same}}(v)}{\mathbb{P}(S = S')}  - \frac{f^{(k)}_{\mathrm{diff}}(v)}{\mathbb{P}(S \neq S')} \Big) \Big|
\end{equation*}
where $f^{(k)}_{\mathrm{same}}(v)\triangleq f^{(k)}_{\mathcal{S}(v)}(v)$ and $f^{(k)}_{\mathrm{diff}}(v)\triangleq f^{(k)}_{1 - \mathcal{S}(v)}(v)$ are the within and cross group exposure, respectively; and $\omega^{(k)}(v) = \mathbb{P}(V = v,  \sigma(V, V') = k)$.
\end{property}

\begin{proof}
For clarity, we introduce the following notation.
\begin{itemize}
    \item $B = \{ S(V) = S(V') \}$ denotes the event that the two sampled nodes share the same sensitive attribute.
    \item $A_s = \{ S(V) = s \}$ and $A'_{s'} = \{ S(V') = s' \}$ denote the events that the first and second sampled nodes, respectively, have sensitive attributes $s$ and $s'$.
    \item $D_k = \{ \sigma(V, V') = k \}$ denotes the event that the shortest-path distance between the two sampled nodes is exactly $k$.
\end{itemize}

Let us start from $\Delta DP$ definition.
\begin{equation*}
    \Delta DP = |\mathbb{E}[\hat{Y} | B] - \mathbb{E}[\hat{Y} | \overline{B}]|
\end{equation*}

\paragraph{First term}
First,
\begin{equation*}
    \mathbb{E}[\hat{Y} | B] = \sum\limits_{k \in \mathcal{K}}\mathbb{P}(D_k|B)\mathbb{E}[\hat{Y} | B, D_k]
\end{equation*}
More,
\begin{equation*}
    \mathbb{E}[\hat{Y} | B, D_k] = \sum\limits_{v \in \mathcal{V}^{(k)}}\mathbb{P}(V =v|B, D_k)\mathbb{E}[\hat{Y} | B, D_k, V =v]
\end{equation*}
More, for $v \in \mathcal{V}^{(k)}$,
\begin{equation*}
    \begin{aligned}
    \mathbb{E}[\hat{Y} &| B, D_k, V =v] = \mathbb{E}[\hat{Y} | D_k, V =v, A'_{S(v)}]\\ 
     &= \frac{1}{\mathbb{P}(A'_{S(v)} | D_k, V = v)} \mathbb{E}[\hat{Y} \bbone_{A'_{S(v)}} | D_k, V =v]\\ 
     &= \frac{1}{\mathbb{P}(A'_{S(v)} | D_k, V = v)} f^{(k)}_{S(v)}(v)
    \end{aligned}
\end{equation*}

Thus, 
\begin{equation*}
    \mathbb{E}[\hat{Y} | B] = \sum\limits_{k \in \mathcal{K}}\sum\limits_{v \in \mathcal{V}^{(k)}} (\omega^{(k)}_{\text{intra}}(v)) \times f^{(k)}_{S(v)}(v)
\end{equation*}
with 
\begin{equation*}
    \begin{aligned}
        \omega^{(k)}_{\text{intra}}(v) &= \frac{\mathbb{P}(D_k|B)\mathbb{P}(V =v|B, D_k)}{\mathbb{P}(A'_{S(v)} | D_k, V = v)} \\
        \omega^{(k)}_{\text{intra}}(v) &= \mathbb{P}(D_k|B)\frac{\mathbb{P}(V =v \cap B| D_k)}{\mathbb{P}(B|D_k)}\frac{\mathbb{P}(V = v | D_k)}{\mathbb{P}(A'_{S(v)} \cap V = v | D_k)} \\
        \omega^{(k)}_{\text{intra}}(v) &= \frac{\mathbb{P}(D_k)}{\mathbb{P}(B)}\mathbb{P}(V = v | D_k) \\
        \omega^{(k)}_{\text{intra}}(v) &= \frac{\mathbb{P}(V = v \cap D_k)}{\mathbb{P}(B)} \\
    \end{aligned}
\end{equation*}

\paragraph{Second term}
First,
\begin{equation*}
    \mathbb{E}[\hat{Y} | \overline{B}] = \sum\limits_{k \in \mathcal{K}}\mathbb{P}(D_k|\overline{B})\mathbb{E}[\hat{Y} | \overline{B}, D_k]
\end{equation*}
More,
\begin{equation*}
    \mathbb{E}[\hat{Y} | \overline{B}, D_k] = \sum\limits_{v \in \mathcal{V}^{(k)}}\mathbb{P}(V = v|\overline{B}, D_k)\mathbb{E}[\hat{Y} | \overline{B}, D_k, V =v]
\end{equation*}
More, for $v \in \mathcal{V}^{(k)}$,
\begin{equation*}
    \begin{aligned}
    \mathbb{E}[\hat{Y} &| \overline{B}, D_k, V =v] = \mathbb{E}[\hat{Y} | D_k, V =v, A'_{1 - S(v)}]\\ 
     &= \frac{1}{\mathbb{P}(A'_{1 - S(v)} | D_k, V = v)} \mathbb{E}[\hat{Y} \bbone_{A'_{1 - S(v)}} | D_k, V =v]\\ 
     &= \frac{1}{\mathbb{P}(A'_{1 - S(v)} | D_k, V = v)} f^{(k)}_{1 - S(v)}(v)
    \end{aligned}
\end{equation*}

Thus, 
\begin{equation*}
    \mathbb{E}[\hat{Y} | \Bar{B}] = \sum\limits_{k \in \mathcal{K}}\sum\limits_{v \in \mathcal{V}^{(k)}} (\omega^{(k)}_{\text{inter}}(v)) \times f^{(k)}_{1 - S(v)}(v)
\end{equation*}
following the same steps as for first term, we obtain
\begin{equation*}
    \begin{aligned}
        \omega^{(k)}_{\text{inter}}(v) &= \frac{\mathbb{P}(D_k|\Bar{B})\mathbb{P}(V =v|\Bar{B}, D_k)}{\mathbb{P}(A'_{1 - S(v)} | D_k, V = v)} \\
        \omega^{(k)}_{\text{inter}}(v) &= \frac{\mathbb{P}(V = v \cap D_k)}{\mathbb{P}(\Bar{B})} \\
    \end{aligned}
\end{equation*}

\paragraph{Conclusion}
Finally,
\begin{equation*}
    \Delta DP  = \Big|\sum\limits_{v \in \mathcal{V}^{(k)}}\sum\limits_{k \in \mathcal{K}} \mathbb{P}(V = v \cap D_k)\big( \frac{f^{(k)}_{S(v)}(v)}{\mathbb{P}(B)} - \frac{f^{(k)}_{1 - S(v)}(v)}{1 - \mathbb{P}(B)} \big) \Big|
\end{equation*}

\end{proof}

\section{Methodology: Computing and Controlling Bias and Fairness}

We introduce computational methodologies to efficiently measure and control structural bias $NB^{(k)}$ and predictive fairness $NF^{(k)}$.

In the following, we consider a given graph $G$, and denote by $n$ its number of nodes and by $A \in \{0,1\}^{n \times n}$ its adjacency matrix. 
Also, for a given matrix $M, \mathop{\chi} (M) = (\bbone\{ M_{i,j} > 0 \})_{i, j}$, $\textbf{1}_{p,q}$ is the all-ones matrix of size $(p,q)$, $I_p$ is the identity matrix of size $p$, and $\odot$ denotes the element-wise (Hadamard) product between matrices.

\subsection{Computing Structural Bias and Predictor Fairness}

We first present the steps for computing $k$-hop structural bias $NB^{(k)}$.

\paragraph{Computing $k$-hop matrix}
To begin, we build the $k$-hop matrix $\Tilde{A}^{(k)} \in \{0,1\}^{n \times n}$ defined as $\Tilde{A}^{(k)}_{i, j} = \bbone_{\sigma(v_i, v_j) = k}$, which indicates whether two nodes are $k$ hops apart.
For this, we use the fact that for $k' \in \mathbb{N}^*, A^{k'}_{i, j}$ is the number of walks of length $k'$ linking nodes $v_i$ and $v_j$~\cite{biggs1993algebraic}.
Hence, computing successive powers of adjacency matrix $A$ and removing previously reached nodes allows us to compute $\Tilde{A}^{(k)}$:
\[
\Tilde{A}^{(k)} = \mathop{\chi}\Big(\mathop{\chi}(A^k) - \mathop{\chi}\big(\sum\limits_{i=1}^{k-1}A^i\big)\Big ) \odot \Big(\textbf{1}_{n, n} - I_n\Big).
\]
Here, the first term of the product retrieves the connections occurring exactly at hop $k$ and the second part acts as a mask for removing self-loops.

\paragraph{Computing $\phi_{s \rightarrow s'}^{(k)}$}

For computing sensitive attribute proportions within $k$-hop neighborhoods, we begin by defining sensitive attribute indicator matrices.
\[
\forall s \in \mathbb{S}, M_s = (\bbone\{\mathcal{V}_s[i] = v_j\})_{1 \leq i \leq |\mathcal{V}_s|, 1 \leq j \leq n} \in \{0,1\}^{|\mathcal{V}_s| \times n}.
\]
This allows us to filter $k$-hop neighbors with a specific sensitive attribute.
This way, the vector containing the number of $k$-hop neighbors with sensitive attribute $s'$ for nodes in $\mathcal{V}_{s}$ can be computed as:
\[
\psi_{s \rightarrow s'}^{(k)}(\Tilde{A}^{(k)}) = (M_{s}\Tilde{A}^{(k)}M_{s'}^\top) \times \textbf{1}_{|\mathcal{V}_{s'}|, 1} \in \mathbb{R}^{|\mathcal{V}_s|}.
\]

To assess neighborhood proportions, we compute the normalization vector:
\[
\lambda_{s}(\Tilde{A}^{(k)}) = 
(\mathop{\chi}(\psi^{(k)}_{s \rightarrow s'}(\Tilde{A}^{(k)})))_{s' \in \mathbb{S}}
\times \textbf{1}_{|\mathbb{S}|, 1} \in \mathbb{R}^{|\mathcal{V}_s|}.
\]

Finally, $\phi$ values can be computed:
\[
\phi_{s \rightarrow s'}^{(k)}(\Tilde{A}^{(k)}) = avg_{\lambda_{s} \neq 0}(\psi_{s \rightarrow s'}^{(k)}(\Tilde{A}^{(k)}) / \lambda_{s}(\Tilde{A}^{(k)})) \in \mathbb{R},
\]
where $ / $ is the element wise division for vectors and $avg_{\lambda_{s} \neq 0}$ is the average of vector elements, ignoring undefined values.

Hence, following the described steps, 
\[
NB^{(k)}(A) = \underset{s_1, s_2, s \in \mathbb{S}}{max} |\phi_{s_1 \rightarrow s}^{(k)}(\Tilde{A}^{(k)}) - \phi_{s_2 \rightarrow s}^{(k)}(\Tilde{A}^{(k)})|.
\]

\noindent\textbf{Remark:} To ensure the subdifferentiability of the framework, we recommend replacing the function $\mathop{\chi}$ with a sigmoid and the $\max$ operation with $LogSumExp$, which preserves subgradients and allows the use of the Adam optimizer.

\paragraph{Evaluating models with $NF^{(k)}$}

The previous process can be used to compute the $k$-hop fairness of a link predictor $h$.
For this, we compute the edge probability matrix $P = (h(v_i, v_j))_{1 \leq i, j \leq |\mathcal{V}|}$, and use $\Tilde{A}^{(k)}$ as a mask matrix:
\[
NF^{(k)}(P) = \underset{s_1, s_2, s \in \mathbb{S}}{max} |\phi_{s_1 \rightarrow s}^{(k)}(P \odot \Tilde{A}^{(k)}) - \phi_{s_2 \rightarrow s}^{(k)}(P \odot \Tilde{A}^{(k)})|
\]
This way, instead of relying on the sensitive attribute proportions within $k$-hop neighborhoods, $NF^{(k)}(P)$ considers existence probabilities, allowing us to evaluate the predictor's tendency to cluster sensitive groups.


\paragraph{Remark on Weighted Directed Graphs}

In this paper, we emphasize that the developed framework generalizes to weighted directed graphs. Indeed, in this case, the adjacency matrix contains the edge weights, and the property from~\cite{biggs1993algebraic} can be reformulated as follows: for all \(k' > 1\), \(A^{k'}_{i,j}\) represents the sum of weights of all walks of length \(k'\) connecting nodes \(v_i\) and \(v_j\), where the weight of a walk is defined as the product of the weights of the edges along that walk.  

Consequently, computing \(\phi_{s \rightarrow s'}^{(k)}(\tilde{A}^{(k)})\) corresponds to calculating the average proportion of information flow originating from nodes with sensitive attribute \(s'\) within the total flow received by nodes with sensitive attribute \(s\), which aligns with the conceptual framework discussed throughout this paper.

\section{Computation Challenges}
\label{sec:challenge}

\begin{table}[!ht]
    \centering
    \renewcommand{\arraystretch}{1.2}
    \setlength{\tabcolsep}{10pt}
    \begin{tabular}{lccccc}
    \toprule
    \textbf{Metric} & \textit{Polblogs} & \textit{Facebook} & \textit{Pokec} & \textit{Citeseer} & \textit{Synthetic} \\
    \midrule
    \# Nodes & 1,222 & 1,034 & 2,779 & 3,264 & 950 \\
    \# Edges & 16,714 & 26,749 & 22,650 & 4,536 & 9,496 \\
    \% Sensitive nodes & 48 & 15 & 34 & 33 & 32 / 16 \\
    Density & 0.022 & 0.050 & 0.006 & 0.001 & 0.021 \\
    \midrule
    $NB^{(1)}$ Compute Time (s) & 0.088 & 0.053 & 0.797 & 1.149 & 0.071 \\
    $NB^{(2)}$ Compute Time (s) & 0.130 & 0.085 & 0.872 & 1.148 & 0.093 \\
    $NB^{(3)}$ Compute Time (s) & 0.277 & 0.174 & 1.316 & 1.173 & 0.167 \\
    $NB^{(4)}$ Compute Time (s) & 0.380 & 0.280 & 1.902 & 1.235 & 0.203 \\
    $NB^{(5)}$ Compute Time (s) & 0.420 & 0.344 & -- & 1.246 & -- \\
    $NB^{(6)}$ Compute Time (s) & 0.561 & 0.580 & -- & 1.363 & -- \\
    \bottomrule
    \end{tabular}
    \caption{Compute time (in seconds) for different $NB^{(k)}$ across datasets and for $k$ in meaningful hop set $\mathcal{K}$, with recursive method. Compute time are averaged over 10 iterations.}
    \label{tab:nb_compute_time}
\end{table}

This section outlines the challenges involved in computing the metrics used in the experiments presented in this paper.

\paragraph{Proof of the $\Tilde{A}^{(k)}$ formula}
In the paper, $k$-hop matrix $\Tilde{A}^{(k)} \in \{0,1\}^{n \times n}$ is defined as $\Tilde{A}^{(k)}_{i, j} = \bbone\{\sigma(v_i, v_j) = k\}$, which indicates whether two nodes are $k$ hops apart.

To this end, we state the following formula:
\[
\Tilde{A}^{(k)} = \mathop{\chi}\Big(\mathop{\chi}(A^k) - \mathop{\chi}\big(\sum\limits_{i=1}^{k-1}A^i\big)\Big ) \odot \Big(\textbf{1}_{n, n} - I_n\Big),
\]
where for a given graph, $A \in \{0,1\}^{n \times n}$ is the adjacency matrix, $n$ is the number of nodes, for a matrix $M$, $\mathop{\chi}(M) = (\bbone\{ M_{i,j} > 0 \})_{i, j}$, $I_p$ is the identity matrix of size $p$, and $\odot$ denotes the element-wise (Hadamard) product between matrices.

\begin{proof}

The proof relies on the fact that $\forall k' > 1, A^{k'}_{i, j}$ is the number of walks of length $k'$ linking nodes $v_i$ and $v_j$~\cite{biggs1993algebraic}.

Suppose that for specific $i \neq j$, $\sigma(v_i, v_j) = k$.
\begin{equation*}
    \begin{aligned}
        &\Leftrightarrow \forall m \leq k-1, A^m_{i, j} = 0 \land A^{k}_{i, j} > 1 \land v_i \neq v_j \\
        &\Leftrightarrow \sum\limits_{m=1}^{k-1} A^m_{i, j} = 0 \land A^{k}_{i, j} > 1\land v_i \neq v_j\\
        &\Leftrightarrow \mathop{\chi}\Big(\mathop{\chi}(A^k) - \mathop{\chi}\big(\sum\limits_{i=1}^{k-1}A^i\big)\Big )_{i, j} = 1 \land \Big(\textbf{1}_{n, n} - I_n\Big)_{i, j} = 1\\
        &\Leftrightarrow \Big(\mathop{\chi}\Big(\mathop{\chi}(A^k) - \mathop{\chi}\big(\sum\limits_{i=1}^{k-1}A^i\big)\Big ) \odot \Big(\textbf{1}_{n, n} - I_n\Big)\Big)_{i, j} = 1
    \end{aligned}
\end{equation*}
Hence the result.
    
\end{proof}

\paragraph{Recursive Computation}

Alternatively, we may prefer to compute \(\Tilde{A}^{(k)}\) using its recursive formulation:

\small
\begin{equation*}
  \begin{aligned}
    &\Tilde{A}^{(1)} = A, \quad \\
&\Tilde{A}^{(k)} = \chi\left(\chi\left(\Tilde{A}^{(k-1)} \times A\right) - \chi\left(\sum_{i=1}^{k-1} \Tilde{A}^{(i)}\right)\right) \odot \left(\mathbf{1}_{n, n} - I_n\right),
\end{aligned}  
\end{equation*}

The proof of correctness follows the same reasoning as the direct formulation presented above.

\paragraph{Complexity Analysis}

Using dense matrix representations, computing \(A^k\) has a time complexity of \(\mathcal{O}(k n^3)\). 

We focus here on the recursive formulation of the computation of \(\Tilde{A}^{(k)}\).
With sparse matrix representations, the product of two square matrices \(B \times C\) of size \(n \times n\) has complexity \(\mathcal{O}(\text{nnz}(B) \cdot d_C)\), where \(\text{nnz}(B)\) denotes the number of non-zero elements in \(B\) and \(d_C\) is the average number of non-zero entries per row in \(C\).

Considering the recursive formula for \(\Tilde{A}^{(k)}\), the dominant cost lies in the computation of \(\Tilde{A}^{(k-1)} \times A\).

In the worst case, the number of non-zero elements in \(\Tilde{A}^{(k)}\) is \(\mathcal{O}(n^2)\), corresponding to the scenario where all pairs of nodes are at distance \(k\). Given that the average number of non-zero elements per row in \(A\) is \(d = m / n\), where \(m\) is the number of edges, the complexity of the matrix multiplication \(\Tilde{A}^{(k-1)} \times A\) is thus \(\mathcal{O}(n^2 \cdot d) = \mathcal{O}(n m)\).

Therefore, the total complexity for computing \(\Tilde{A}^{(k)}\) for \(k\) iterations is:
\[
\mathcal{O}(k \, n \, m).
\]

Under the common sparsity assumption $m = \mathcal{O}(n)$, this gives us the following result.
\begin{property}
For a graph $G$, $\alpha \in \mathbb{R}$, and $k \in \mathcal{K}$, the cost of one iteration in the minimization of $\ell_{\text{pre}}^{(k)}$ is $\mathcal{O}(k n^{2})$.
\end{property}

Alternatively, when differentiability is not required, computing these quantities only involves a breadth-first search (BFS) from each node, which has worst-case complexity 
\[
\mathcal{O}(n \cdot (n+m)).
\] 
All other operations have lower complexity, thus, under the sparsity assumption, this yields an overall complexity of 
\[
\mathcal{O}(n^2).
\]
Hence,
\begin{property}
For a graph $G$, for a link predictor $h$ and for $k \in \mathcal{K}$,
computing $NF^{(k)}(h)$ and $NB^{(k)}(G)$ is $\mathcal{O}(n^2)$.
\end{property}

\paragraph{Compute Time}

Table~\ref{tab:nb_compute_time} reports several characteristics of the graphs used in the experiments, along with the execution time required to compute the $k$-hop biases $NB^{(k)}$ for different hops.

All computations were performed using a single CPU, according to the infrastructure described in the following section.

Note that the matrices $\Tilde{A}^{(k)}$ can be stored in memory to avoid redundant computations throughout the rest of the framework.

\paragraph{Computing Infrastructure}
\label{infra}
All experiments were run on a MacBook with an Apple M3 Pro chip, 11-core CPU, and 18GB of unified memory. No GPU was used. All code was executed on CPU only. The software environment was based on Python 3.11.6. Versions of all relevant packages are specified in the provided code.

\cleardoublepage
\section{Additional results}
\subsection{Relationship Between k-hop Structural Bias and Fairness (RQ1)}
We present here the result for the \textit{Synthetic} dataset.  
The figures highlight the particularly strong structural biases \(NB^{(k)}\) embedded in this graph, with values above $0.6$ for hops $k \in \{1, 2, 4 \}$  .  

\begin{figure}[htbp]
    \centering
    \includegraphics[width=.4\linewidth]{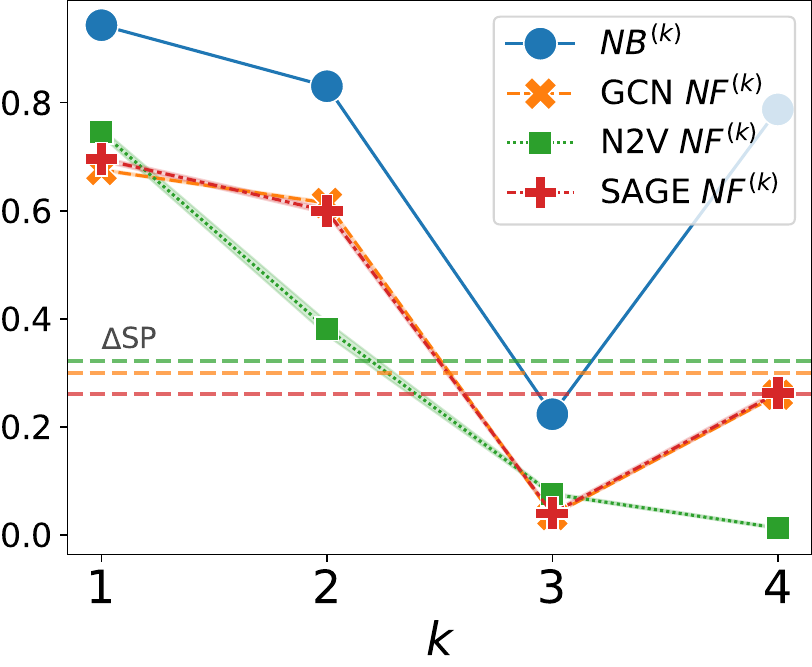}
    \caption{Comparison between graph structural bias and classical LP models $k$-hop fairness across meaningful hops for \textbf{Synthetic} dataset.}
\end{figure}

\subsection{Dependencies Across Different Hops (RQ2)}

We present here the result for the \textit{Synthetic} dataset.  
We observe that all when targeting $NB^{(1)}$ and $NB^{(2)}$ biases are generally positively correlated under an edge adjunction process, except for $NB^{(3)}$.

\begin{figure}[H]
    \centering
    \includegraphics[width=.4\linewidth]{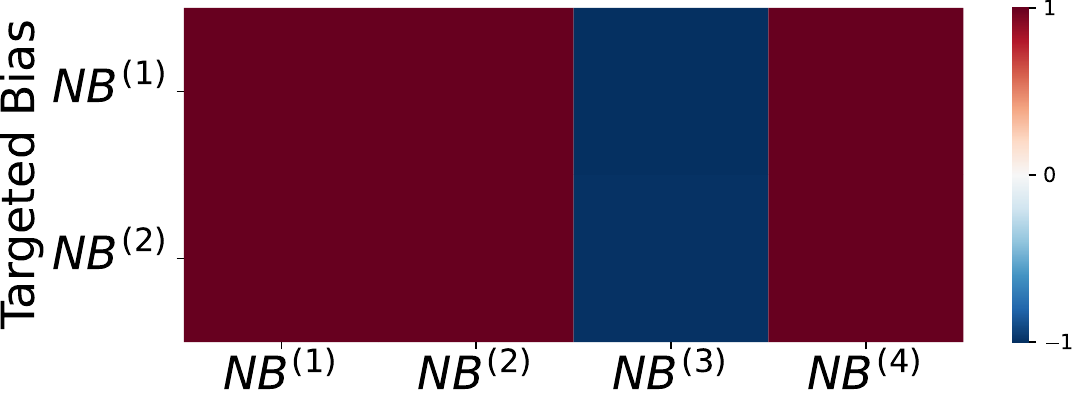}
    \caption{Influence of $NB^{(k)}$ minimization via edge addition on structural bias at other hops for \textbf{Synthetic} dataset. }
\end{figure}

\begin{figure}[ht]
    \centering
    \begin{subfigure}{0.9\linewidth}
        \centering \includegraphics[width=.4\linewidth]{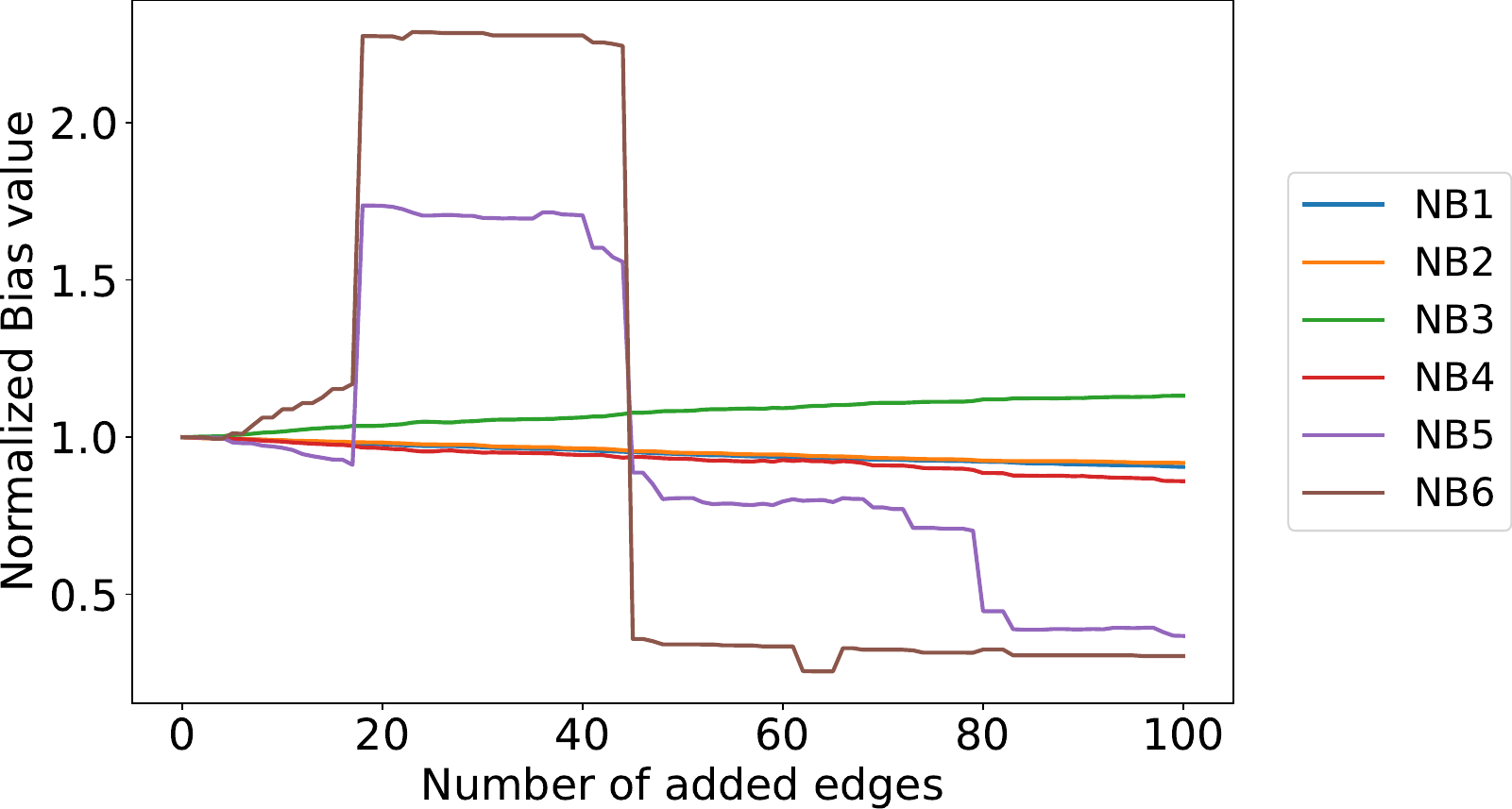}
        \caption{Polblogs - $k=1$}
    \end{subfigure}
    \begin{subfigure}{0.9\linewidth}
        \centering \includegraphics[width=.4\linewidth]{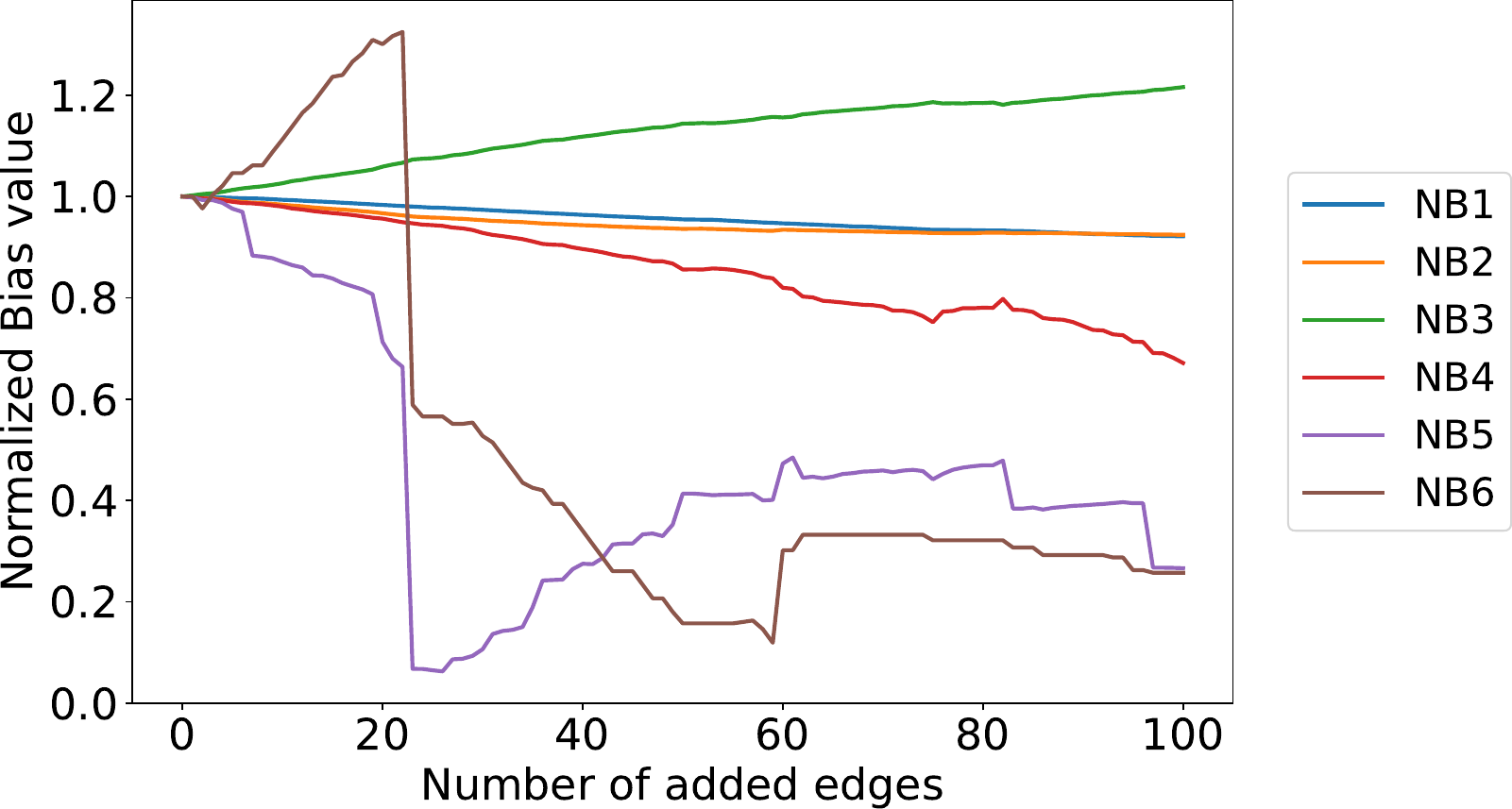}
        \caption{Polblogs - $k=2$}
    \end{subfigure}
    \begin{subfigure}{0.9\linewidth}
        \centering \includegraphics[width=.4\linewidth]{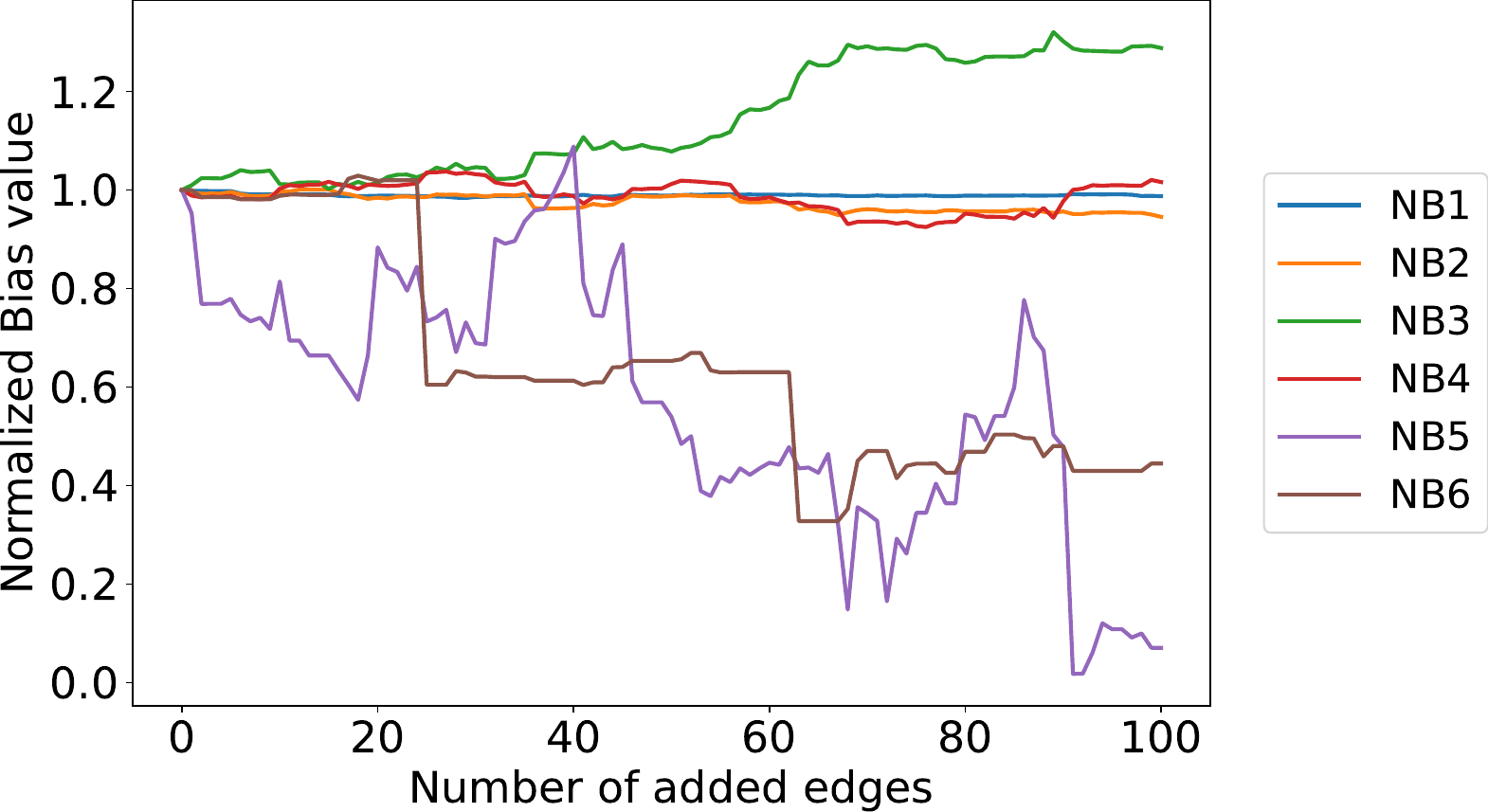}
        \caption{Facebook - $k=2$}
    \end{subfigure}
    \begin{subfigure}{0.9\linewidth}
        \centering \includegraphics[width=.4\linewidth]{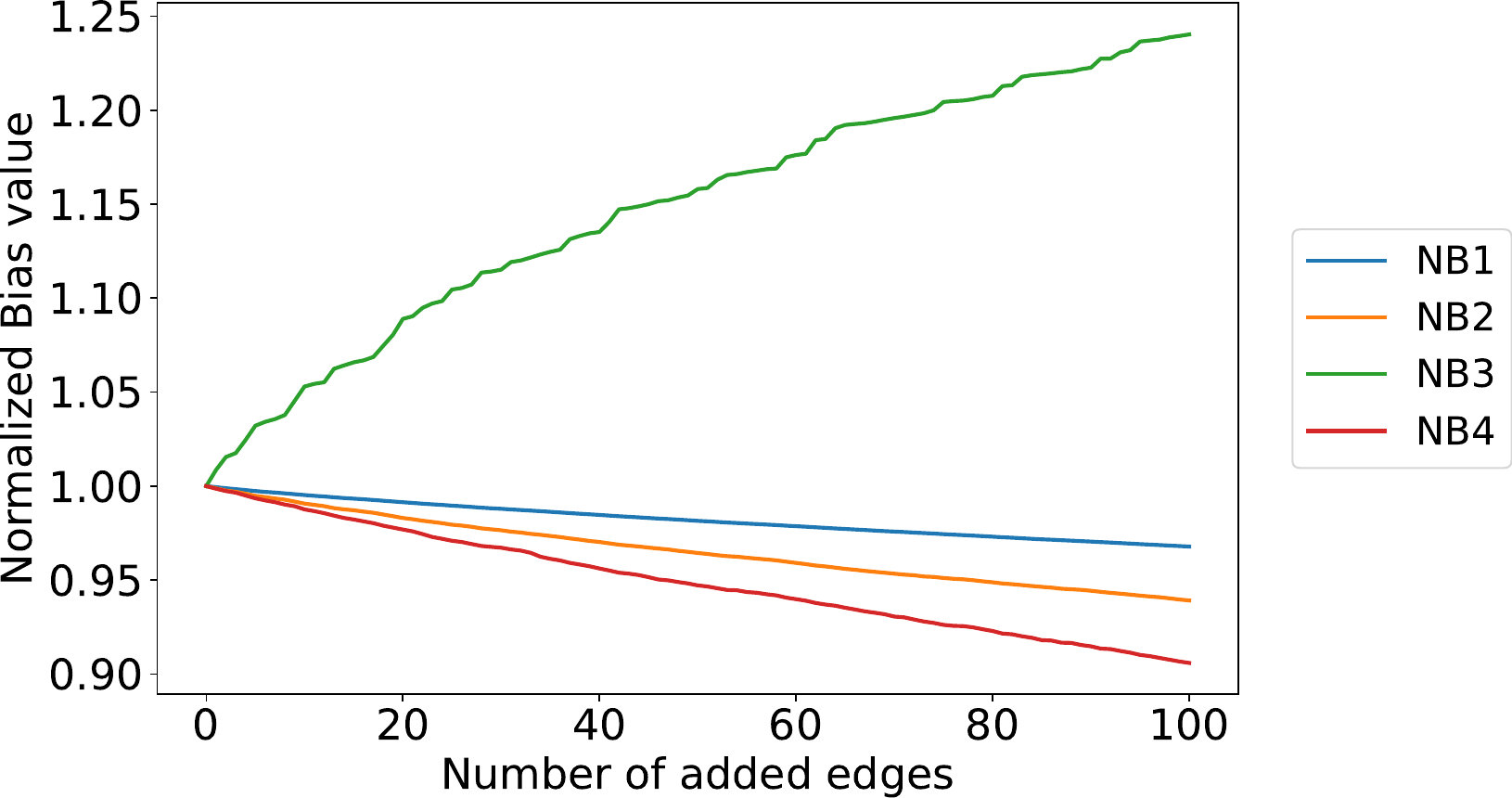}
        \caption{Synthetic - $k=1$}
    \end{subfigure}
    \begin{subfigure}{0.9\linewidth}
        \centering \includegraphics[width=.4\linewidth]{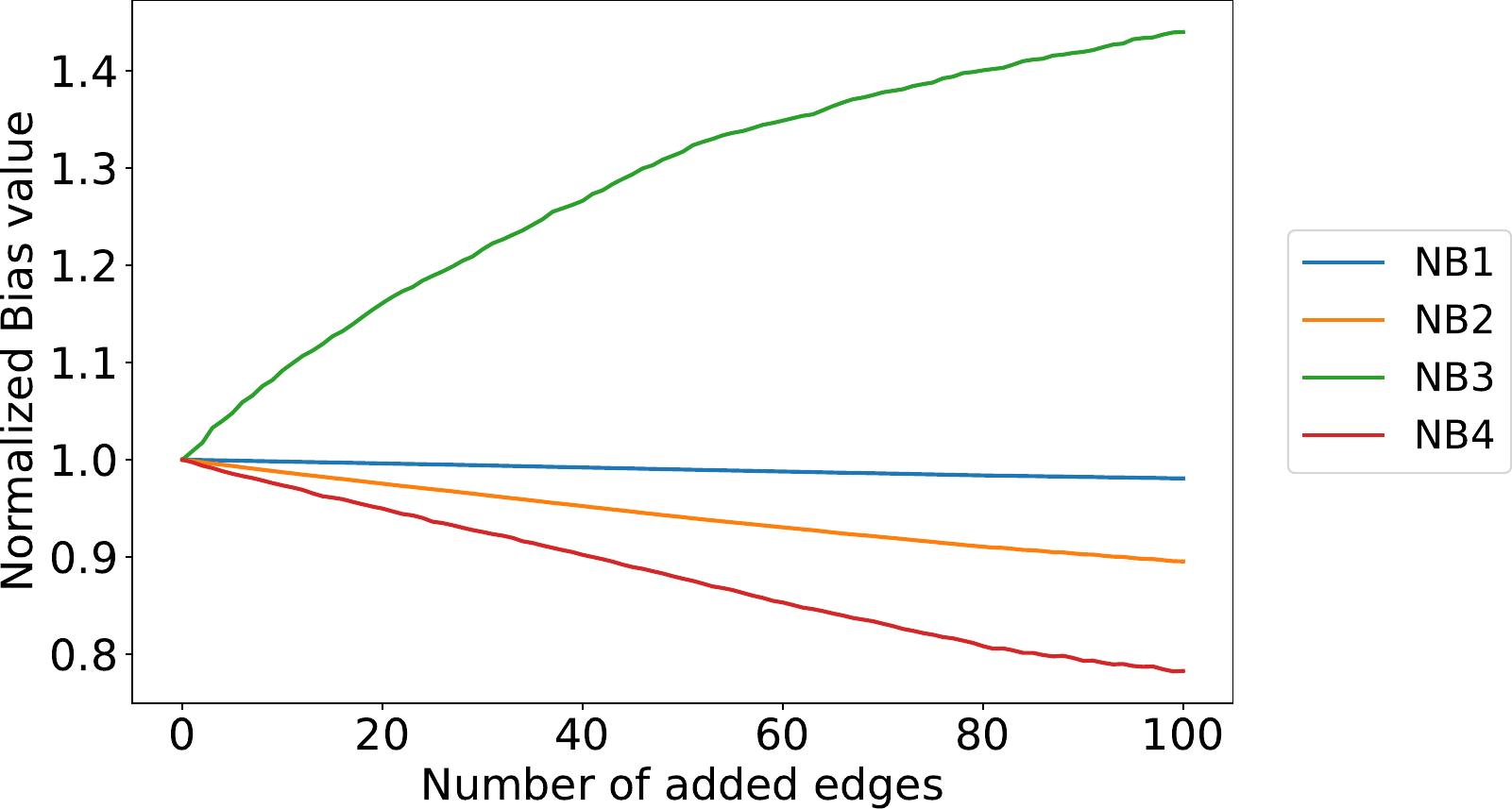}
        \caption{Synthetic - $k=2$}
    \end{subfigure}
    \caption{Evolution of structural bias across edge addition rewiring for several values of targeted $k$ for \textbf{Polblogs}, \textbf{Facebook} and \textbf{Synthetic} datasets. Bias values are normalized (divided by original value) to increase readability.}
\end{figure}

\begin{figure}[t]
    \centering
    \begin{subfigure}{0.9\linewidth}
        \centering \includegraphics[width=.4\linewidth]{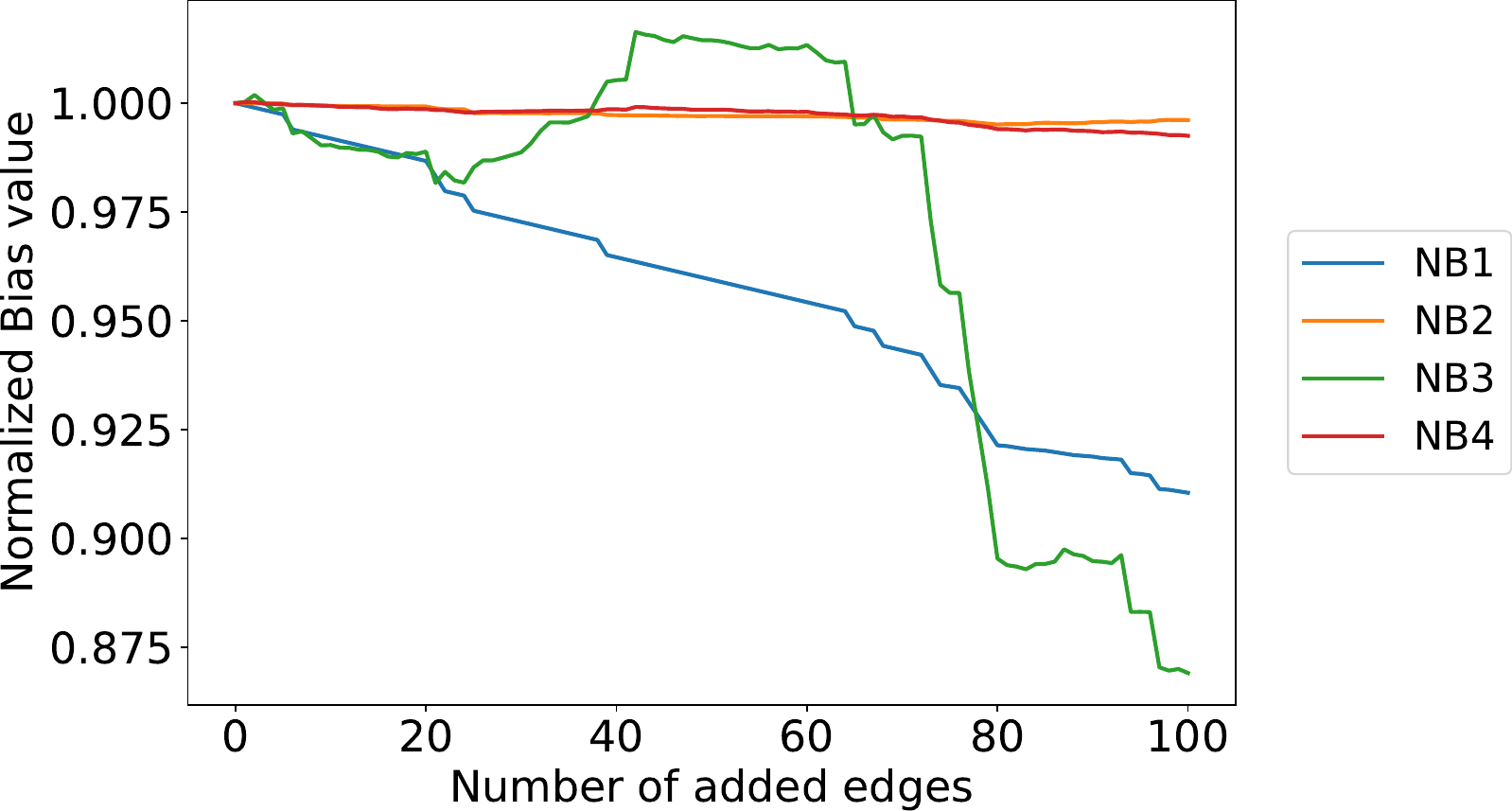}
        \caption{Pokec - $k=1$}
    \end{subfigure}
    \begin{subfigure}{0.9\linewidth}
        \centering \includegraphics[width=.4\linewidth]{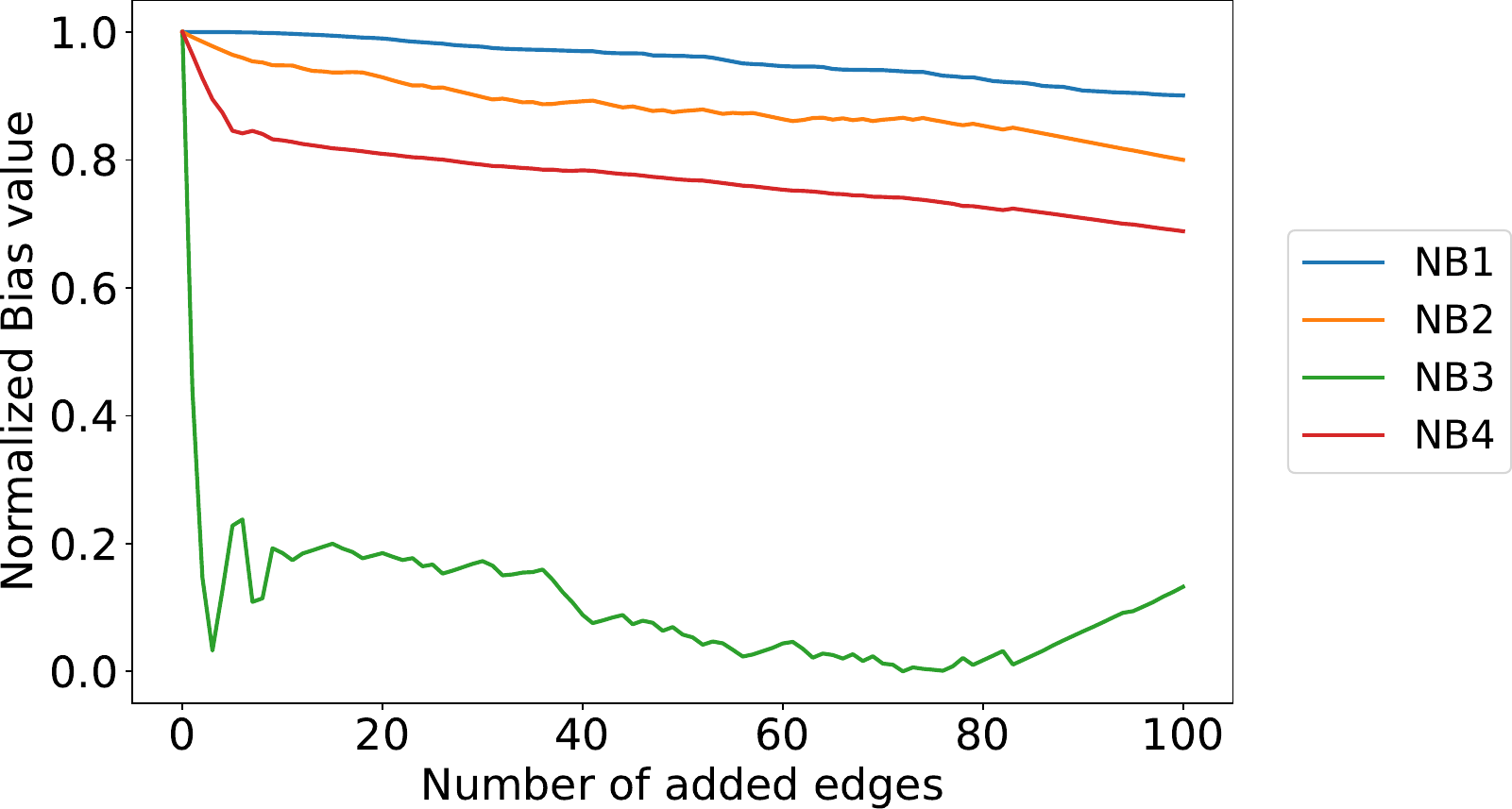}
        \caption{Pokec - $k=2$}
    \end{subfigure}
    \begin{subfigure}{0.9\linewidth}
        \centering \includegraphics[width=.4\linewidth]{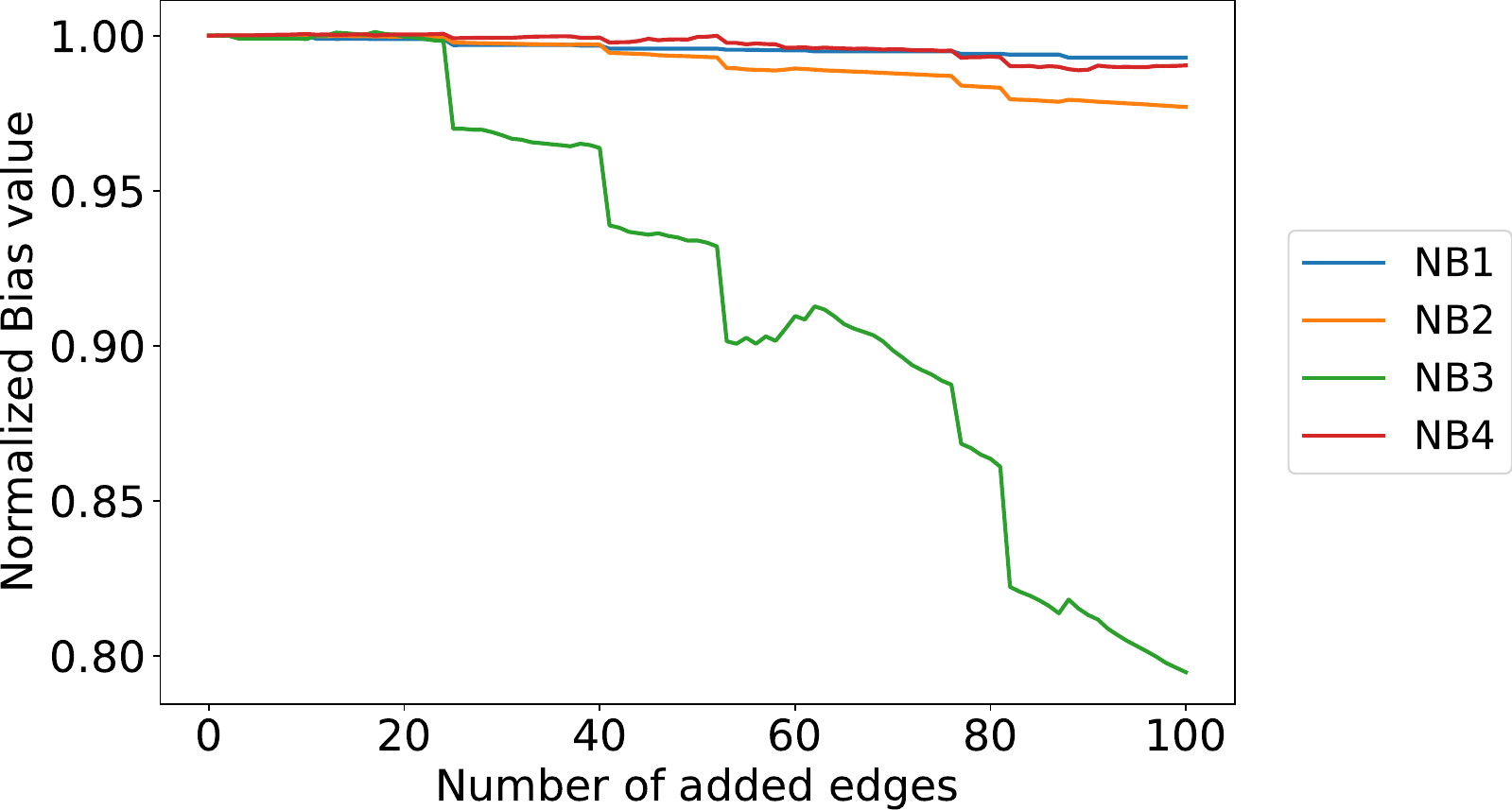}
        \caption{Pokec - $k=4$}
    \end{subfigure}
    \caption{Evolution of structural bias across edge addition rewiring for several values of targeted $k$ for \textbf{Pokec} dataset. Bias values are normalized (divided by original value) to increase readability.}
\end{figure}

\begin{figure}[t]
    \centering
    \begin{subfigure}{0.9\linewidth}
        \centering \includegraphics[width=.4\linewidth]{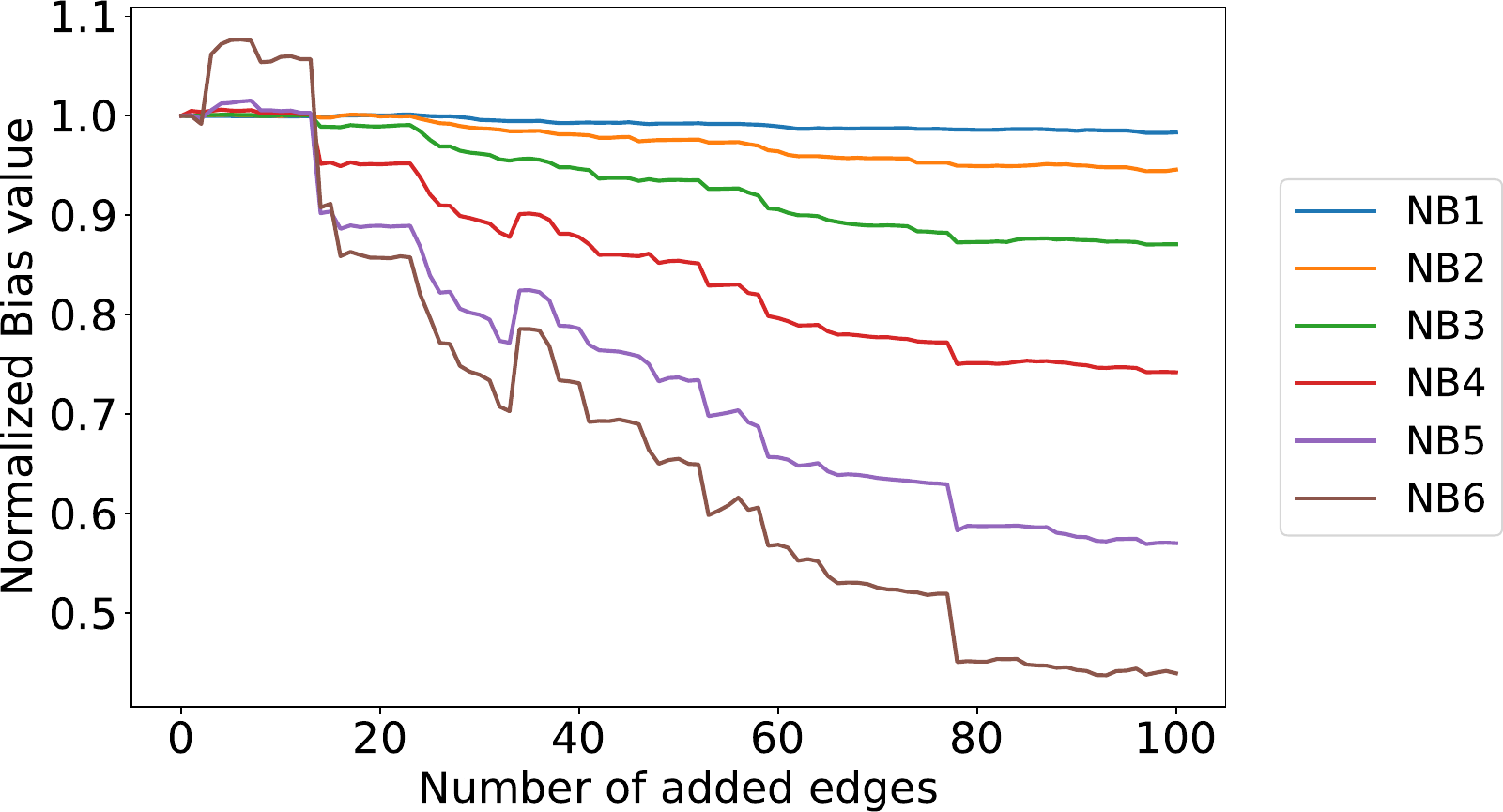}
        \caption{Citeseer - $k=1$}
    \end{subfigure}
    \begin{subfigure}{0.9\linewidth}
        \centering \includegraphics[width=.4\linewidth]{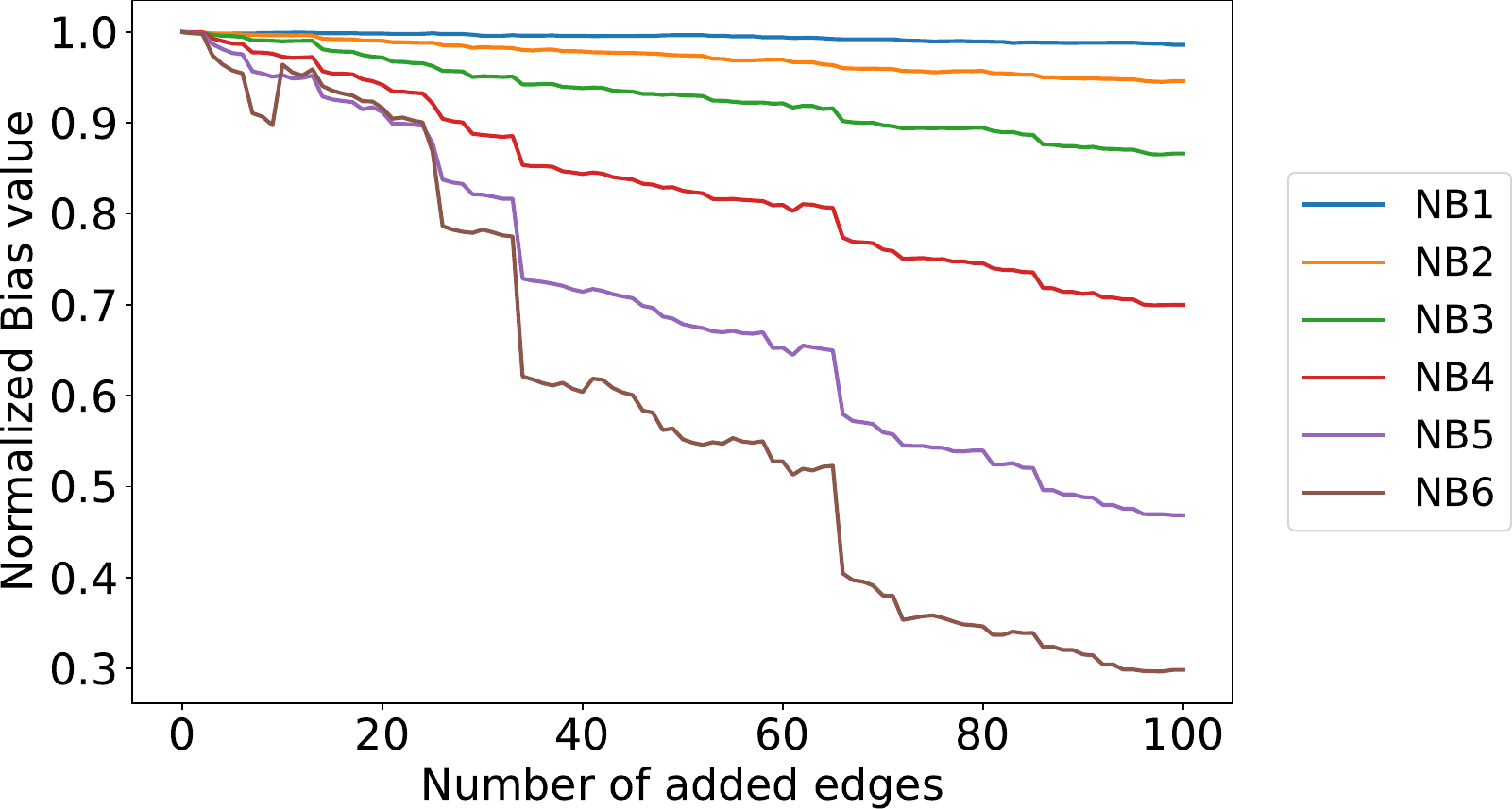}
        \caption{Citeseer - $k=2$}
    \end{subfigure}
    \begin{subfigure}{0.9\linewidth}
        \centering \includegraphics[width=.4\linewidth]{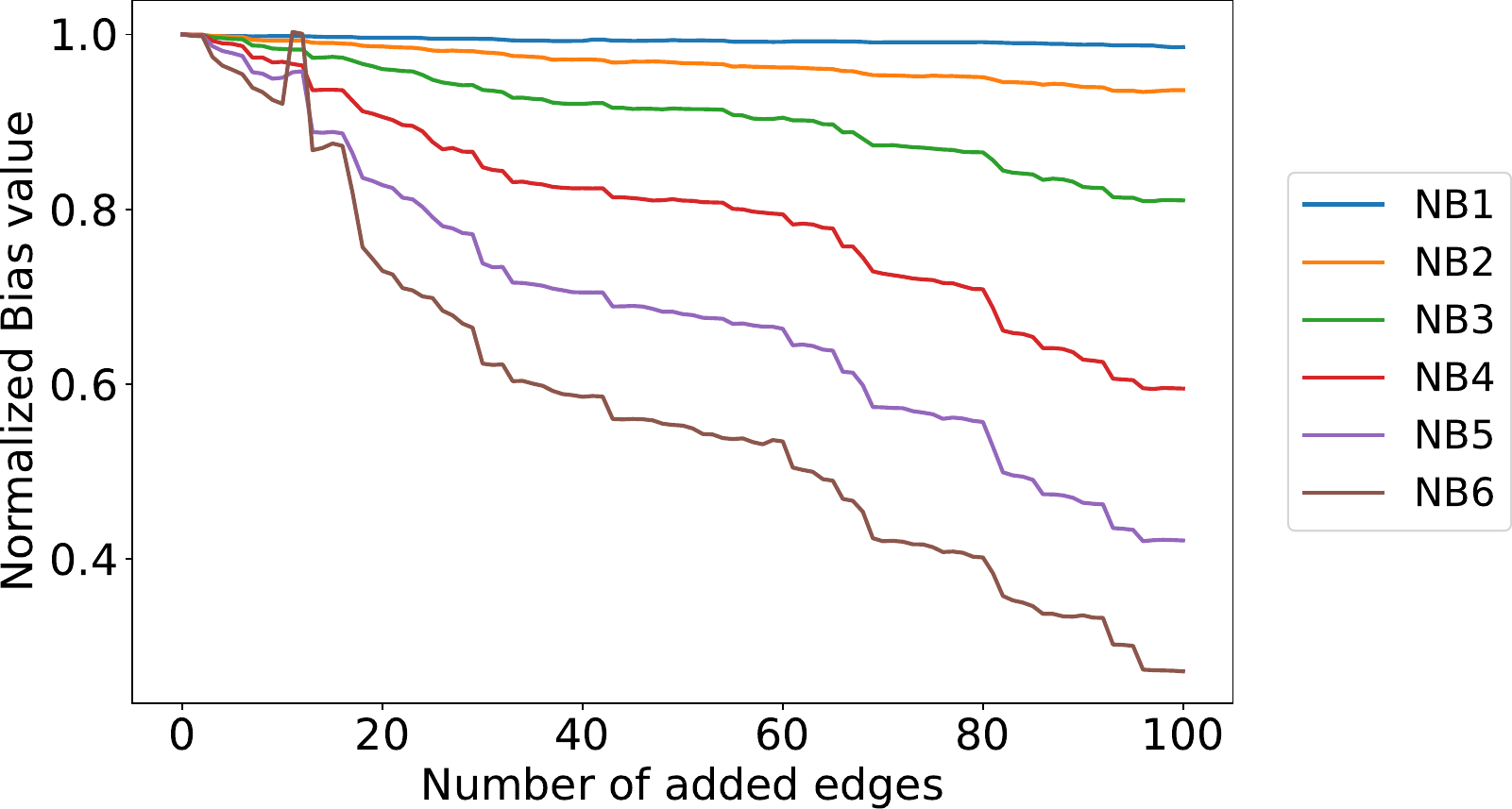}
        \caption{Citeseer - $k=3$}
    \end{subfigure}
    \caption{Evolution of structural bias across edge addition rewiring for several values of targeted $k$ for \textbf{Citeseer} dataset. Bias values are normalized (divided by original value) to increase readability.}
\end{figure}

\clearpage
\subsection{Post-processing for k-hop Fair LP (RQ3)}

In Table~\ref{tab:rq3}, we report the impact of our post-processing method on dyadic metrics. Overall, no consistent trend is observed across datasets or hop orders. This highlights that improvements in $NF^{(k)}$ are not trivially reflected in dyadic metrics, reinforcing the complementarity of our $k$-hop perspective. In particular, on \textit{Facebook}, the absence of noticeable changes in $\Delta DP$ and $\Delta EO$ is consistent with the low initial level of bias.

\begin{table}[ht]
\centering
\renewcommand{\arraystretch}{1.05}
\begin{tabular}{l c p{1.3cm} p{1.3cm} p{1.3cm} p{1.3cm}}
\toprule
Dataset & $k$ & $\Delta NF^{(k)}$ & $\Delta AUC$ & $\Delta DP$ & $\Delta EO$ \\
\midrule
\textit{Polblogs} & $1$ & \scriptsize -0.46 \tiny ± 0.01 & \scriptsize 0.00 \tiny ± 0.00 & \scriptsize -0.00 \tiny ± 0.00 & \scriptsize -0.00 \tiny ± 0.00 \\
 & $2$ & \scriptsize -0.25 \tiny ± 0.01 & \scriptsize -0.28 \tiny ± 0.01 & \scriptsize -0.07 \tiny ± 0.03 & \scriptsize 0.31 \tiny ± 0.04 \\
 & $4$ & \scriptsize -0.17 \tiny ± 0.01 & \scriptsize -0.00 \tiny ± 0.00 & \scriptsize 0.05 \tiny ± 0.00 & \scriptsize 0.00 \tiny ± 0.00 \\
\midrule
\textit{Facebook} & $1$ & \scriptsize -0.07 \tiny ± 0.00 & \scriptsize 0.00 \tiny ± 0.00 & \scriptsize 0.00 \tiny ± 0.00 & \scriptsize -0.00 \tiny ± 0.00 \\
 & $2$ & \scriptsize -0.03 \tiny ± 0.00 & \scriptsize -0.00 \tiny ± 0.00 & \scriptsize 0.02 \tiny ± 0.00 & \scriptsize 0.04 \tiny ± 0.00 \\
 & $6$ & \scriptsize -0.01 \tiny ± 0.01 & \scriptsize 0.00 \tiny ± 0.00 & \scriptsize 0.00 \tiny ± 0.00 & \scriptsize -0.00 \tiny ± 0.00 \\
\midrule
\textit{Pokec} & $1$ & \scriptsize -0.28 \tiny ± 0.01 & \scriptsize 0.00 \tiny ± 0.00 & \scriptsize 0.00 \tiny ± 0.00 & \scriptsize -0.00 \tiny ± 0.00 \\
 & $2$ & \scriptsize -0.18 \tiny ± 0.01 & \scriptsize -0.07 \tiny ± 0.01 & \scriptsize -0.07 \tiny ± 0.00 & \scriptsize -0.03 \tiny ± 0.04 \\
 & $4$ & \scriptsize -0.01 \tiny ± 0.00 & \scriptsize 0.00 \tiny ± 0.00 & \scriptsize 0.01 \tiny ± 0.00 & \scriptsize 0.00 \tiny ± 0.00 \\
\midrule
\textit{Citeseer} & $1$ & \scriptsize -0.43 \tiny ± 0.01 & \scriptsize 0.00 \tiny ± 0.00 & \scriptsize -0.00 \tiny ± 0.00 & \scriptsize 0.00 \tiny ± 0.00 \\
 & $2$ & \scriptsize -0.26 \tiny ± 0.01 & \scriptsize -0.13 \tiny ± 0.01 & \scriptsize 0.01 \tiny ± 0.03 & \scriptsize 0.14 \tiny ± 0.06 \\
 & $3$ & \scriptsize -0.14 \tiny ± 0.01 & \scriptsize -0.03 \tiny ± 0.00 & \scriptsize -0.04 \tiny ± 0.01 & \scriptsize 0.03 \tiny ± 0.06 \\
\midrule
\textit{Synthetic} & $1$ & \scriptsize -0.48 \tiny ± 0.06 & \scriptsize 0.00 \tiny ± 0.00 & \scriptsize 0.00 \tiny ± 0.00 & \scriptsize -0.00 \tiny ± 0.00 \\
 & $2$ & \scriptsize -0.45 \tiny ± 0.02 & \scriptsize -0.25 \tiny ± 0.01 & \scriptsize -0.29 \tiny ± 0.01 & \scriptsize -0.25 \tiny ± 0.02 \\
 & $4$ & \scriptsize -0.25 \tiny ± 0.00 & \scriptsize 0.00 \tiny ± 0.00 & \scriptsize 0.18 \tiny ± 0.00 & \scriptsize 0.19 \tiny ± 0.01 \\
\bottomrule
\end{tabular}
\caption{Effects of our post-processing method on dyadic fairness metrics DP, EO, performance AUC and targeted $k$-hop metrics on GCN base model (average ± standard deviation). }
\label{tab:rq3}
\end{table}

We now present the extended results for \textbf{RQ3}, including AUC vs. $NF^{(k)}$ plots and $\alpha$ sensitivity analysis plots (Figures~\ref{fig:rq3Pol}--\ref{fig:rq3Sage}), for the five datasets used in this paper.
Also, in Table~\ref{table:full1}--\ref{table:full2}, we present numerical evaluations of the baselines and proposed models over all the datasets and metrics, including the standard deviations computed over 10 iterations.

The extended results confirm the conclusions previously reported: our models achieve the best fairness scores with respect to the targeted hops, at the cost of a certain trade-off with AUC, depending on the targeted hop.  
The standard deviations further support the robustness of these results.

\clearpage

\begin{figure}[htbp]
    \centering
    \begin{subfigure}{0.9\linewidth}
        \centering \includegraphics[width=.4\linewidth]{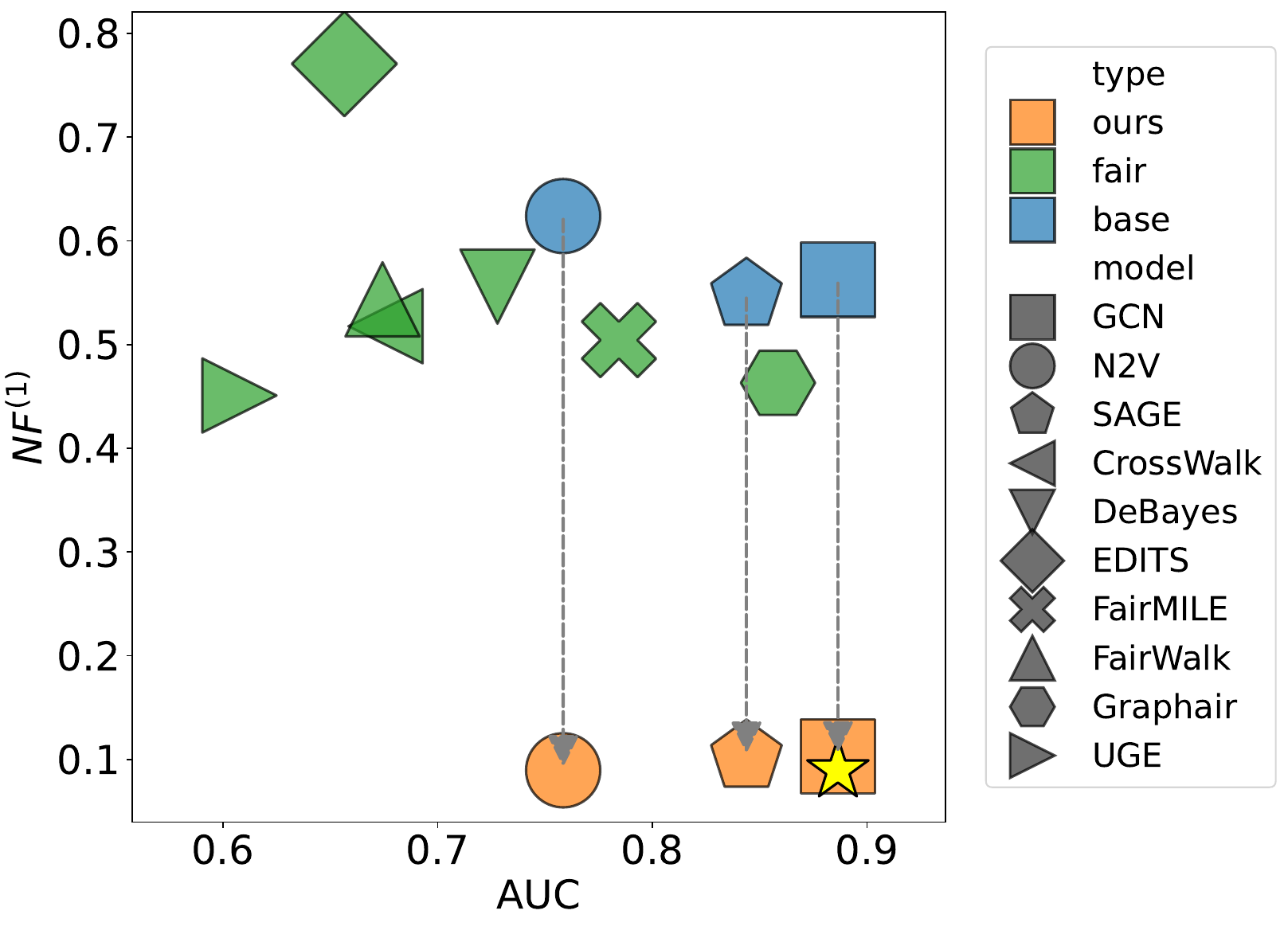}
        \caption{Polblogs - $k=1$}
    \end{subfigure}
    \begin{subfigure}{0.9\linewidth}
        \centering \includegraphics[width=.4\linewidth]{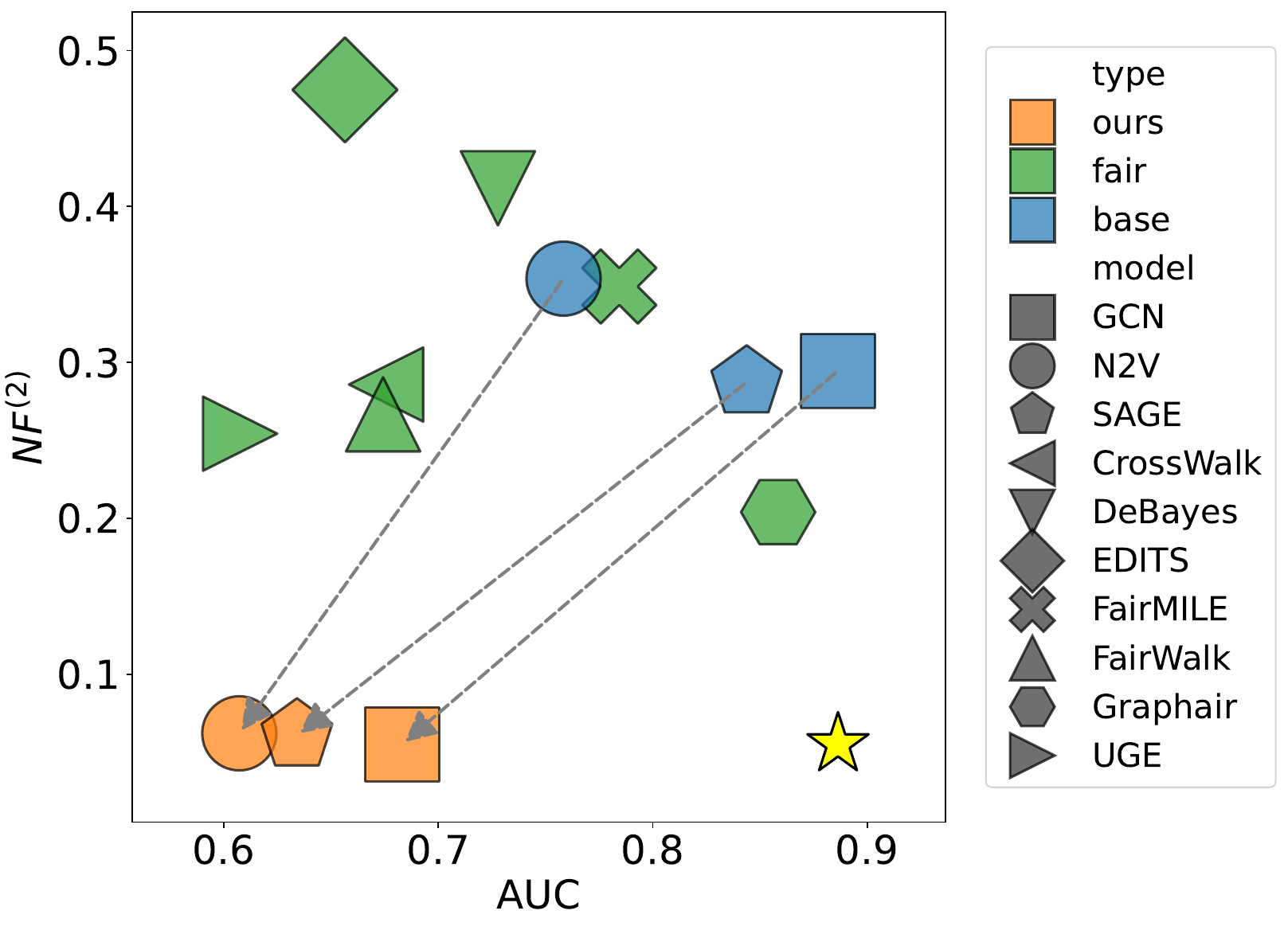}
        \caption{Polblogs - $k=2$}
    \end{subfigure}
    \begin{subfigure}{0.9\linewidth}
        \centering \includegraphics[width=.4\linewidth]{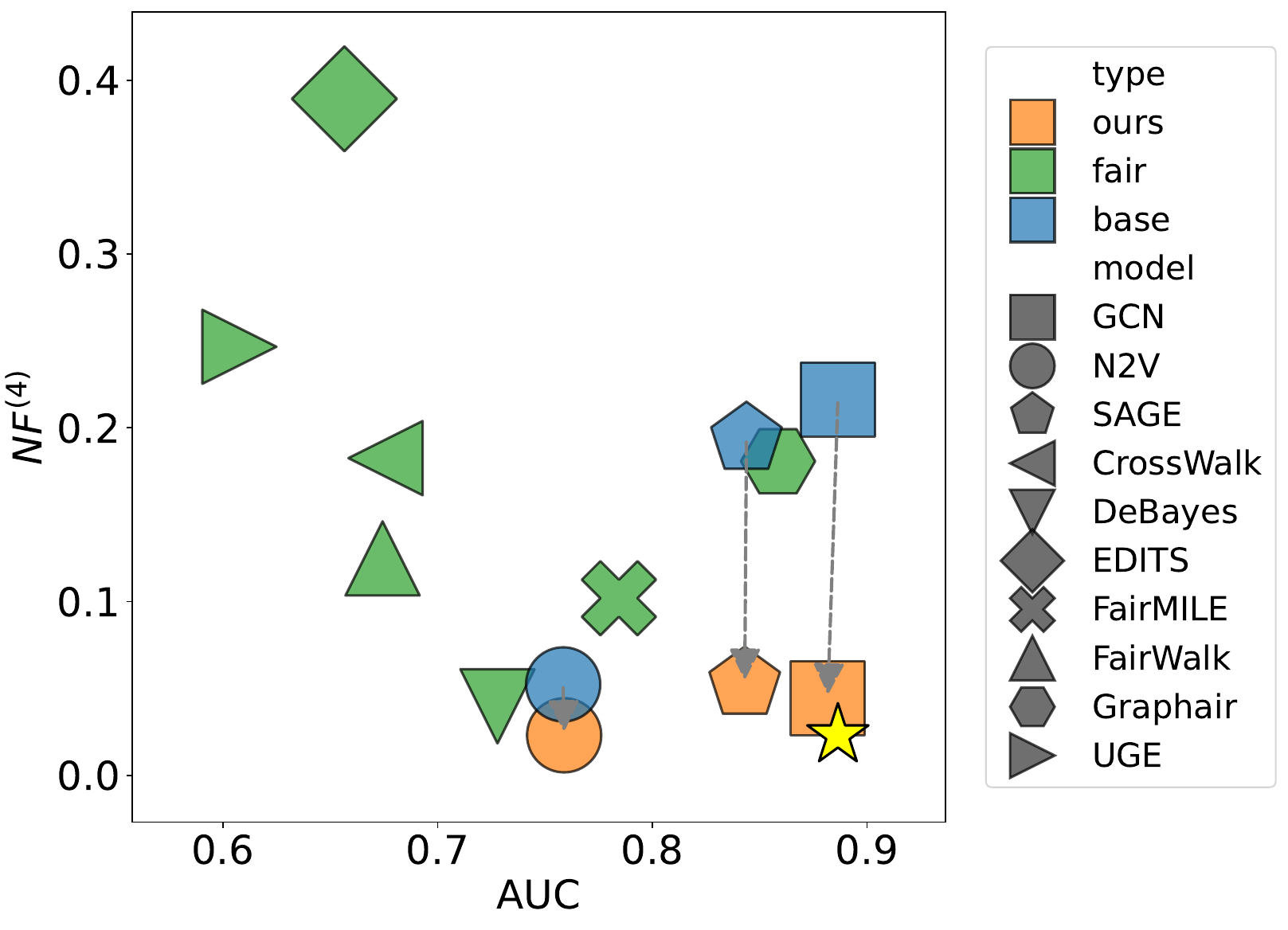}
        \caption{Polblogs - $k=4$}
    \end{subfigure}
    \caption{AUC vs $NF^{(k)}$ plot for base and fairness-aware baselines and for targeted $k$ values. Our post-processing results are shown through arrows from base models. The star marker denotes the hypothetical model achieving the best AUC and the best $NF^{(k)}$ across baselines. Only baselines with AUC $> 0.55$ are displayed.}
    \label{fig:rq3Pol}
\end{figure}

\begin{figure}[htbp]
    \centering
    \begin{subfigure}{0.9\linewidth}
        \centering \includegraphics[width=.4\linewidth]{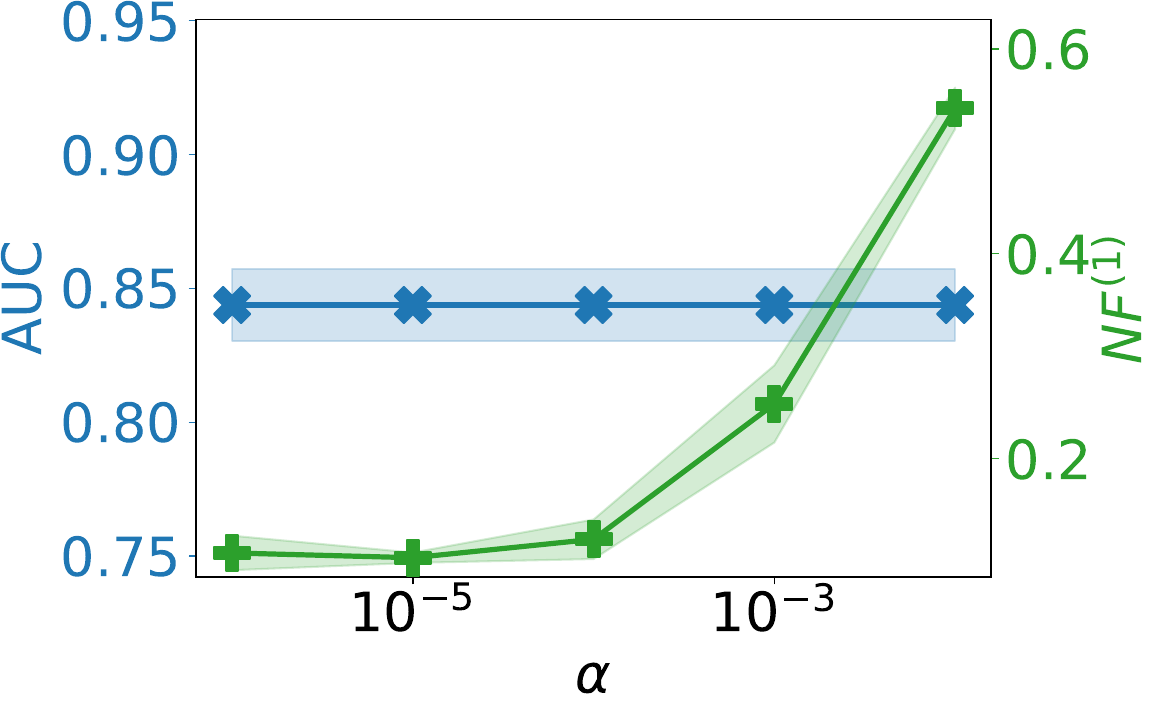}
        \caption{Polblogs - $k=1$}
    \end{subfigure}
    \begin{subfigure}{0.9\linewidth}
        \centering \includegraphics[width=.4\linewidth]{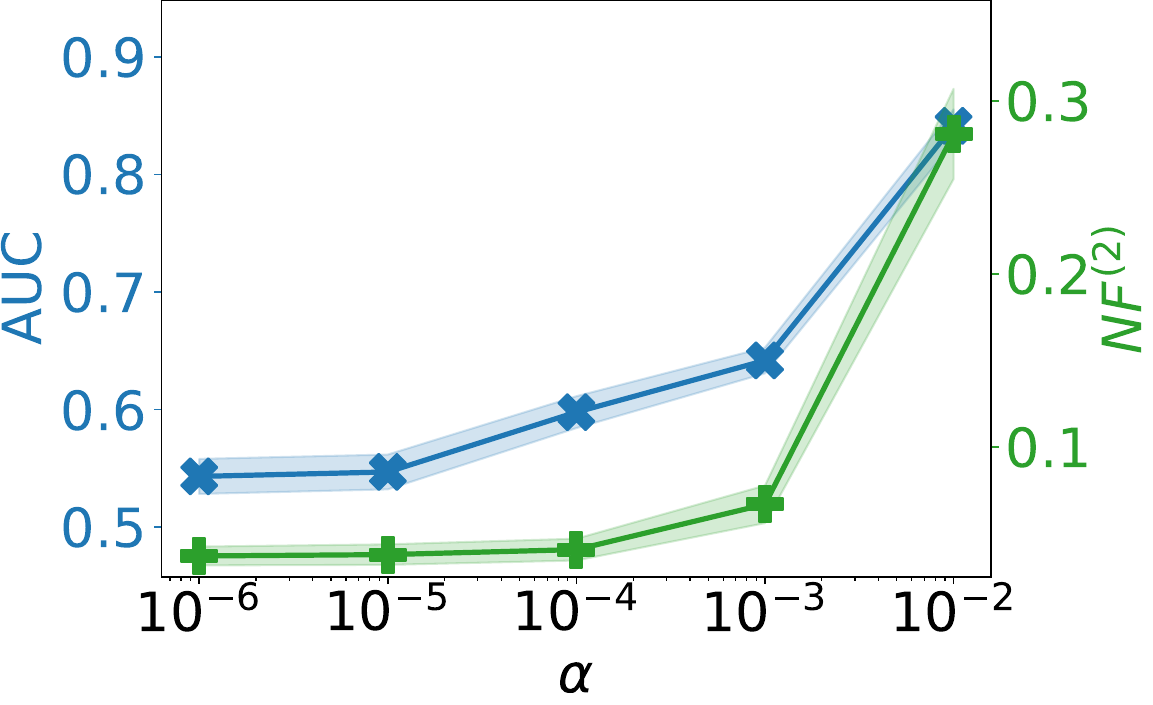}
        \caption{Polblogs - $k=2$}
    \end{subfigure}
    \begin{subfigure}{0.9\linewidth}
        \centering \includegraphics[width=.4\linewidth]{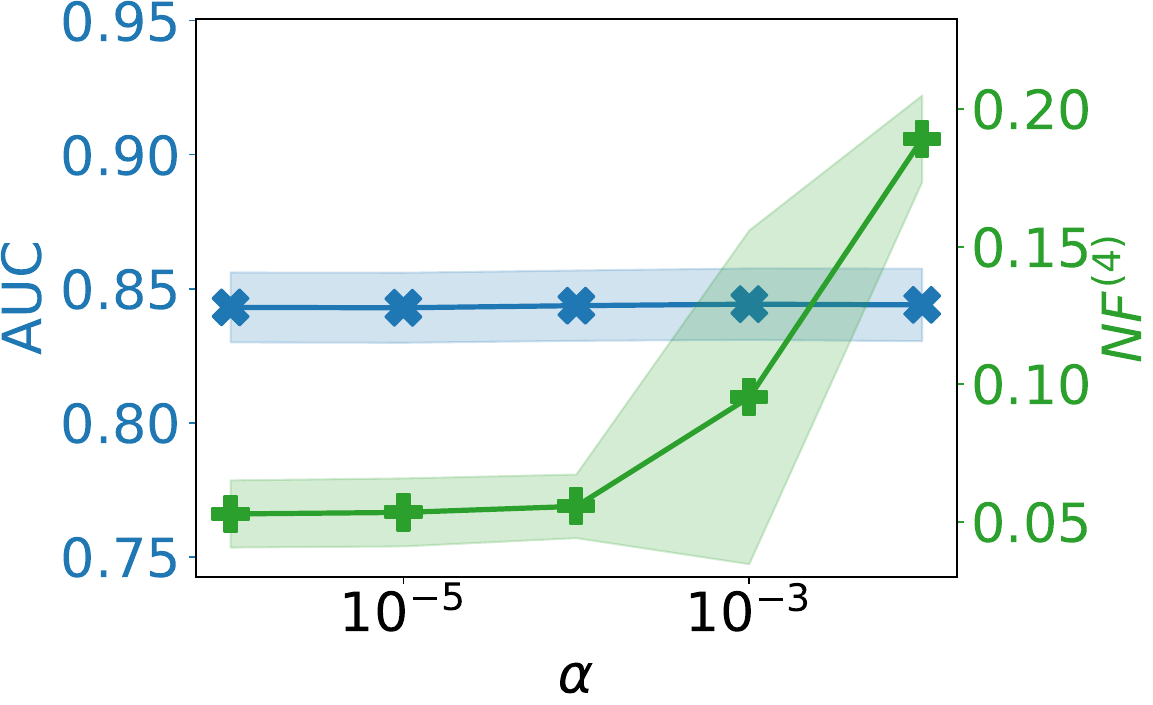}
        \caption{Polblogs - $k=4$}
    \end{subfigure}
    \caption{$\alpha$ effect on $NF^{(k)}$ and $AUC$ in the post-processing method for \textbf{SAGE} predictions (log-scale for x-axis).}
\end{figure}

\begin{figure}[htbp]
    \centering
    \begin{subfigure}{0.9\linewidth}
        \centering \includegraphics[width=.4\linewidth]{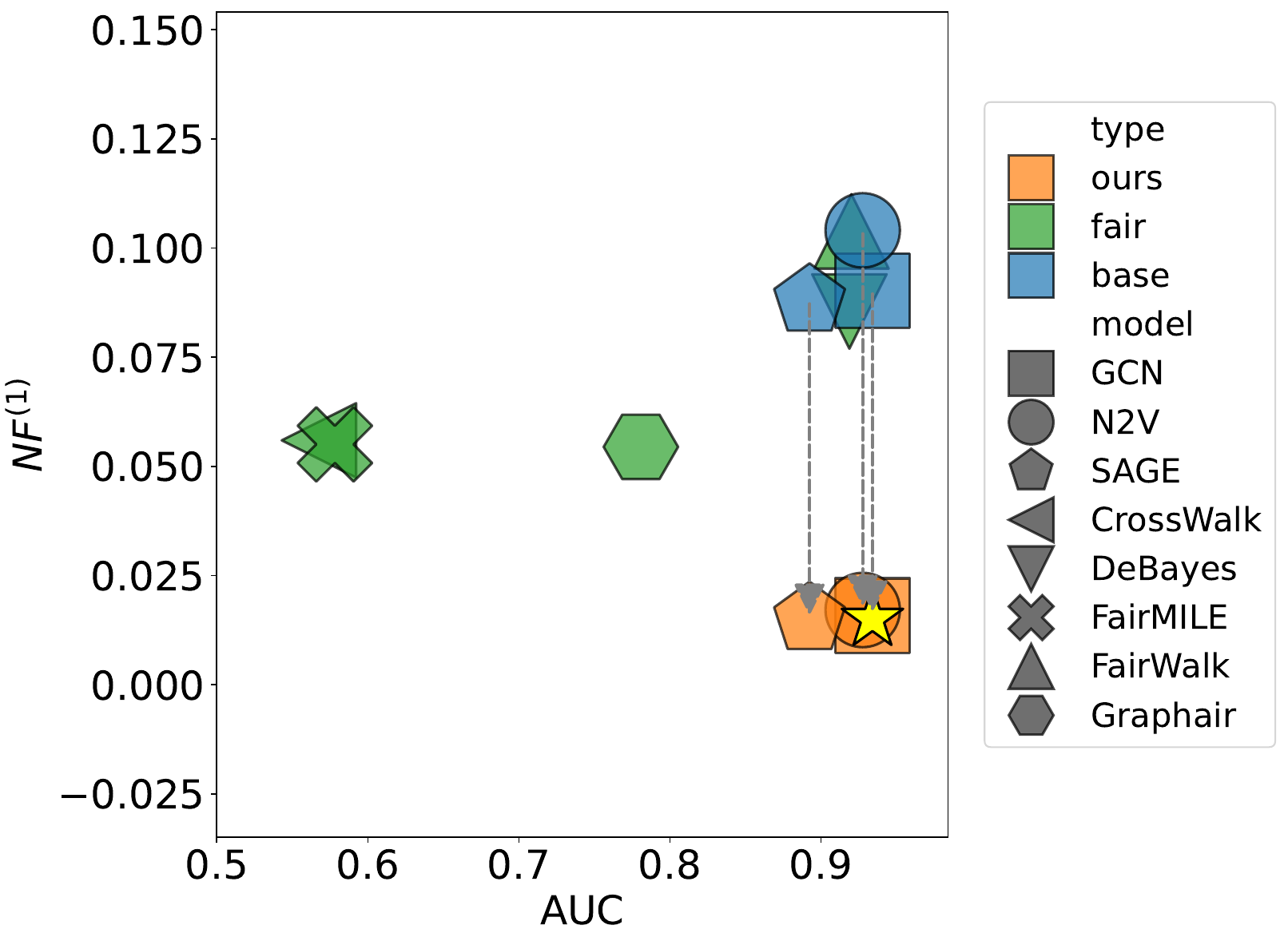}
        \caption{Facebook - $k=1$}
    \end{subfigure}
    \begin{subfigure}{0.9\linewidth}
        \centering \includegraphics[width=.4\linewidth]{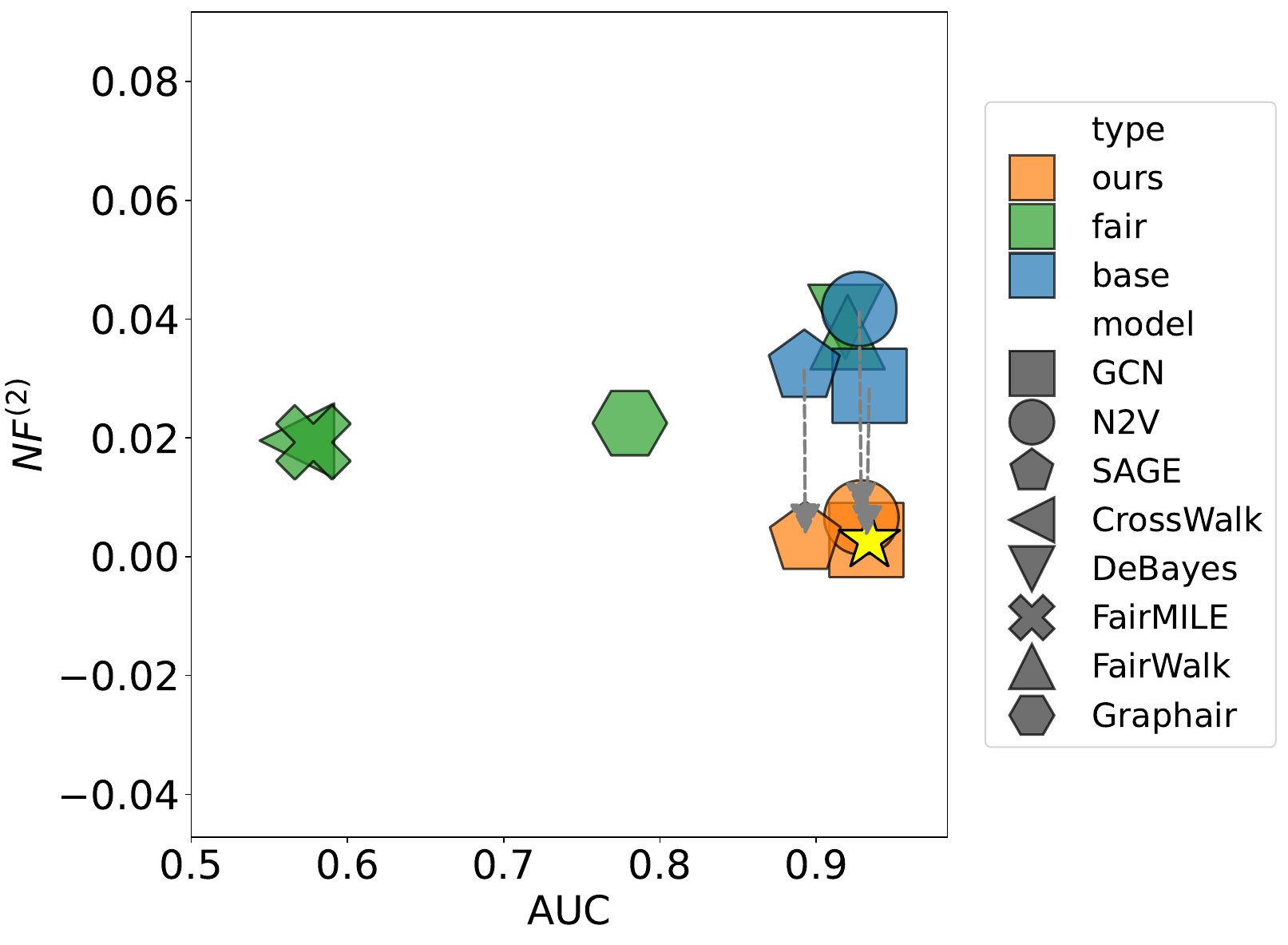}
        \caption{Facebook - $k=2$}
    \end{subfigure}
    \begin{subfigure}{0.9\linewidth}
        \centering \includegraphics[width=.4\linewidth]{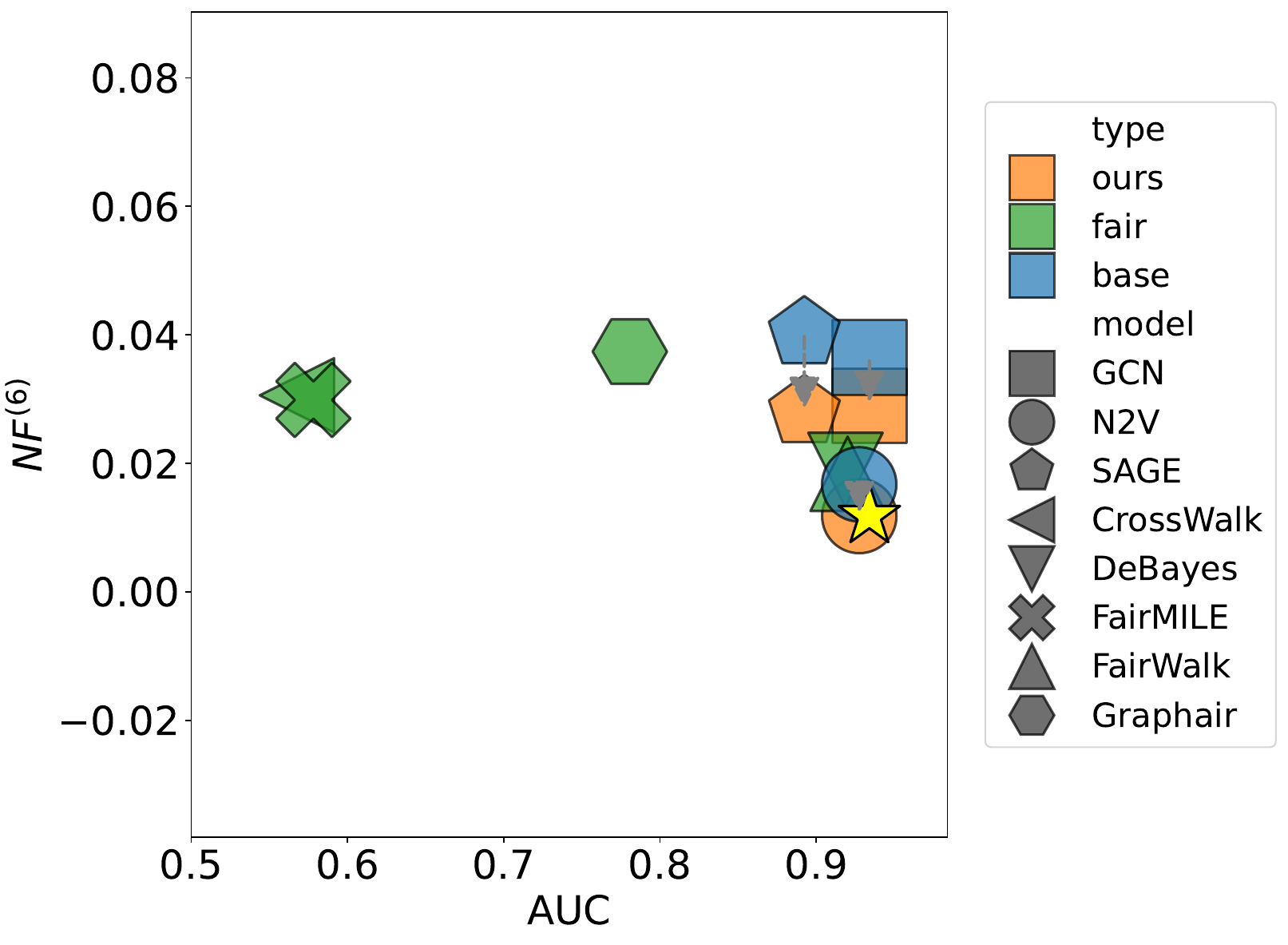}
        \caption{Facebook - $k=6$}
    \end{subfigure}
    \caption{AUC vs $NF^{(k)}$ plot for base and fairness-aware baselines and for targeted $k$ values. Our post-processing results are shown through arrows from base models. The star marker denotes the hypothetical model achieving the best AUC and the best $NF^{(k)}$ across baselines. Only baselines with AUC $> 0.55$ are displayed.}
\end{figure}

\begin{figure}[htbp]
    \centering
    \begin{subfigure}{0.9\linewidth}
        \centering \includegraphics[width=.4\linewidth]{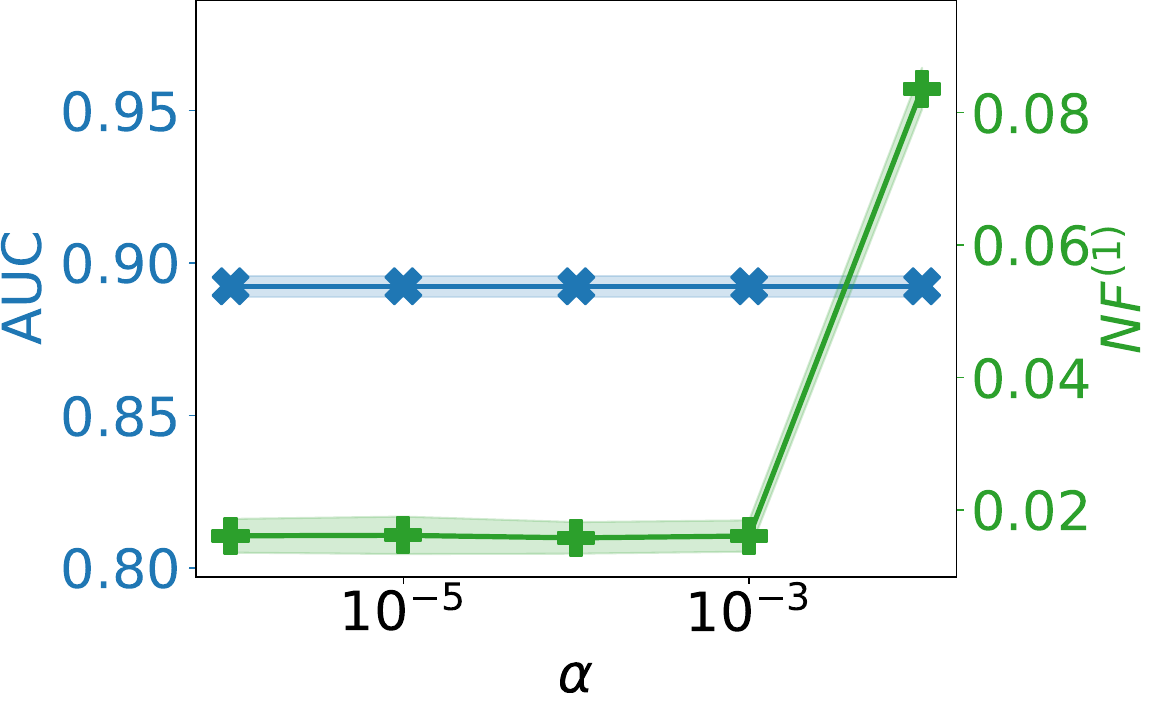}
        \caption{Facebook - $k=1$}
    \end{subfigure}
    \begin{subfigure}{0.9\linewidth}
        \centering \includegraphics[width=.4\linewidth]{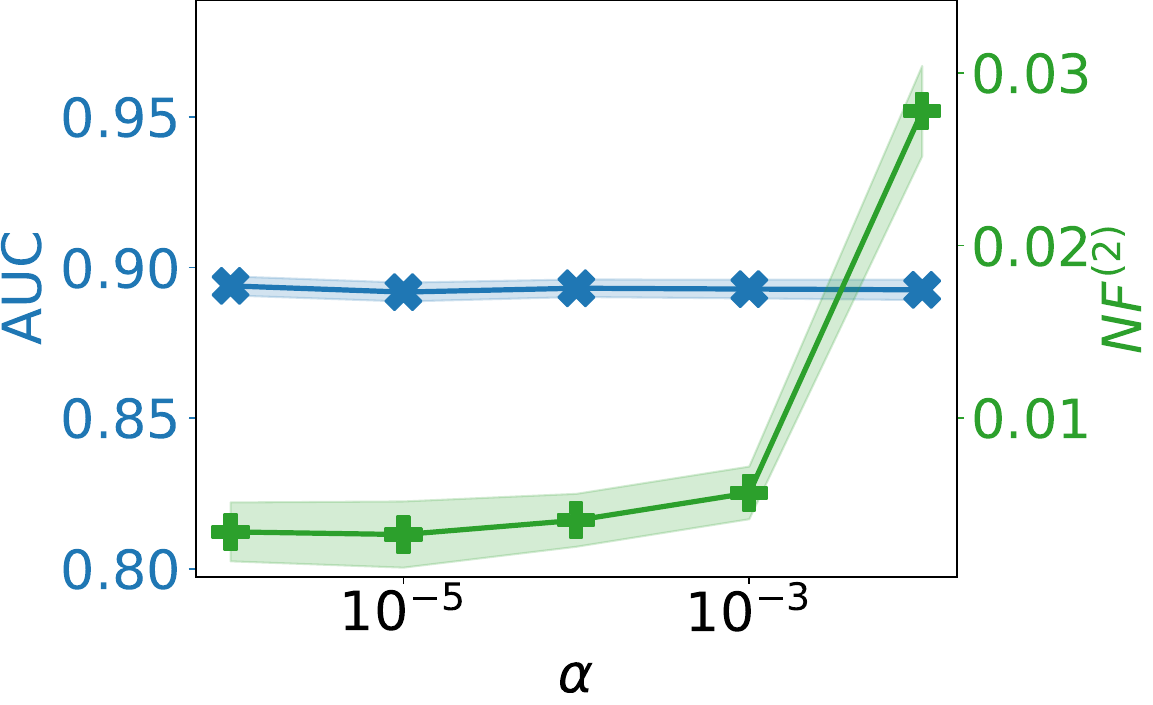}
        \caption{Facebook - $k=2$}
    \end{subfigure}
    \begin{subfigure}{0.9\linewidth}
        \centering \includegraphics[width=.4\linewidth]{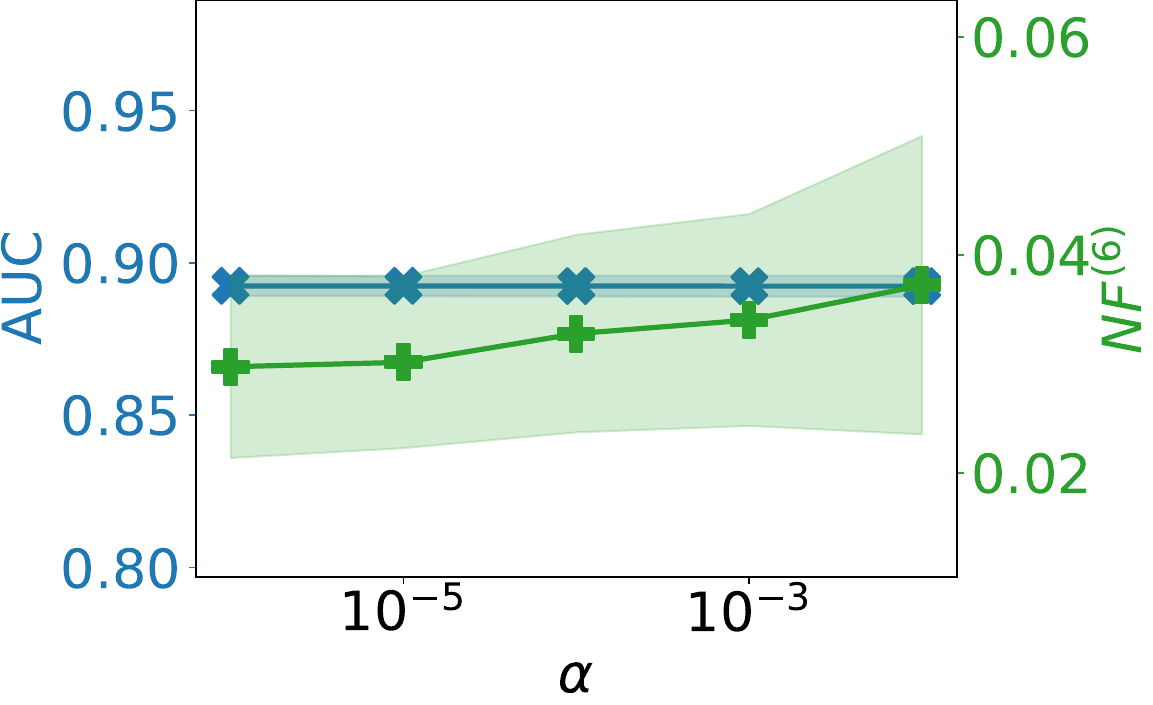}
        \caption{Facebook - $k=6$}
    \end{subfigure}
    \caption{$\alpha$ effect on $NF^{(k)}$ and $AUC$ in the post-processing method for \textbf{SAGE} predictions (log-scale for x-axis).}
\end{figure}

\begin{figure}[htbp]
    \centering
    \begin{subfigure}{0.9\linewidth}
        \centering \includegraphics[width=.4\linewidth]{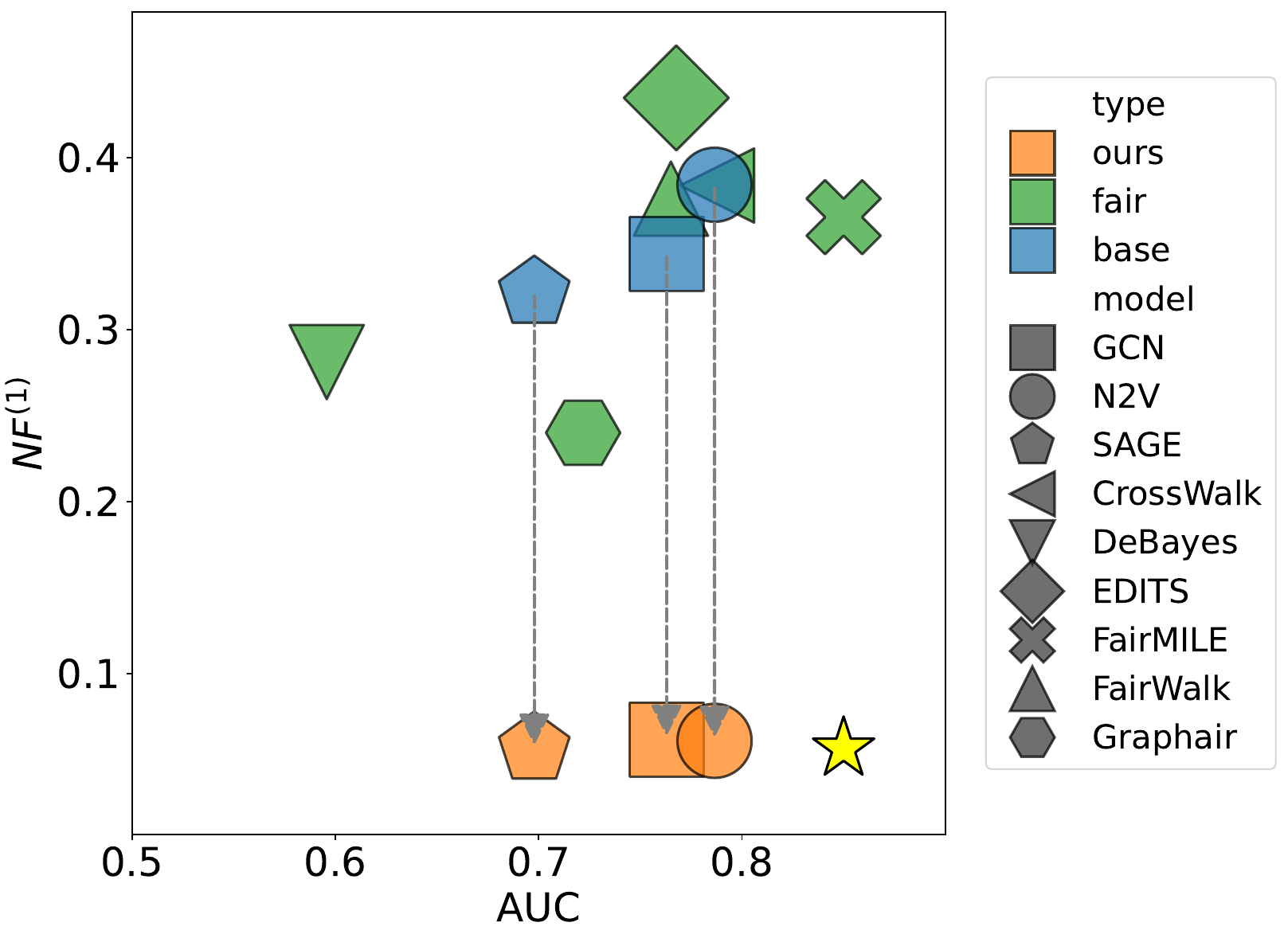}
        \caption{Pokec - $k=1$}
    \end{subfigure}
    \begin{subfigure}{0.9\linewidth}
        \centering \includegraphics[width=.4\linewidth]{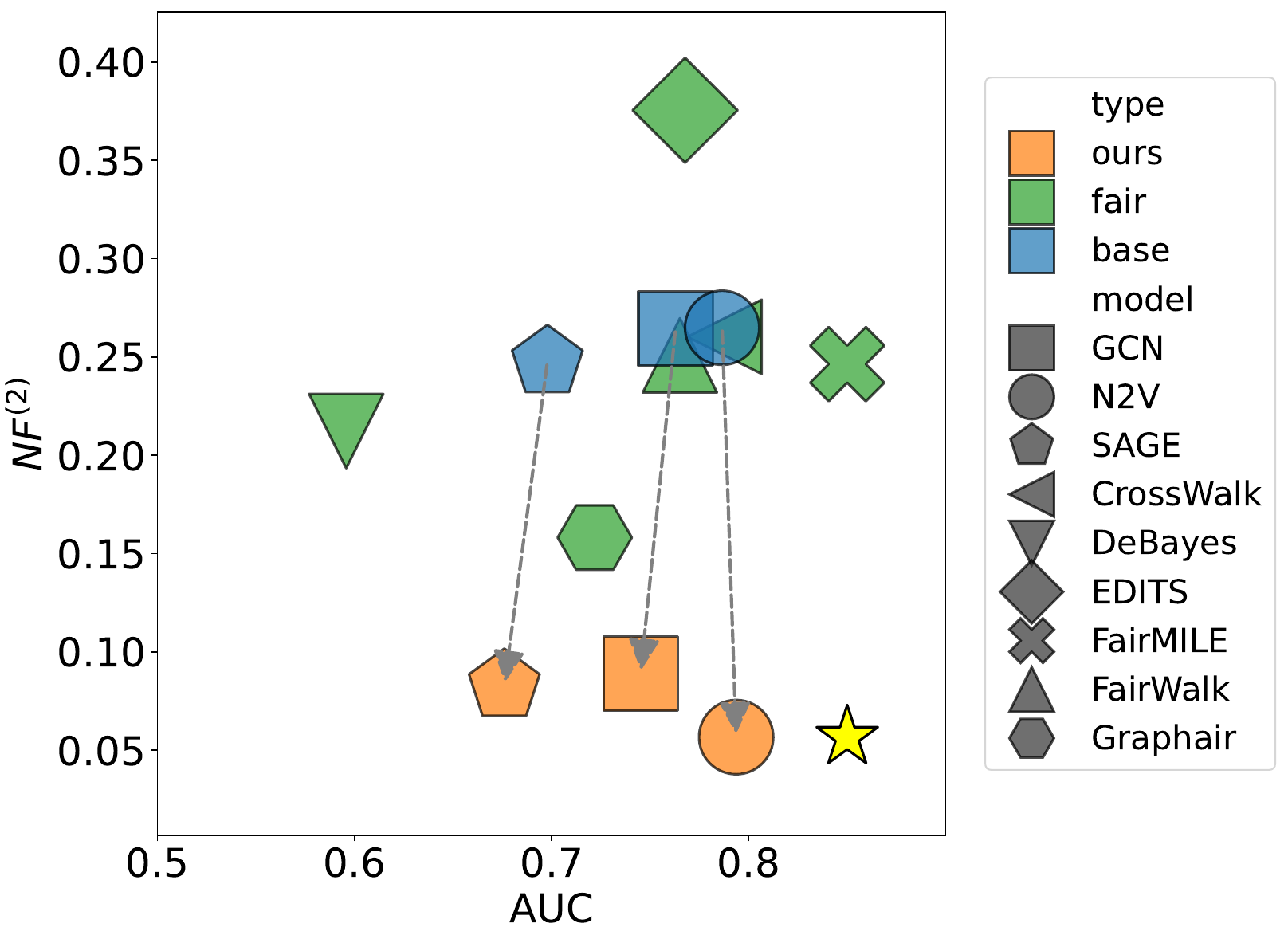}
        \caption{Pokec - $k=2$}
    \end{subfigure}
    \begin{subfigure}{0.9\linewidth}
        \centering \includegraphics[width=.4\linewidth]{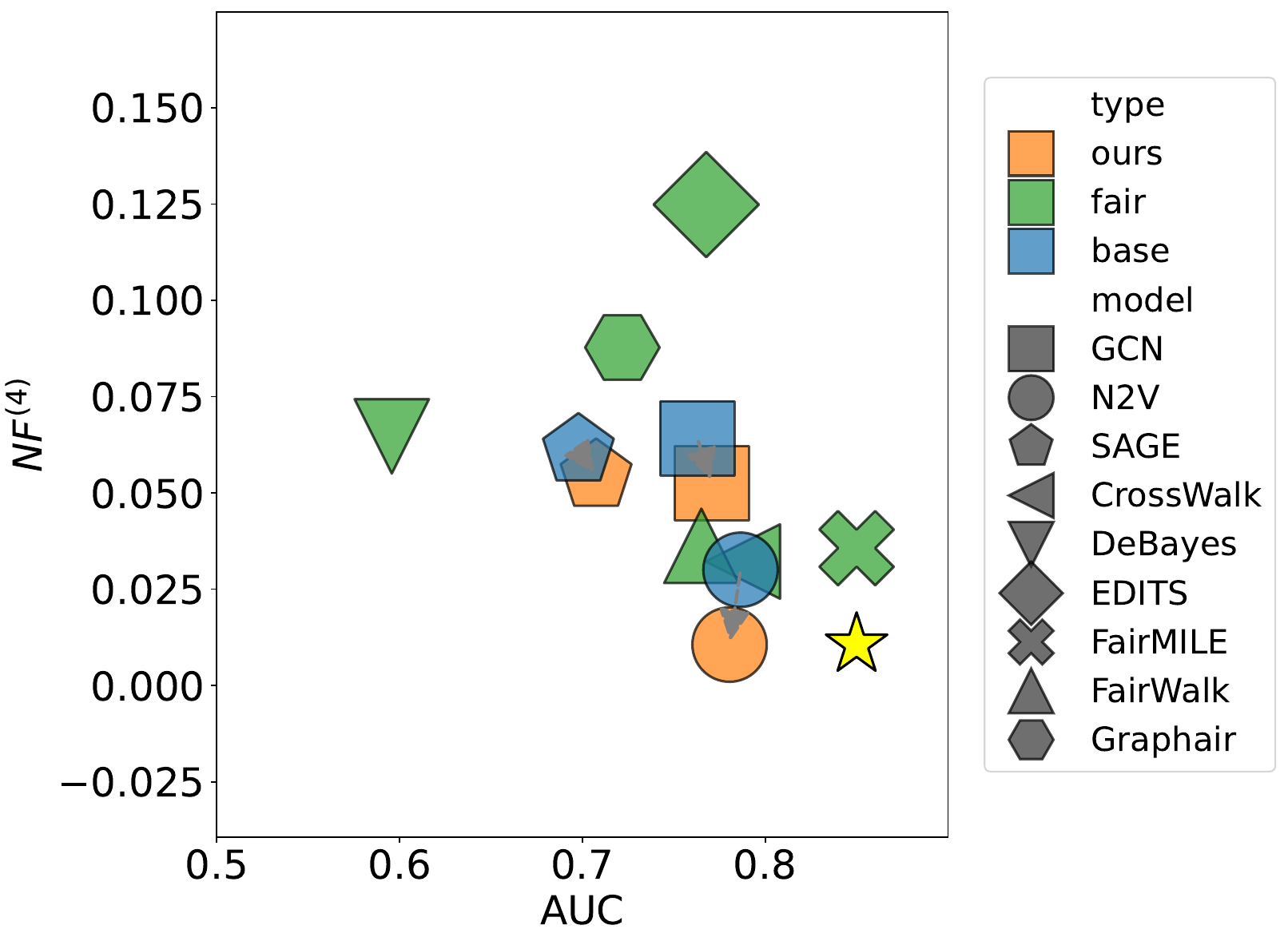}
        \caption{Pokec - $k=4$}
    \end{subfigure}
    \caption{AUC vs $NF^{(k)}$ plot for base and fairness-aware baselines and for targeted $k$ values. Our post-processing results are shown through arrows from base models. The star marker denotes the hypothetical model achieving the best AUC and the best $NF^{(k)}$ across baselines. Only baselines with AUC $> 0.55$ are displayed.}
\end{figure}

\begin{figure}[htbp]
    \centering
    \begin{subfigure}{0.9\linewidth}
        \centering \includegraphics[width=.4\linewidth]{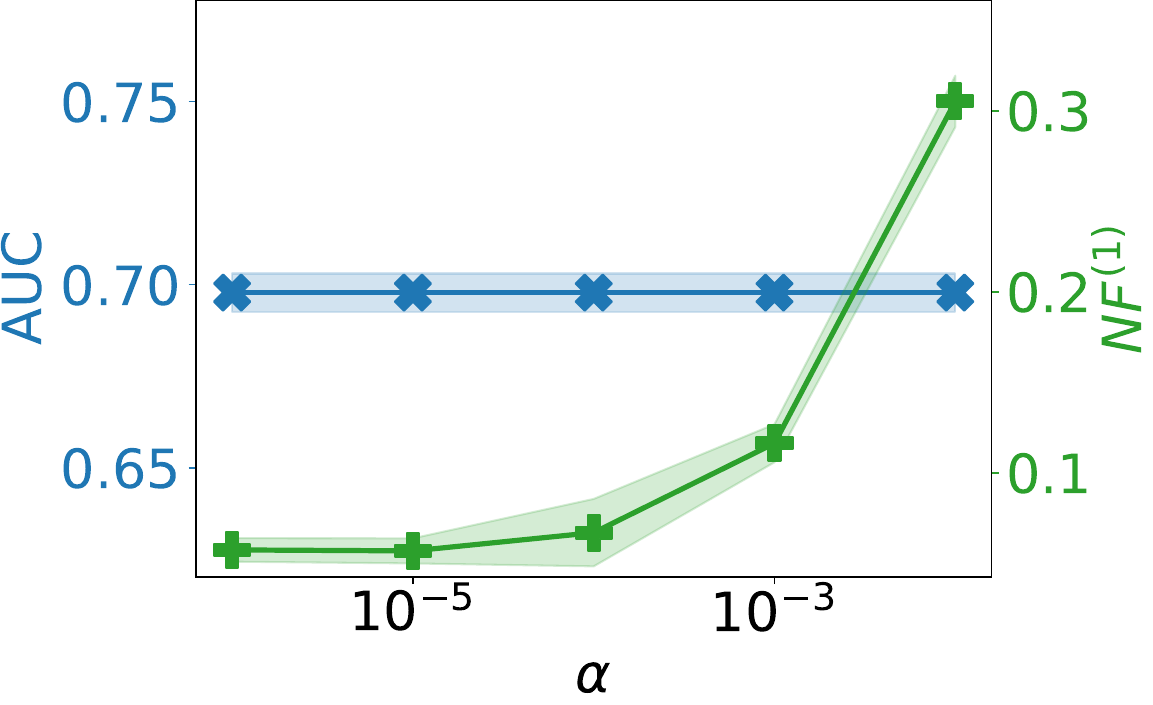}
        \caption{Pokec - $k=1$}
    \end{subfigure}
    \begin{subfigure}{0.9\linewidth}
        \centering \includegraphics[width=.4\linewidth]{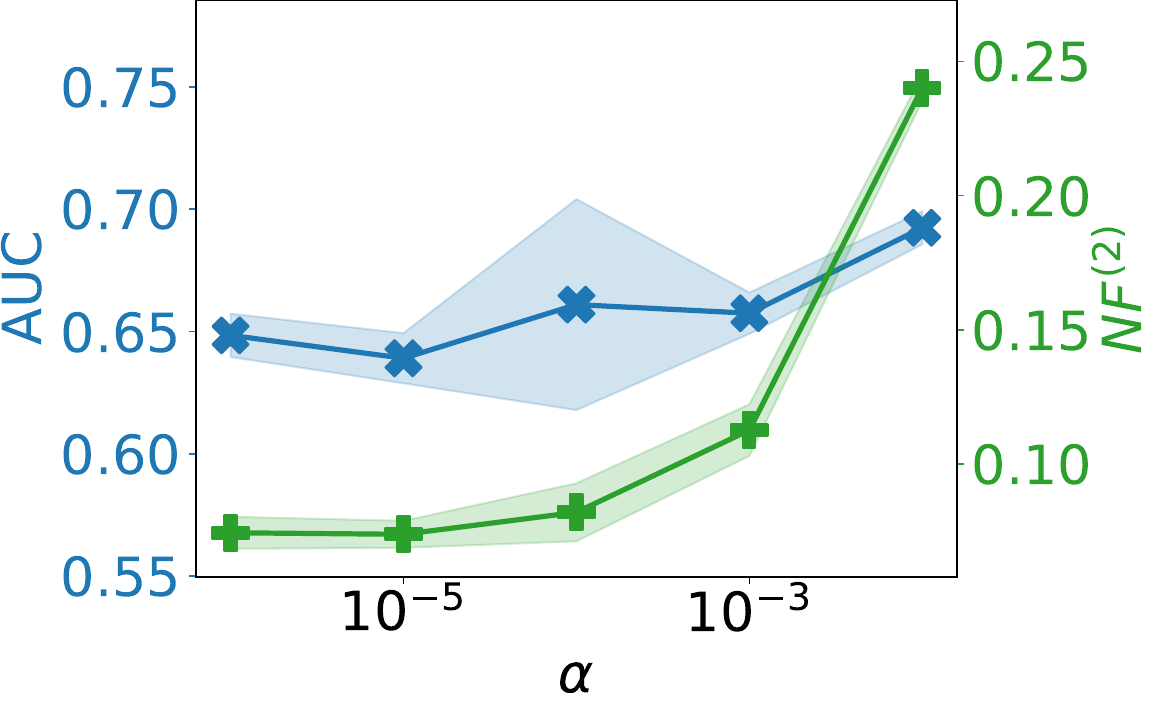}
        \caption{Pokec - $k=2$}
    \end{subfigure}
    \begin{subfigure}{0.9\linewidth}
        \centering \includegraphics[width=.4\linewidth]{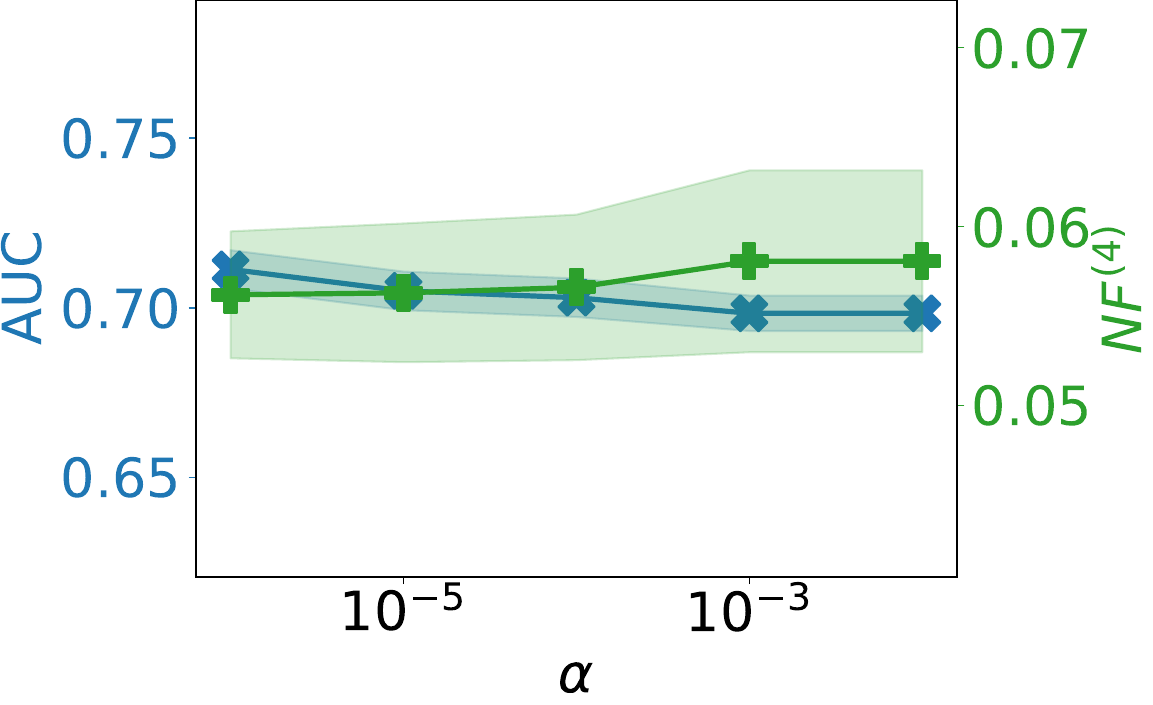}
        \caption{Pokec - $k=4$}
    \end{subfigure}
    \caption{$\alpha$ effect on $NF^{(k)}$ and $AUC$ in the post-processing method for \textbf{SAGE} predictions (log-scale for x-axis).}
\end{figure}

\begin{figure}[htbp]
    \centering
    \begin{subfigure}{0.9\linewidth}
        \centering \includegraphics[width=.4\linewidth]{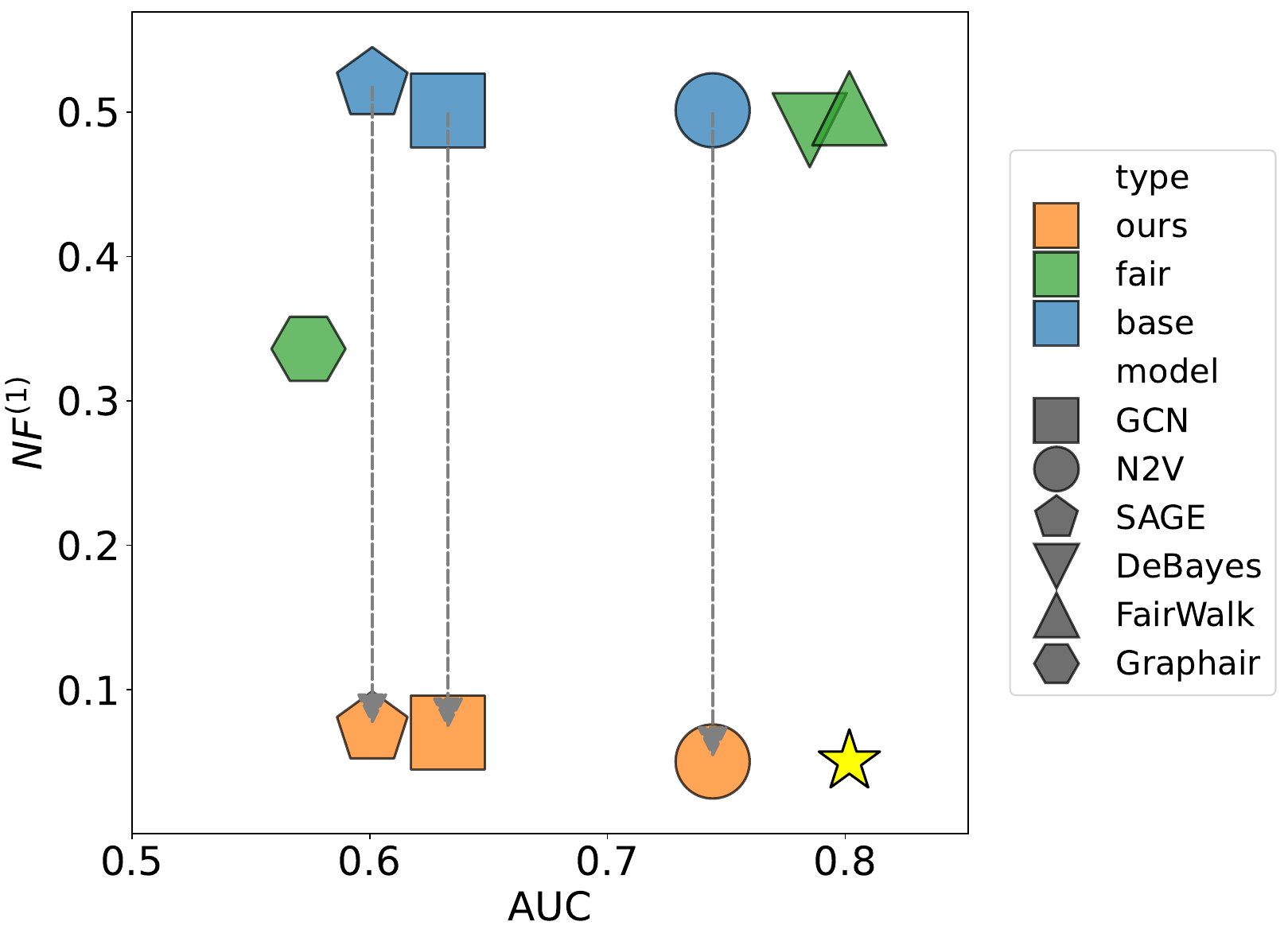}
        \caption{Citeseer - $k=1$}
    \end{subfigure}
    \begin{subfigure}{0.9\linewidth}
        \centering \includegraphics[width=.4\linewidth]{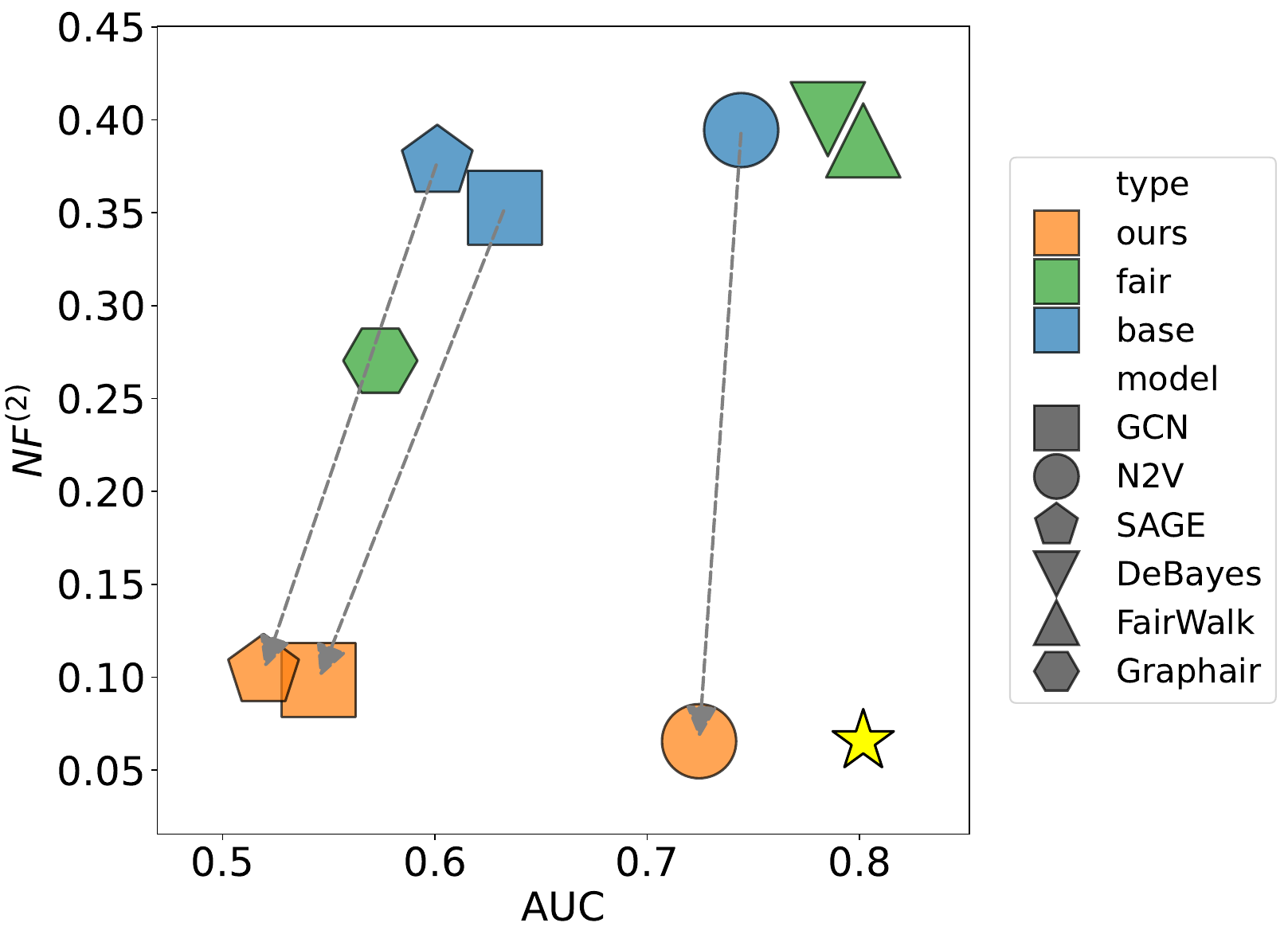}
        \caption{Citeseer - $k=2$}
    \end{subfigure}
    \begin{subfigure}{0.9\linewidth}
        \centering \includegraphics[width=.4\linewidth]{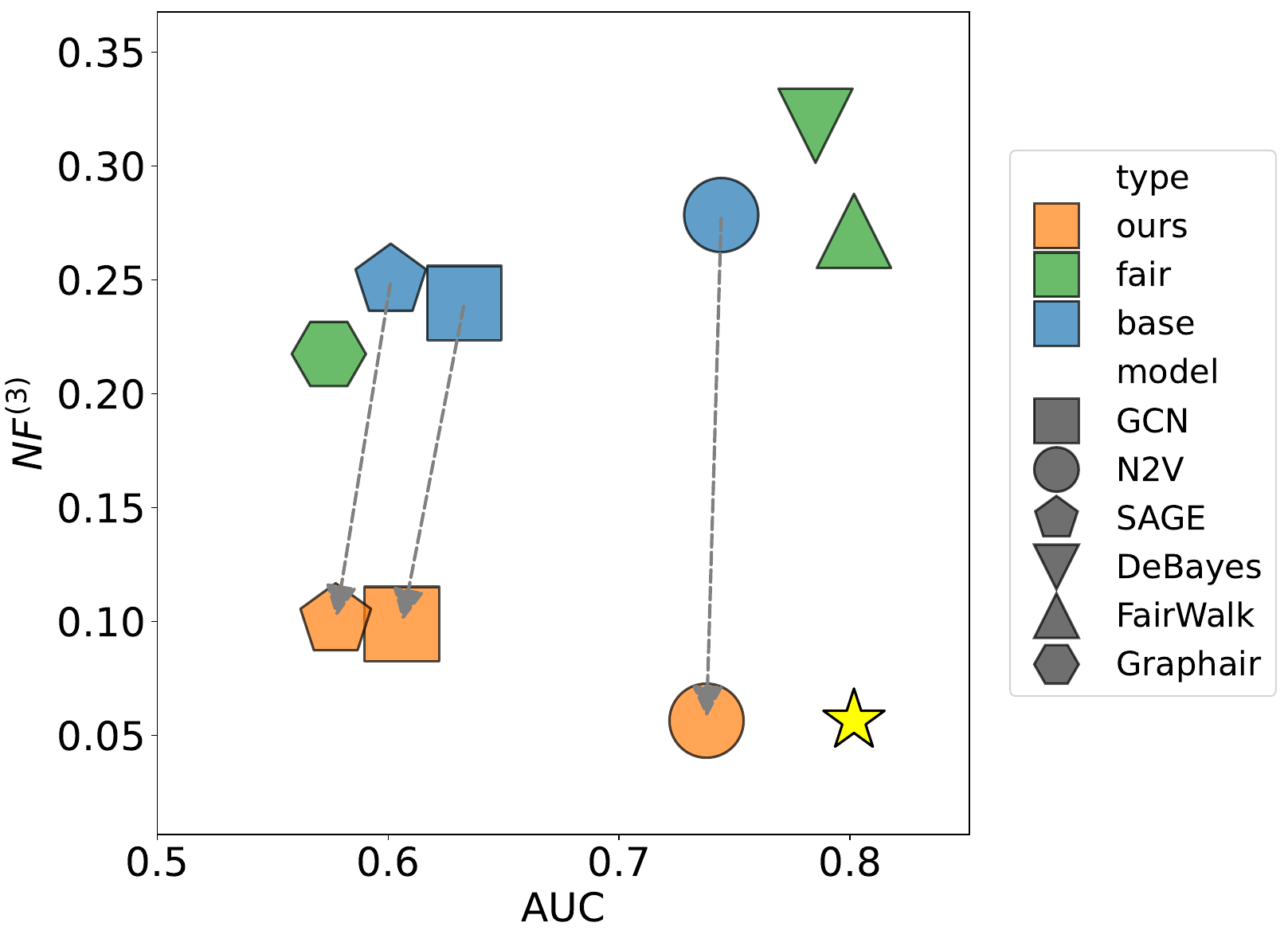}
        \caption{Citeseer - $k=3$}
    \end{subfigure}
    \caption{AUC vs $NF^{(k)}$ plot for base and fairness-aware baselines and for targeted $k$ values. Our post-processing results are shown through arrows from base models. The star marker denotes the hypothetical model achieving the best AUC and the best $NF^{(k)}$ across baselines. Only baselines with AUC $> 0.55$ are displayed.}
\end{figure}

\begin{figure}[htbp]
    \centering
    \begin{subfigure}{0.9\linewidth}
        \centering \includegraphics[width=.4\linewidth]{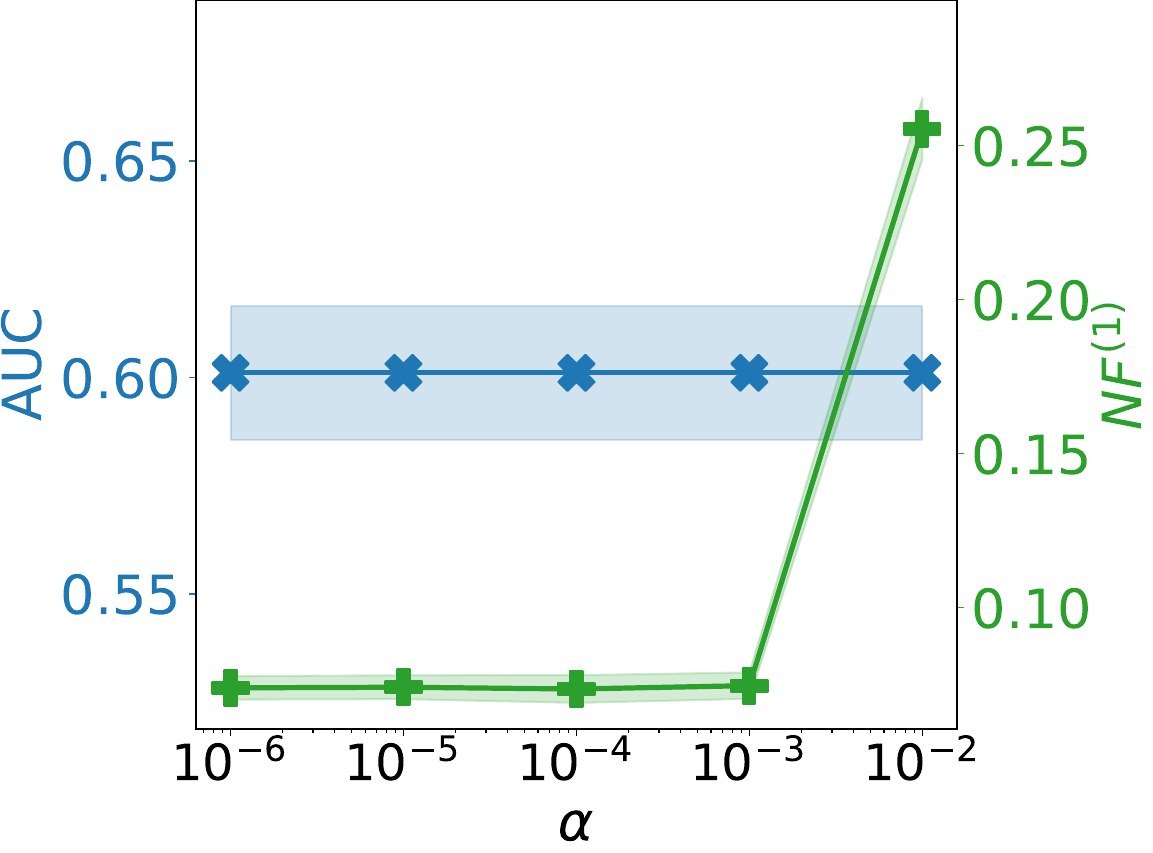}
        \caption{Citeseer - $k=1$}
    \end{subfigure}
    \begin{subfigure}{0.9\linewidth}
        \centering \includegraphics[width=.4\linewidth]{figures/alpha_citeseer_SAGE_2.pdf}
        \caption{Citeseer - $k=2$}
    \end{subfigure}
    \begin{subfigure}{0.9\linewidth}
        \centering \includegraphics[width=.4\linewidth]{figures/alpha_citeseer_SAGE_3.pdf}
        \caption{Citeseer - $k=3$}
    \end{subfigure}
    \caption{$\alpha$ effect on $NF^{(k)}$ and $AUC$ in the post-processing method for \textbf{SAGE} predictions (log-scale for x-axis).}
\end{figure}

\begin{figure}[htbp]
    \centering
    \begin{subfigure}{0.9\linewidth}
        \centering \includegraphics[width=.4\linewidth]{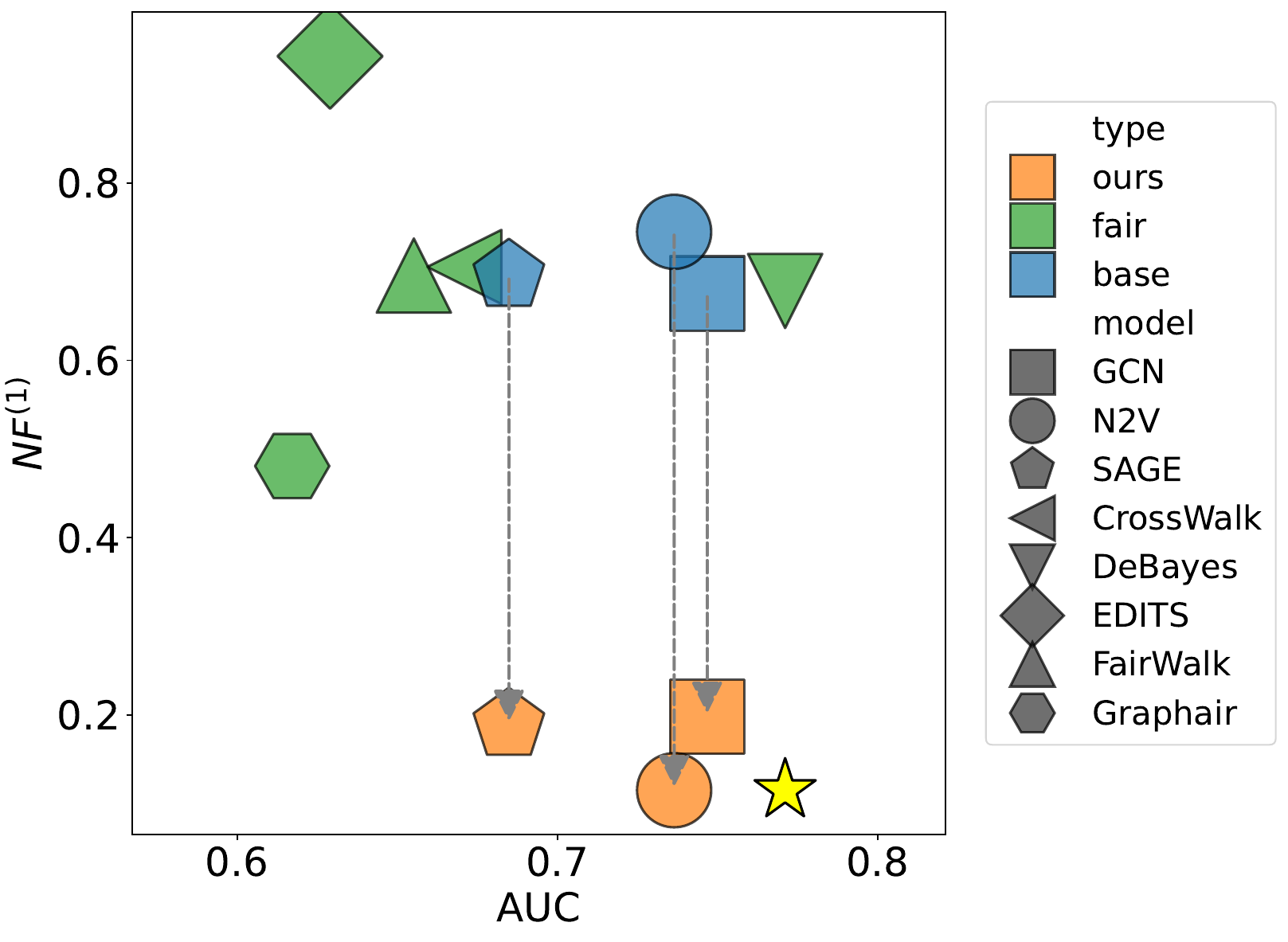}
        \caption{Synthetic - $k=1$}
    \end{subfigure}
    \begin{subfigure}{0.9\linewidth}
        \centering \includegraphics[width=.4\linewidth]{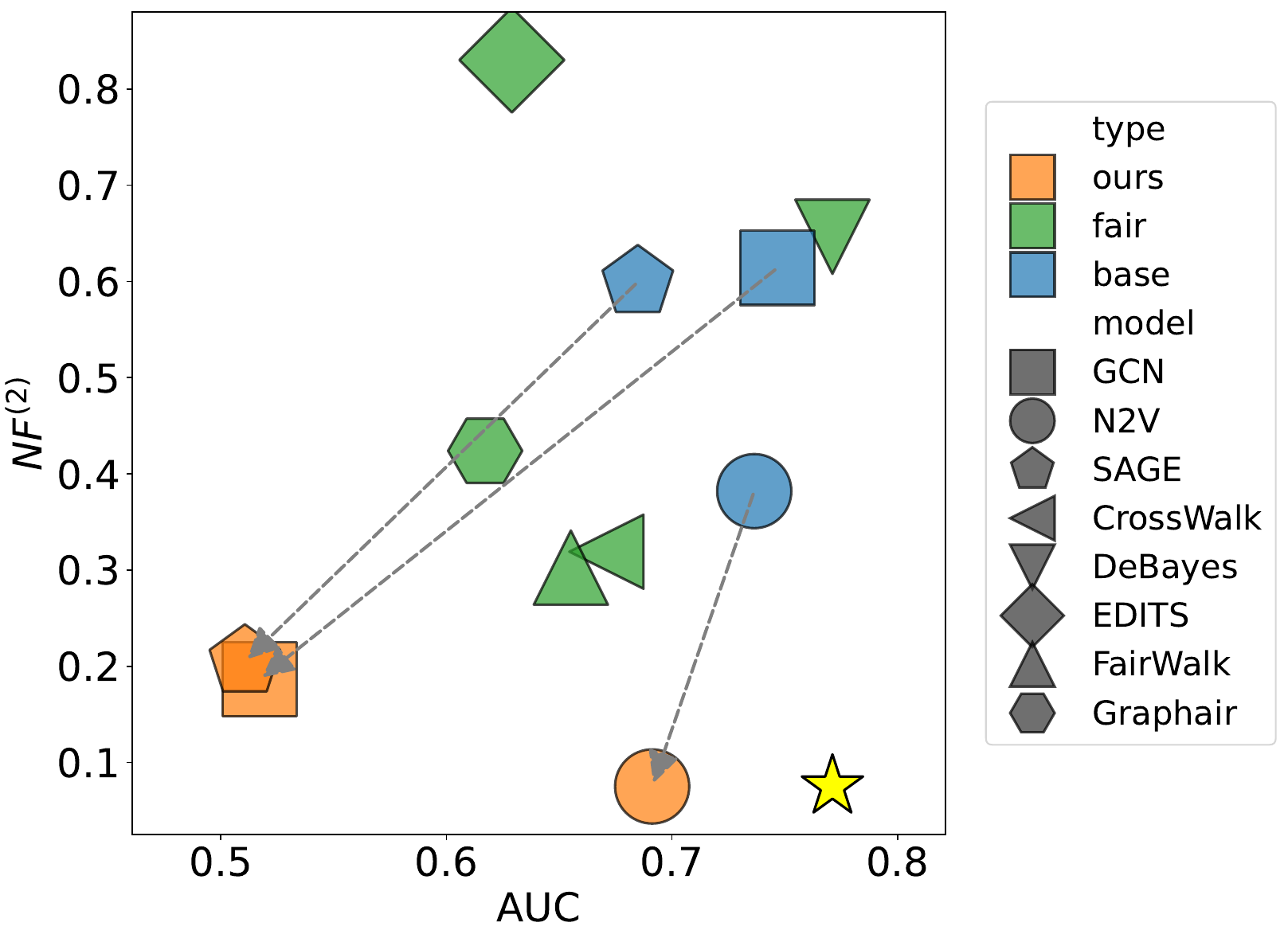}
        \caption{Synthetic - $k=2$}
    \end{subfigure}
    \begin{subfigure}{0.9\linewidth}
        \centering \includegraphics[width=.4\linewidth]{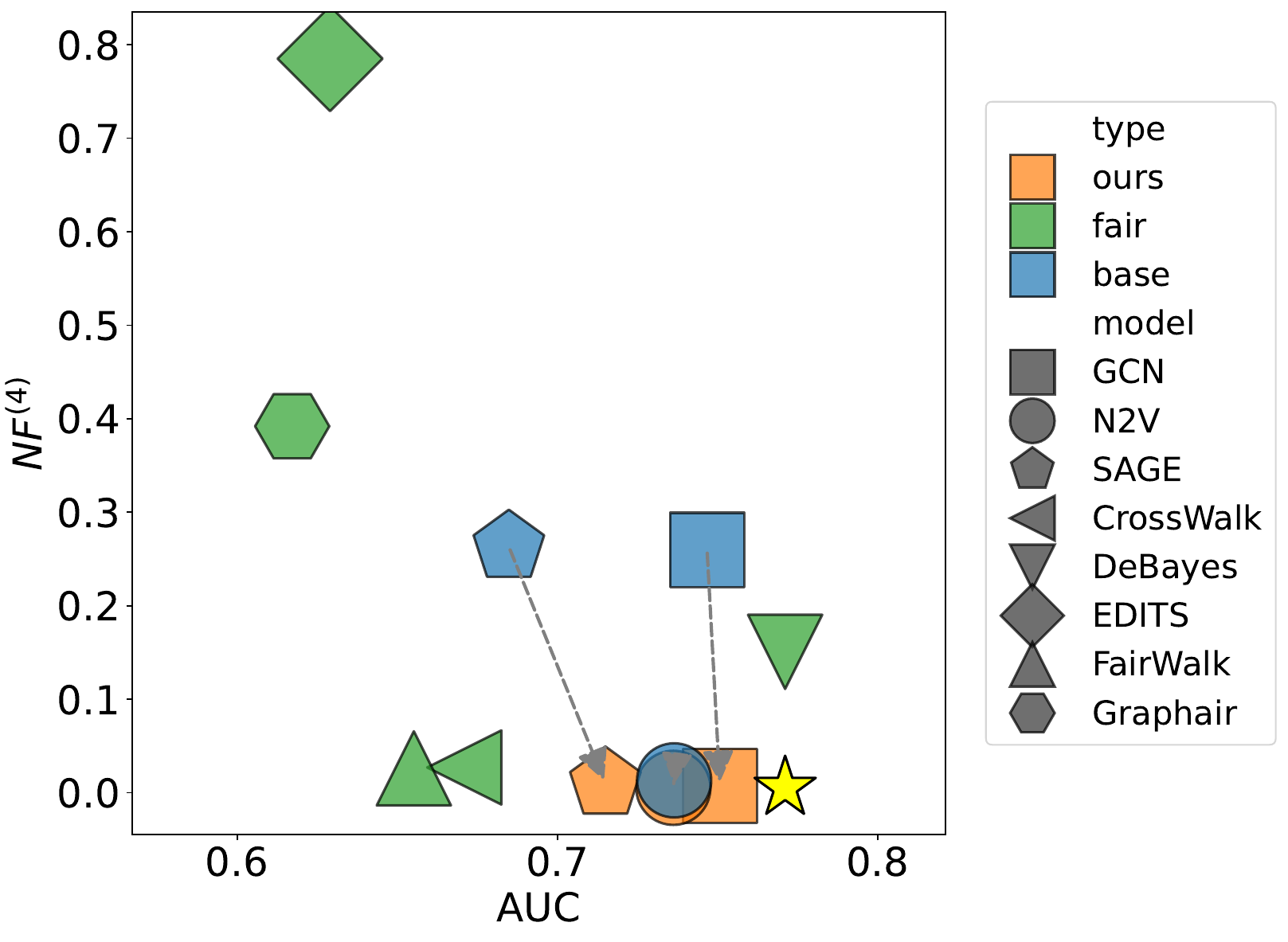}
        \caption{Synthetic - $k=4$}
    \end{subfigure}
    \caption{AUC vs $NF^{(k)}$ plot for base and fairness-aware baselines and for targeted $k$ values. Our post-processing results are shown through arrows from base models. The star marker denotes the hypothetical model achieving the best AUC and the best $NF^{(k)}$ across baselines. Only baselines with AUC $> 0.55$ are displayed.}
\end{figure}

\begin{figure}[htbp]
    \centering
    \begin{subfigure}{0.9\linewidth}
        \centering \includegraphics[width=.4\linewidth]{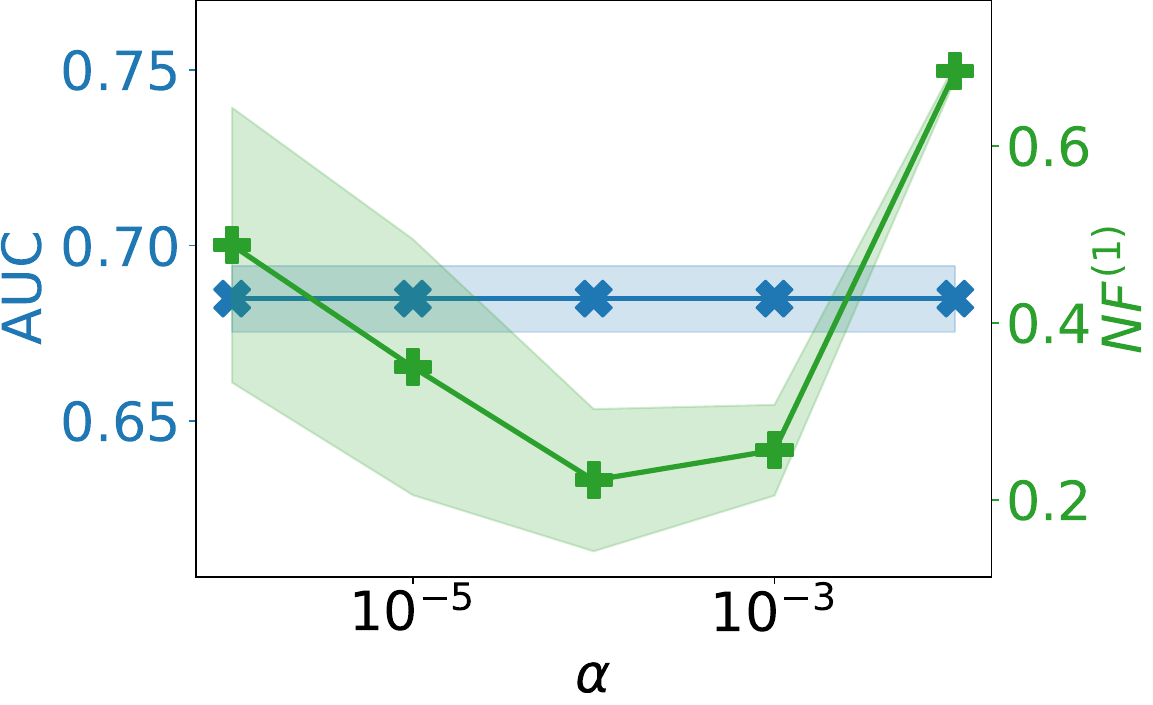}
        \caption{Synthetic - $k=1$}
    \end{subfigure}
    \begin{subfigure}{0.9\linewidth}
        \centering \includegraphics[width=.4\linewidth]{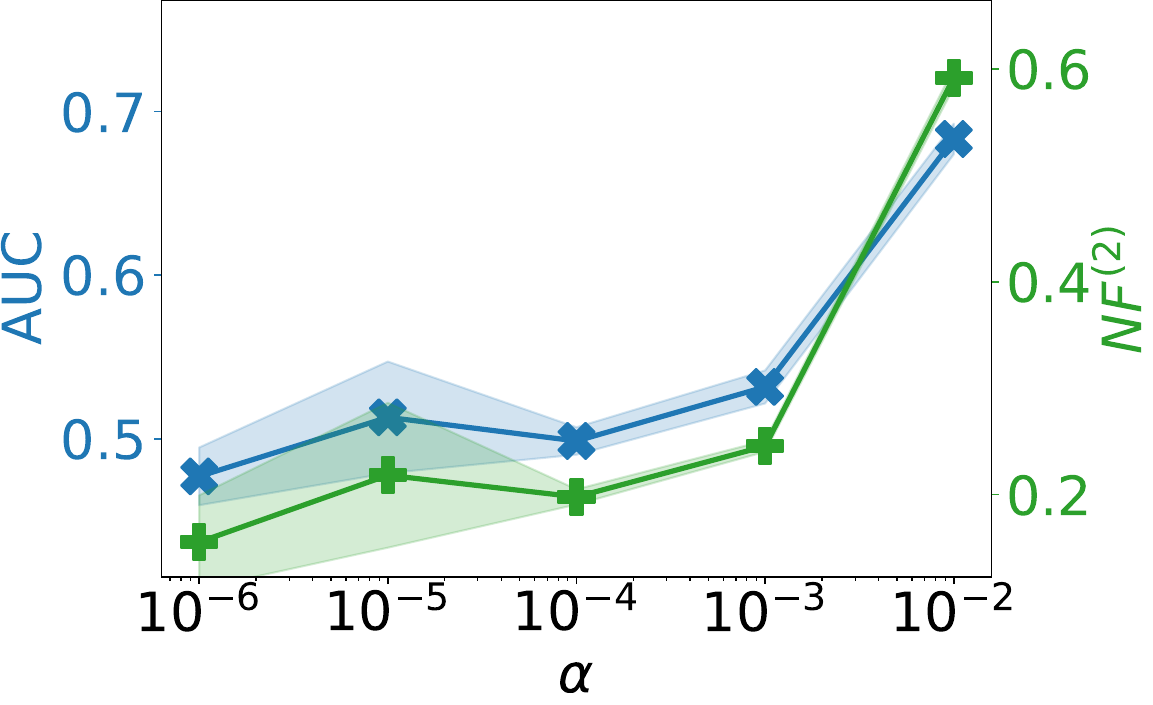}
        \caption{Synthetic - $k=2$}
    \end{subfigure}
    \begin{subfigure}{0.9\linewidth}
        \centering \includegraphics[width=.4\linewidth]{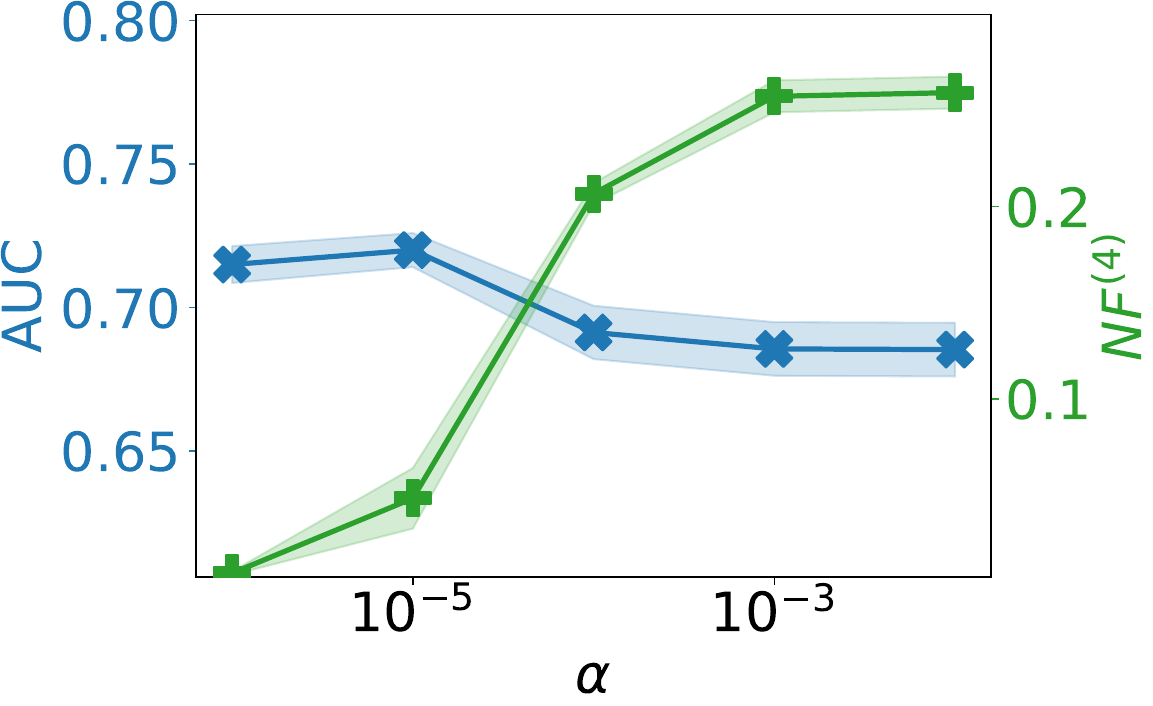}
        \caption{Synthetic - $k=4$}
    \end{subfigure}
    \caption{$\alpha$ effect on $NF^{(k)}$ and $AUC$ in the post-processing method for \textbf{SAGE} predictions (log-scale for x-axis).}
    \label{fig:rq3Sage}

\end{figure}

\cleardoublepage
\footnotesize
\begin{table}[ht]
\centering
\begin{tabular}{llccccccccc}
\toprule
Dataset & Model & $NF^{(1)}$ & $NF^{(2)}$ & $NF^{(3)}$ & $NF^{(4)}$ & $NF^{(5)}$ & $NF^{(6)}$ & DP & EO & AUC \\
\midrule
\textit{Polblogs} & N2V & \footnotesize 0.62 ± \tiny 0.01 & \footnotesize 0.35 ± \tiny 0.01 & \footnotesize 0.07 ± \tiny 0.01 & \footnotesize 0.05 ± \tiny 0.01 & \footnotesize 0.07 ± \tiny 0.04 & \footnotesize 0.08 ± \tiny 0.06 & \footnotesize 0.35 ± \tiny 0.02 & \footnotesize 0.26 ± \tiny 0.02 & \footnotesize 0.76 ± \tiny 0.01 \\
 & GCN & \footnotesize 0.56 ± \tiny 0.01 & \footnotesize 0.29 ± \tiny 0.01 & \footnotesize 0.12 ± \tiny 0.01 & \footnotesize 0.22 ± \tiny 0.01 & \footnotesize 0.12 ± \tiny 0.06 & \footnotesize 0.12 ± \tiny 0.08 & \footnotesize 0.20 ± \tiny 0.02 & \footnotesize 0.09 ± \tiny 0.03 & \footnotesize 0.89 ± \tiny 0.00 \\
 & SAGE & \footnotesize 0.55 ± \tiny 0.02 & \footnotesize 0.29 ± \tiny 0.03 & \footnotesize 0.11 ± \tiny 0.03 & \footnotesize 0.19 ± \tiny 0.02 & \footnotesize 0.12 ± \tiny 0.05 & \footnotesize 0.06 ± \tiny 0.04 & \footnotesize 0.18 ± \tiny 0.05 & \footnotesize 0.06 ± \tiny 0.04 & \footnotesize 0.84 ± \tiny 0.01 \\
 \specialrule{0.2pt}{0pt}{0pt}
 & CrossWalk & \footnotesize 0.52 ± \tiny 0.01 & \footnotesize 0.29 ± \tiny 0.01 & \footnotesize 0.07 ± \tiny 0.01 & \footnotesize 0.18 ± \tiny 0.01 & \footnotesize 0.06 ± \tiny 0.01 & \footnotesize 0.12 ± \tiny 0.05 & \footnotesize 0.13 ± \tiny 0.01 & \footnotesize 0.11 ± \tiny 0.01 & \footnotesize 0.68 ± \tiny 0.01 \\
 & DeBayes & \footnotesize 0.56 ± \tiny 0.00 & \footnotesize 0.41 ± \tiny 0.00 & \footnotesize 0.11 ± \tiny 0.00 & \footnotesize 0.04 ± \tiny 0.01 & \footnotesize 0.07 ± \tiny 0.06 & \footnotesize 0.10 ± \tiny 0.11 & \footnotesize 0.42 ± \tiny 0.01 & \footnotesize 0.33 ± \tiny 0.01 & \footnotesize 0.73 ± \tiny 0.00 \\
 & EDITS & \footnotesize 0.77 ± \tiny 0.01 & \footnotesize 0.47 ± \tiny 0.00 & \footnotesize 0.30 ± \tiny 0.00 & \footnotesize 0.39 ± \tiny 0.00 & \footnotesize 0.09 ± \tiny 0.00 & \footnotesize 0.16 ± \tiny 0.01 & \footnotesize 0.06 ± \tiny 0.00 & \footnotesize 0.01 ± \tiny 0.01 & \footnotesize 0.66 ± \tiny 0.01 \\
 & FairMILE & \footnotesize 0.50 ± \tiny 0.02 & \footnotesize 0.35 ± \tiny 0.02 & \footnotesize 0.05 ± \tiny 0.02 & \footnotesize 0.10 ± \tiny 0.03 & \footnotesize 0.10 ± \tiny 0.04 & \footnotesize 0.26 ± \tiny 0.05 & \footnotesize 0.28 ± \tiny 0.06 & \footnotesize 0.22 ± \tiny 0.06 & \footnotesize 0.78 ± \tiny 0.02 \\
 & FairWalk & \footnotesize 0.54 ± \tiny 0.01 & \footnotesize 0.27 ± \tiny 0.02 & \footnotesize 0.05 ± \tiny 0.01 & \footnotesize 0.12 ± \tiny 0.01 & \footnotesize 0.07 ± \tiny 0.03 & \footnotesize 0.06 ± \tiny 0.05 & \footnotesize 0.14 ± \tiny 0.02 & \footnotesize 0.03 ± \tiny 0.02 & \footnotesize 0.67 ± \tiny 0.01 \\
 & Graphair & \footnotesize 0.46 ± \tiny 0.02 & \footnotesize 0.20 ± \tiny 0.02 & \footnotesize 0.12 ± \tiny 0.01 & \footnotesize 0.18 ± \tiny 0.01 & \footnotesize 0.07 ± \tiny 0.02 & \footnotesize 0.07 ± \tiny 0.05 & \footnotesize 0.10 ± \tiny 0.01 & \footnotesize 0.02 ± \tiny 0.01 & \footnotesize 0.86 ± \tiny 0.01 \\
 & UGE & \footnotesize 0.45 ± \tiny 0.00 & \footnotesize 0.25 ± \tiny 0.00 & \footnotesize 0.11 ± \tiny 0.00 & \footnotesize 0.25 ± \tiny 0.00 & \footnotesize 0.08 ± \tiny 0.00 & \footnotesize 0.02 ± \tiny 0.01 & \footnotesize 0.06 ± \tiny 0.00 & \footnotesize 0.10 ± \tiny 0.01 & \footnotesize 0.61 ± \tiny 0.01 \\
 \specialrule{0.2pt}{0pt}{0pt}
 & GCN-$NF^{(1)}$ & \footnotesize \textbf{0.10} ± \tiny 0.00 & \footnotesize 0.29 ± \tiny 0.01 & \footnotesize 0.12 ± \tiny 0.01 & \footnotesize 0.22 ± \tiny 0.01 & \footnotesize 0.12 ± \tiny 0.06 & \footnotesize 0.12 ± \tiny 0.08 & \footnotesize 0.20 ± \tiny 0.02 & \footnotesize 0.09 ± \tiny 0.03 & \footnotesize 0.89 ± \tiny 0.00 \\
 & GCN-$NF^{(2)}$ & \footnotesize 0.52 ± \tiny 0.01 & \footnotesize \textbf{0.04} ± \tiny 0.00 & \footnotesize 0.12 ± \tiny 0.01 & \footnotesize 0.22 ± \tiny 0.01 & \footnotesize 0.12 ± \tiny 0.06 & \footnotesize 0.12 ± \tiny 0.08 & \footnotesize 0.13 ± \tiny 0.01 & \footnotesize 0.40 ± \tiny 0.01 & \footnotesize 0.61 ± \tiny 0.01 \\
 & GCN-$NF^{(4)}$ & \footnotesize 0.56 ± \tiny 0.01 & \footnotesize 0.30 ± \tiny 0.01 & \footnotesize 0.10 ± \tiny 0.01 & \footnotesize \textbf{0.04} ± \tiny 0.01 & \footnotesize 0.12 ± \tiny 0.06 & \footnotesize 0.12 ± \tiny 0.08 & \footnotesize 0.25 ± \tiny 0.02 & \footnotesize 0.09 ± \tiny 0.03 & \footnotesize 0.88 ± \tiny 0.00 \\
\midrule
\textit{Facebook} & N2V & \footnotesize 0.10 ± \tiny 0.00 & \footnotesize 0.04 ± \tiny 0.00 & \footnotesize 0.02 ± \tiny 0.00 & \footnotesize 0.01 ± \tiny 0.00 & \footnotesize 0.01 ± \tiny 0.01 & \footnotesize 0.02 ± \tiny 0.01 & \footnotesize 0.01 ± \tiny 0.01 & \footnotesize 0.01 ± \tiny 0.01 & \footnotesize 0.93 ± \tiny 0.00 \\
 & GCN & \footnotesize 0.09 ± \tiny 0.00 & \footnotesize 0.03 ± \tiny 0.00 & \footnotesize 0.00 ± \tiny 0.00 & \footnotesize 0.02 ± \tiny 0.00 & \footnotesize 0.02 ± \tiny 0.01 & \footnotesize 0.04 ± \tiny 0.01 & \footnotesize 0.00 ± \tiny 0.00 & \footnotesize 0.02 ± \tiny 0.00 & \footnotesize 0.93 ± \tiny 0.00 \\
 & SAGE & \footnotesize 0.09 ± \tiny 0.00 & \footnotesize 0.03 ± \tiny 0.00 & \footnotesize 0.00 ± \tiny 0.00 & \footnotesize 0.02 ± \tiny 0.00 & \footnotesize 0.02 ± \tiny 0.00 & \footnotesize 0.04 ± \tiny 0.01 & \footnotesize 0.01 ± \tiny 0.00 & \footnotesize 0.03 ± \tiny 0.00 & \footnotesize 0.89 ± \tiny 0.00 \\
 \specialrule{0.2pt}{0pt}{0pt}
 & CrossWalk & \footnotesize 0.06 ± \tiny 0.00 & \footnotesize 0.02 ± \tiny 0.00 & \footnotesize 0.01 ± \tiny 0.00 & \footnotesize 0.02 ± \tiny 0.00 & \footnotesize 0.00 ± \tiny 0.00 & \footnotesize 0.03 ± \tiny 0.00 & \footnotesize 0.00 ± \tiny 0.00 & \footnotesize 0.00 ± \tiny 0.00 & \footnotesize 0.57 ± \tiny 0.01 \\
 & DeBayes & \footnotesize 0.09 ± \tiny 0.00 & \footnotesize 0.04 ± \tiny 0.00 & \footnotesize 0.01 ± \tiny 0.00 & \footnotesize 0.00 ± \tiny 0.00 & \footnotesize 0.01 ± \tiny 0.00 & \footnotesize 0.02 ± \tiny 0.01 & \footnotesize 0.01 ± \tiny 0.00 & \footnotesize 0.05 ± \tiny 0.01 & \footnotesize 0.92 ± \tiny 0.00 \\
 & EDITS & \footnotesize 0.11 ± \tiny 0.00 & \footnotesize 0.04 ± \tiny 0.00 & \footnotesize 0.02 ± \tiny 0.00 & \footnotesize 0.03 ± \tiny 0.00 & \footnotesize 0.01 ± \tiny 0.00 & \footnotesize 0.06 ± \tiny 0.00 & \footnotesize 0.00 ± \tiny 0.00 & \footnotesize 0.01 ± \tiny 0.01 & \footnotesize 0.39 ± \tiny 0.01 \\
 & FairMILE & \footnotesize 0.06 ± \tiny 0.00 & \footnotesize 0.02 ± \tiny 0.00 & \footnotesize 0.00 ± \tiny 0.00 & \footnotesize 0.02 ± \tiny 0.00 & \footnotesize 0.00 ± \tiny 0.00 & \footnotesize 0.03 ± \tiny 0.00 & \footnotesize 0.00 ± \tiny 0.00 & \footnotesize 0.00 ± \tiny 0.00 & \footnotesize 0.58 ± \tiny 0.01 \\
 & FairWalk & \footnotesize 0.10 ± \tiny 0.00 & \footnotesize 0.04 ± \tiny 0.00 & \footnotesize 0.02 ± \tiny 0.00 & \footnotesize 0.01 ± \tiny 0.00 & \footnotesize 0.01 ± \tiny 0.01 & \footnotesize 0.02 ± \tiny 0.01 & \footnotesize 0.01 ± \tiny 0.00 & \footnotesize 0.01 ± \tiny 0.00 & \footnotesize 0.92 ± \tiny 0.01 \\
 & Graphair & \footnotesize 0.05 ± \tiny 0.00 & \footnotesize 0.02 ± \tiny 0.00 & \footnotesize 0.01 ± \tiny 0.00 & \footnotesize 0.02 ± \tiny 0.00 & \footnotesize 0.01 ± \tiny 0.00 & \footnotesize 0.04 ± \tiny 0.01 & \footnotesize 0.01 ± \tiny 0.01 & \footnotesize 0.03 ± \tiny 0.01 & \footnotesize 0.78 ± \tiny 0.02 \\
 & UGE & \footnotesize 0.06 ± \tiny 0.00 & \footnotesize 0.02 ± \tiny 0.00 & \footnotesize 0.00 ± \tiny 0.00 & \footnotesize 0.02 ± \tiny 0.00 & \footnotesize 0.00 ± \tiny 0.00 & \footnotesize 0.03 ± \tiny 0.00 & \footnotesize 0.00 ± \tiny 0.00 & \footnotesize 0.00 ± \tiny 0.00 & \footnotesize 0.54 ± \tiny 0.01 \\
 \specialrule{0.2pt}{0pt}{0pt}
 & GCN-$NF^{(1)}$ & \footnotesize \textbf{0.02} ± \tiny 0.00 & \footnotesize 0.03 ± \tiny 0.00 & \footnotesize 0.00 ± \tiny 0.00 & \footnotesize 0.02 ± \tiny 0.00 & \footnotesize 0.02 ± \tiny 0.01 & \footnotesize 0.04 ± \tiny 0.01 & \footnotesize 0.00 ± \tiny 0.00 & \footnotesize 0.02 ± \tiny 0.00 & \footnotesize 0.93 ± \tiny 0.00 \\
 & GCN-$NF^{(2)}$ & \footnotesize 0.08 ± \tiny 0.00 & \footnotesize \textbf{0.00} ± \tiny 0.00 & \footnotesize 0.00 ± \tiny 0.00 & \footnotesize 0.02 ± \tiny 0.00 & \footnotesize 0.02 ± \tiny 0.01 & \footnotesize 0.04 ± \tiny 0.01 & \footnotesize 0.02 ± \tiny 0.00 & \footnotesize 0.06 ± \tiny 0.00 & \footnotesize 0.93 ± \tiny 0.00 \\
 & GCN-$NF^{(6)}$ & \footnotesize 0.09 ± \tiny 0.00 & \footnotesize 0.03 ± \tiny 0.00 & \footnotesize 0.00 ± \tiny 0.00 & \footnotesize 0.02 ± \tiny 0.00 & \footnotesize 0.02 ± \tiny 0.01 & \footnotesize \textbf{0.03} ± \tiny 0.01 & \footnotesize 0.00 ± \tiny 0.00 & \footnotesize 0.02 ± \tiny 0.00 & \footnotesize 0.93 ± \tiny 0.00 \\
\midrule
\textit{Pokec} & N2V & \footnotesize 0.38 ± \tiny 0.01 & \footnotesize 0.26 ± \tiny 0.01 & \footnotesize 0.05 ± \tiny 0.01 & \footnotesize 0.03 ± \tiny 0.01 & - & - & \footnotesize 0.16 ± \tiny 0.02 & \footnotesize 0.05 ± \tiny 0.02 & \footnotesize 0.79 ± \tiny 0.01 \\
 & GCN & \footnotesize 0.34 ± \tiny 0.01 & \footnotesize 0.26 ± \tiny 0.01 & \footnotesize 0.13 ± \tiny 0.01 & \footnotesize 0.06 ± \tiny 0.00 & - & - & \footnotesize 0.15 ± \tiny 0.01 & \footnotesize 0.11 ± \tiny 0.02 & \footnotesize 0.76 ± \tiny 0.01 \\
 & SAGE & \footnotesize 0.32 ± \tiny 0.01 & \footnotesize 0.25 ± \tiny 0.00 & \footnotesize 0.13 ± \tiny 0.01 & \footnotesize 0.06 ± \tiny 0.01 & - & - & \footnotesize 0.15 ± \tiny 0.01 & \footnotesize 0.09 ± \tiny 0.02 & \footnotesize 0.70 ± \tiny 0.01 \\
 \specialrule{0.2pt}{0pt}{0pt}
 & CrossWalk & \footnotesize 0.38 ± \tiny 0.01 & \footnotesize 0.26 ± \tiny 0.02 & \footnotesize 0.04 ± \tiny 0.01 & \footnotesize 0.03 ± \tiny 0.01 & - & - & \footnotesize 0.15 ± \tiny 0.02 & \footnotesize 0.02 ± \tiny 0.02 & \footnotesize 0.79 ± \tiny 0.01 \\
 & DeBayes & \footnotesize 0.28 ± \tiny 0.01 & \footnotesize 0.21 ± \tiny 0.01 & \footnotesize 0.08 ± \tiny 0.01 & \footnotesize 0.06 ± \tiny 0.00 & - & - & \footnotesize 0.08 ± \tiny 0.01 & \footnotesize 0.03 ± \tiny 0.01 & \footnotesize 0.60 ± \tiny 0.01 \\
 & EDITS & \footnotesize 0.43 ± \tiny 0.01 & \footnotesize 0.38 ± \tiny 0.00 & \footnotesize 0.18 ± \tiny 0.01 & \footnotesize 0.12 ± \tiny 0.00 & - & - & \footnotesize 0.21 ± \tiny 0.01 & \footnotesize 0.10 ± \tiny 0.01 & \footnotesize 0.77 ± \tiny 0.00 \\
 & FairMILE & \footnotesize 0.37 ± \tiny 0.03 & \footnotesize 0.25 ± \tiny 0.03 & \footnotesize 0.14 ± \tiny 0.04 & \footnotesize 0.04 ± \tiny 0.01 & - & - & \footnotesize 0.27 ± \tiny 0.03 & \footnotesize 0.12 ± \tiny 0.03 & \footnotesize 0.85 ± \tiny 0.01 \\
 & FairWalk & \footnotesize 0.38 ± \tiny 0.01 & \footnotesize 0.25 ± \tiny 0.02 & \footnotesize 0.03 ± \tiny 0.01 & \footnotesize 0.04 ± \tiny 0.01 & - & - & \footnotesize 0.13 ± \tiny 0.01 & \footnotesize 0.02 ± \tiny 0.02 & \footnotesize 0.77 ± \tiny 0.01 \\
 & Graphair & \footnotesize 0.24 ± \tiny 0.01 & \footnotesize 0.16 ± \tiny 0.01 & \footnotesize 0.02 ± \tiny 0.00 & \footnotesize 0.09 ± \tiny 0.00 & - & - & \footnotesize 0.02 ± \tiny 0.01 & \footnotesize 0.00 ± \tiny 0.00 & \footnotesize 0.72 ± \tiny 0.02 \\
 & UGE & \footnotesize 0.27 ± \tiny 0.00 & \footnotesize 0.12 ± \tiny 0.00 & \footnotesize 0.01 ± \tiny 0.00 & \footnotesize 0.07 ± \tiny 0.00 & - & - & \footnotesize 0.02 ± \tiny 0.00 & \footnotesize 0.01 ± \tiny 0.01 & \footnotesize 0.55 ± \tiny 0.01 \\
 \specialrule{0.2pt}{0pt}{0pt}
 & N2V-$NF^{(1)}$ & \footnotesize \textbf{0.06} ± \tiny 0.01 & \footnotesize 0.26 ± \tiny 0.01 & \footnotesize 0.05 ± \tiny 0.01 & \footnotesize 0.03 ± \tiny 0.01 & - & - & \footnotesize 0.16 ± \tiny 0.02 & \footnotesize 0.05 ± \tiny 0.02 & \footnotesize 0.79 ± \tiny 0.01 \\
 & N2V-$NF^{(2)}$ & \footnotesize 0.35 ± \tiny 0.01 & \footnotesize \textbf{0.06} ± \tiny 0.01 & \footnotesize 0.05 ± \tiny 0.01 & \footnotesize 0.03 ± \tiny 0.01 & - & - & \footnotesize 0.13 ± \tiny 0.01 & \footnotesize 0.02 ± \tiny 0.02 & \footnotesize 0.77 ± \tiny 0.01 \\
 & N2V-$NF^{(4)}$ & \footnotesize 0.38 ± \tiny 0.01 & \footnotesize 0.26 ± \tiny 0.01 & \footnotesize 0.05 ± \tiny 0.01 & \footnotesize \textbf{0.01} ± \tiny 0.00 & - & - & \footnotesize 0.17 ± \tiny 0.01 & \footnotesize 0.05 ± \tiny 0.02 & \footnotesize 0.78 ± \tiny 0.01 \\
\bottomrule
\end{tabular}
\caption{Comparison of model performance on \textit{Polblogs}, \textit{Facebook}, and \textit{Pokec} datasets, highlighting \textbf{best} scores. 
For our methods, we report only the post-processing variant applied to the LP model yielding the best results.
Presented scores are averaged over 10 iterations with different train/test split random seeds, standard deviation being displayed next to averaged values.}
\label{table:full1}
\end{table}

\cleardoublepage
\footnotesize
\begin{table}[ht]
\centering
\begin{tabular}{llccccccccc}
\toprule
Dataset & Model & $NF^{(1)}$ & $NF^{(2)}$ & $NF^{(3)}$ & $NF^{(4)}$ & $NF^{(5)}$ & $NF^{(6)}$ & DP & EO & AUC \\
\midrule
\textit{Citeseer} & N2V & \footnotesize 0.50 ± \tiny 0.01 & \footnotesize 0.39 ± \tiny 0.01 & \footnotesize 0.28 ± \tiny 0.02 & \footnotesize 0.18 ± \tiny 0.02 & \footnotesize 0.10 ± \tiny 0.02 & \footnotesize 0.05 ± \tiny 0.01 & \footnotesize 0.17 ± \tiny 0.03 & \footnotesize 0.08 ± \tiny 0.05 & \footnotesize 0.74 ± \tiny 0.02 \\
 & GCN & \footnotesize 0.50 ± \tiny 0.01 & \footnotesize 0.35 ± \tiny 0.01 & \footnotesize 0.24 ± \tiny 0.00 & \footnotesize 0.19 ± \tiny 0.00 & \footnotesize 0.14 ± \tiny 0.00 & \footnotesize 0.10 ± \tiny 0.00 & \footnotesize 0.05 ± \tiny 0.01 & \footnotesize 0.04 ± \tiny 0.03 & \footnotesize 0.63 ± \tiny 0.01 \\
 & SAGE & \footnotesize 0.52 ± \tiny 0.01 & \footnotesize 0.38 ± \tiny 0.01 & \footnotesize 0.25 ± \tiny 0.01 & \footnotesize 0.19 ± \tiny 0.00 & \footnotesize 0.15 ± \tiny 0.00 & \footnotesize 0.10 ± \tiny 0.00 & \footnotesize 0.06 ± \tiny 0.01 & \footnotesize 0.05 ± \tiny 0.02 & \footnotesize 0.60 ± \tiny 0.02 \\
 \specialrule{0.2pt}{0pt}{0pt}
 & CrossWalk & \footnotesize 0.30 ± \tiny 0.00 & \footnotesize 0.26 ± \tiny 0.00 & \footnotesize 0.22 ± \tiny 0.00 & \footnotesize 0.18 ± \tiny 0.00 & \footnotesize 0.14 ± \tiny 0.00 & \footnotesize 0.10 ± \tiny 0.00 & \footnotesize 0.00 ± \tiny 0.00 & \footnotesize 0.00 ± \tiny 0.00 & \footnotesize 0.52 ± \tiny 0.02 \\
 & DeBayes & \footnotesize 0.49 ± \tiny 0.01 & \footnotesize 0.40 ± \tiny 0.01 & \footnotesize 0.32 ± \tiny 0.01 & \footnotesize 0.24 ± \tiny 0.01 & \footnotesize 0.17 ± \tiny 0.02 & \footnotesize 0.11 ± \tiny 0.01 & \footnotesize 0.19 ± \tiny 0.02 & \footnotesize 0.13 ± \tiny 0.04 & \footnotesize 0.79 ± \tiny 0.01 \\
 & EDITS & \footnotesize 0.31 ± \tiny 0.00 & \footnotesize 0.27 ± \tiny 0.00 & \footnotesize 0.23 ± \tiny 0.00 & \footnotesize 0.19 ± \tiny 0.00 & \footnotesize 0.15 ± \tiny 0.00 & \footnotesize 0.10 ± \tiny 0.00 & \footnotesize 0.00 ± \tiny 0.00 & \footnotesize 0.01 ± \tiny 0.01 & \footnotesize 0.49 ± \tiny 0.01 \\
 & FairMILE & \footnotesize 0.30 ± \tiny 0.00 & \footnotesize 0.26 ± \tiny 0.00 & \footnotesize 0.22 ± \tiny 0.00 & \footnotesize 0.18 ± \tiny 0.00 & \footnotesize 0.14 ± \tiny 0.00 & \footnotesize 0.10 ± \tiny 0.00 & \footnotesize 0.00 ± \tiny 0.00 & \footnotesize 0.00 ± \tiny 0.00 & \footnotesize 0.52 ± \tiny 0.02 \\
 & FairWalk & \footnotesize 0.50 ± \tiny 0.01 & \footnotesize 0.39 ± \tiny 0.01 & \footnotesize 0.27 ± \tiny 0.02 & \footnotesize 0.18 ± \tiny 0.02 & \footnotesize 0.10 ± \tiny 0.02 & \footnotesize 0.05 ± \tiny 0.01 & \footnotesize 0.17 ± \tiny 0.02 & \footnotesize 0.07 ± \tiny 0.04 & \footnotesize 0.80 ± \tiny 0.01 \\
 & Graphair & \footnotesize 0.34 ± \tiny 0.01 & \footnotesize 0.27 ± \tiny 0.01 & \footnotesize 0.22 ± \tiny 0.01 & \footnotesize 0.18 ± \tiny 0.01 & \footnotesize 0.14 ± \tiny 0.01 & \footnotesize 0.09 ± \tiny 0.01 & \footnotesize 0.01 ± \tiny 0.01 & \footnotesize 0.01 ± \tiny 0.01 & \footnotesize 0.57 ± \tiny 0.03 \\
 & UGE & \footnotesize 0.30 ± \tiny 0.00 & \footnotesize 0.26 ± \tiny 0.00 & \footnotesize 0.22 ± \tiny 0.00 & \footnotesize 0.18 ± \tiny 0.00 & \footnotesize 0.14 ± \tiny 0.00 & \footnotesize 0.10 ± \tiny 0.00 & \footnotesize 0.00 ± \tiny 0.00 & \footnotesize 0.00 ± \tiny 0.00 & \footnotesize 0.49 ± \tiny 0.01 \\
 \specialrule{0.2pt}{0pt}{0pt}
 & N2V-$NF^{(1)}$ & \footnotesize \textbf{0.05} ± \tiny 0.01 & \footnotesize 0.39 ± \tiny 0.01 & \footnotesize 0.28 ± \tiny 0.02 & \footnotesize 0.18 ± \tiny 0.02 & \footnotesize 0.10 ± \tiny 0.02 & \footnotesize 0.05 ± \tiny 0.01 & \footnotesize 0.17 ± \tiny 0.03 & \footnotesize 0.08 ± \tiny 0.05 & \footnotesize 0.74 ± \tiny 0.02 \\
 & N2V-$NF^{(2)}$ & \footnotesize 0.48 ± \tiny 0.01 & \footnotesize \textbf{0.06} ± \tiny 0.01 & \footnotesize 0.28 ± \tiny 0.02 & \footnotesize 0.18 ± \tiny 0.02 & \footnotesize 0.10 ± \tiny 0.02 & \footnotesize 0.05 ± \tiny 0.01 & \footnotesize 0.03 ± \tiny 0.02 & \footnotesize 0.18 ± \tiny 0.07 & \footnotesize 0.72 ± \tiny 0.02 \\
 & N2V-$NF^{(3)}$ & \footnotesize 0.49 ± \tiny 0.01 & \footnotesize 0.36 ± \tiny 0.01 & \footnotesize \textbf{0.06} ± \tiny 0.01 & \footnotesize 0.18 ± \tiny 0.02 & \footnotesize 0.10 ± \tiny 0.02 & \footnotesize 0.05 ± \tiny 0.01 & \footnotesize 0.11 ± \tiny 0.02 & \footnotesize 0.06 ± \tiny 0.05 & \footnotesize 0.74 ± \tiny 0.02 \\
\midrule
\textit{Synthetic} & N2V & \footnotesize 0.74 ± \tiny 0.01 & \footnotesize 0.38 ± \tiny 0.02 & \footnotesize 0.08 ± \tiny 0.00 & \footnotesize 0.01 ± \tiny 0.00 & - & - & \footnotesize 0.32 ± \tiny 0.01 & \footnotesize 0.31 ± \tiny 0.02 & \footnotesize 0.74 ± \tiny 0.01 \\
 & GCN & \footnotesize 0.68 ± \tiny 0.00 & \footnotesize 0.61 ± \tiny 0.00 & \footnotesize 0.04 ± \tiny 0.00 & \footnotesize 0.26 ± \tiny 0.00 & - & - & \footnotesize 0.30 ± \tiny 0.00 & \footnotesize 0.27 ± \tiny 0.01 & \footnotesize 0.75 ± \tiny 0.01 \\
 & SAGE & \footnotesize 0.70 ± \tiny 0.01 & \footnotesize 0.60 ± \tiny 0.01 & \footnotesize 0.04 ± \tiny 0.01 & \footnotesize 0.26 ± \tiny 0.01 & - & - & \footnotesize 0.26 ± \tiny 0.01 & \footnotesize 0.21 ± \tiny 0.01 & \footnotesize 0.68 ± \tiny 0.01 \\
 \specialrule{0.2pt}{0pt}{0pt}
 & CrossWalk & \footnotesize 0.71 ± \tiny 0.01 & \footnotesize 0.32 ± \tiny 0.01 & \footnotesize 0.06 ± \tiny 0.01 & \footnotesize 0.03 ± \tiny 0.00 & - & - & \footnotesize 0.24 ± \tiny 0.01 & \footnotesize 0.21 ± \tiny 0.02 & \footnotesize 0.67 ± \tiny 0.01 \\
 & DeBayes & \footnotesize 0.68 ± \tiny 0.01 & \footnotesize 0.65 ± \tiny 0.01 & \footnotesize 0.19 ± \tiny 0.00 & \footnotesize 0.15 ± \tiny 0.00 & - & - & \footnotesize 0.51 ± \tiny 0.01 & \footnotesize 0.44 ± \tiny 0.01 & \footnotesize 0.77 ± \tiny 0.00 \\
 & EDITS & \footnotesize 0.94 ± \tiny 0.00 & \footnotesize 0.83 ± \tiny 0.00 & \footnotesize 0.22 ± \tiny 0.00 & \footnotesize 0.79 ± \tiny 0.00 & - & - & \footnotesize 0.04 ± \tiny 0.00 & \footnotesize 0.10 ± \tiny 0.01 & \footnotesize 0.63 ± \tiny 0.01 \\
 & FairWalk & \footnotesize 0.70 ± \tiny 0.01 & \footnotesize 0.30 ± \tiny 0.00 & \footnotesize 0.05 ± \tiny 0.00 & \footnotesize 0.03 ± \tiny 0.00 & - & - & \footnotesize 0.22 ± \tiny 0.01 & \footnotesize 0.19 ± \tiny 0.02 & \footnotesize 0.66 ± \tiny 0.01 \\
 & Graphair & \footnotesize 0.48 ± \tiny 0.00 & \footnotesize 0.42 ± \tiny 0.00 & \footnotesize 0.11 ± \tiny 0.00 & \footnotesize 0.39 ± \tiny 0.00 & - & - & \footnotesize 0.01 ± \tiny 0.00 & \footnotesize 0.01 ± \tiny 0.00 & \footnotesize 0.62 ± \tiny 0.01 \\
 & UGE & \footnotesize 0.55 ± \tiny 0.01 & \footnotesize 0.32 ± \tiny 0.00 & \footnotesize 0.07 ± \tiny 0.00 & \footnotesize 0.27 ± \tiny 0.00 & - & - & \footnotesize 0.02 ± \tiny 0.01 & \footnotesize 0.04 ± \tiny 0.02 & \footnotesize 0.52 ± \tiny 0.01 \\
 \specialrule{0.2pt}{0pt}{0pt}
 & GCN-$NF^{(1)}$ & \footnotesize \textbf{0.20} ± \tiny 0.06 & \footnotesize 0.61 ± \tiny 0.00 & \footnotesize 0.04 ± \tiny 0.00 & \footnotesize 0.26 ± \tiny 0.00 & - & - & \footnotesize 0.30 ± \tiny 0.00 & \footnotesize 0.27 ± \tiny 0.01 & \footnotesize 0.75 ± \tiny 0.01 \\
 & GCN-$NF^{(2)}$ & \footnotesize 0.62 ± \tiny 0.00 & \footnotesize \textbf{0.16} ± \tiny 0.02 & \footnotesize 0.04 ± \tiny 0.00 & \footnotesize 0.26 ± \tiny 0.00 & - & - & \footnotesize 0.01 ± \tiny 0.01 & \footnotesize 0.02 ± \tiny 0.01 & \footnotesize 0.50 ± \tiny 0.01 \\
 & GCN-$NF^{(4)}$ & \footnotesize 0.68 ± \tiny 0.00 & \footnotesize 0.62 ± \tiny 0.00 & \footnotesize 0.10 ± \tiny 0.00 & \footnotesize \textbf{0.01} ± \tiny 0.00 & - & - & \footnotesize 0.48 ± \tiny 0.00 & \footnotesize 0.46 ± \tiny 0.02 & \footnotesize 0.75 ± \tiny 0.01 \\
\bottomrule
\end{tabular}
\caption{Comparison of model performance on \textit{Citeseer}, and \textit{Synthetic} datasets, highlighting \textbf{best} scores. 
For our methods, we report only the post-processing variant applied to the LP model yielding the best results.
Presented scores are averaged over 10 iterations with different train/test split random seeds, standard deviation being displayed next to averaged values.}
\label{table:full2}
\end{table}

\end{document}